\documentclass[twoside,11pt]{article}

\usepackage{blindtext}

\usepackage[T1]{fontenc}
\usepackage[utf8]{inputenc}

\usepackage{amsmath,amssymb,amsthm,amstext,mathtools,bm,bbm}

\usepackage{graphicx}
\usepackage{xcolor}
\usepackage{enumitem}
\usepackage{algorithm}
\usepackage{algorithmic}
\usepackage{subcaption}
\usepackage{tikz}
\usetikzlibrary{arrows.meta,shadows,positioning,tikzmark}

\usepackage{changes}

\newcommand{\X}{\mathcal{X}}
\newcommand{\A}{\mathcal{A}}
\newcommand{\E}{\mathbb{E}}
\newcommand{\1}{\mathbbm{1}}
\newcommand{\R}{\mathbb{R}}
\newcommand{\cov}{\mathrm{Cov}}
\newcommand{\law}{\mathrm{Law}}
\newcommand{\sZ}{\mathsf{Z}}
\newcommand{\J}{\mathcal{J}}

\DeclarePairedDelimiter\abs{\lvert}{\rvert}

\newcommand{\limT}{\lim_{T\rightarrow\infty}\frac{1}{T}\sum_{t=0}^{T-1}}

\newcommand{\opt}{\bm{(\mathrm{OPT})}}
\newcommand{\optH}{\bm{(\mathrm{OPT})}_{\mathcal H}}

\newcommand{\opttarc}[1]{%
  \tikzmarknode{a}{\textbf{\textup{#1}}}%
  \begin{tikzpicture}[overlay,remember picture]
    \draw ([yshift=1pt]a.north west) to[bend left=20] ([yshift=1pt]a.north east);
  \end{tikzpicture}%
}

\newcommand{\optt}{%
  \ensuremath{\boldsymbol{(}}\ensuremath{\opttarc{OPT}}\ensuremath{\boldsymbol{)}}%
}

\newcommand{\opttH}{\bm{(\mathrm{OPT})}_{\mathcal H}
}

\usepackage[preprint]{jmlr2e}
\newtheorem{assumption}[theorem]{Assumption}

\newenvironment{qedblock}
{\par\pushQED{\qed}\normalfont}
{\popQED\par}

\usepackage{lastpage}
\ShortHeadings{{Maximum Entropy IRL for MFGs with Average Reward}}{Alkır, Saldı, Anahtarcı, Karıksız}
\begin{document}
\title{Maximum Entropy Inverse Reinforcement Learning for Mean-Field Games with Average Reward}

\author{\name \c{S}evket Kaan Alkır \email kaan.alkir@bilkent.edu.tr \\
\addr Department of Mathematics \\
Bilkent University\\
06800 Ankara, Turkey
\AND
\name Naci Saldı \email naci.saldi@bilkent.edu.tr \\
\addr Department of Mathematics \\
Bilkent University\\
06800 Ankara, Turkey
\AND
\name Berkay Anahtarcı \email berkay.anahtarci@ozyegin.edu.tr \\
\addr Department of Mathematical Engineering \\
Özyeğin University\\
34794 Istanbul, Turkey
\AND
\name Can Deha Karıksız \email deha.kariksiz@ozyegin.edu.tr \\
\addr Department of Mathematical Engineering \\
Özyeğin University\\
34794 Istanbul, Turkey
}

\editor{TBD}

\maketitle

\begin{abstract}
We study inverse reinforcement learning for discrete-time, infinite-horizon mean-field games (MFGs) under an average-reward criterion. Expert demonstrations are assumed to arise from a stationary mean-field equilibrium under an unknown reward, and the goal is to recover a policy explaining the observed behaviour via the maximum causal entropy principle. We formulate the inverse problem by enforcing consistency with the expert mean-field term and long-run feature expectations, treating two reward classes within a unified occupation-measure framework. For finite-dimensional linear rewards, we give a convex dual reformulation with an explicit log-partition objective, and prove smoothness and curvature properties justifying constant-step-size gradient descent. For infinite-dimensional RKHS rewards, we develop a Lagrangian relaxation whose inner-maximising policy is characterised by a soft Bellman equation. The main obstacle is the absence of a discount-factor contraction. We resolve this by introducing a minorisation-based sub-stochastic kernel yielding a strict contraction for the soft Bellman operator. We establish Fr\'echet differentiability and Lipschitz smoothness of the log-likelihood score, leading to a gradient ascent algorithm with convergence guarantees. Two numerical examples, a malware-spread MFG and an RKHS-based consumer-choice model, show that the recovered policies closely match expert behaviour.
\end{abstract}

%\begin{abstract}
%We study inverse reinforcement learning for discrete-time, infinite-horizon mean-field games with an average-reward criterion. Expert demonstrations are assumed to arise from a stationary mean-field equilibrium under an unknown reward, and the goal is to recover a policy that explains the observed behaviour via the maximum causal entropy principle. We formulate the problem by enforcing consistency with the expert mean-field term and long-run feature expectations. In the finite-dimensional linear reward setting, we show that the problem admits an equivalent convex reformulation over occupation measures, derive an explicit dual log-partition objective, and establish regularity properties that justify first-order solution methods. We then extend the framework to richer reward classes by modelling rewards in a reproducing kernel Hilbert space. For this kernel-based formulation, we develop a dual relaxation linked to an entropy-regularised average-reward control problem. To overcome the lack of standard contraction arguments in the average-reward setting, we introduce an equivalent Bellman representation based on a sub-stochastic kernel, which yields a strict contraction for the soft Bellman operator. We further establish Fréchet differentiability and smoothness properties of the associated objective, providing a basis for gradient-based algorithms. A numerical example illustrates that the proposed approach recovers a policy consistent with the expert behaviour.
%\end{abstract}
  
\newpage

\section{Introduction}
Mean-field games (MFGs) provide a framework for modelling strategic interactions in large populations of agents, where each individual responds to the aggregate behaviour of the population rather than to pairwise interactions \citep{HuMaCa06,LaLi07}. In stationary settings, the population behaviour is summarised by a time-invariant state distribution \citep{WeBeRo05}, reducing the problem to a Markov decision process (MDP) constrained by a consistency condition on this distribution. The corresponding equilibrium notion, the stationary mean-field equilibrium, consists of a policy--distribution pair satisfying the Nash certainty equivalence principle \citep{HuMaCa06}; under standard assumptions, existence follows from Kakutani's fixed-point theorem, and the resulting equilibrium approximates a Nash equilibrium in large but finite populations \citep{AdJoWe15}. 
For stationary and undiscounted variants of large-population dynamic games and mean-field models, see also \citet{WiAl05} and \citet{Wie19}.

%Mean-field games (MFGs) provide a framework for modelling strategic interactions in large populations of agents, where each individual responds to the aggregate behaviour of the population rather than to pairwise interactions \citep{HuMaCa06,LaLi07}.  In stationary settings, the population behaviour is summarised by a time-invariant state distribution \citep{WeBeRo05}, reducing the problem to a Markov decision process (MDP) constrained by a consistency condition on this distribution. The corresponding equilibrium notion, known as the stationary mean-field equilibrium (MFE), consists of a policy--distribution pair satisfying the {\color{red} Nash certainty equivalence principle} \citep{HuMaCa06}. Under standard assumptions, existence follows from Kakutani's fixed-point theorem, and the resulting equilibrium approximates a Nash equilibrium in large but finite populations \citep{AdJoWe15}. For stationary and undiscounted variants of large-population dynamic games and mean-field models, see also \citet{WiAl05} and \citet{Wie19}.

When the reward function is known, equilibria can be computed using classical MFG methods or learned through reinforcement learning techniques \citep{LaPePeGiMuElGePi24}. In many applications, however, the reward function is unknown and must be inferred from observed behaviour. Inverse reinforcement learning (IRL) addresses this problem by recovering a reward function that explains expert demonstrations, thereby improving interpretability and enabling generalisation to new environments \citep{AdCoBe22}. Classical single-agent IRL includes the maximum-margin apprenticeship method of \citet{abbeel2004apprenticeship} and the maximum-entropy framework of \citet{Ziebart2008}, and has since been enriched by Gaussian-process \citep{levine2011nonlinear} and adversarial \citep{ho2016generative,fu2018airl} formulations. The infinite-horizon variant is most cleanly handled by the maximum causal entropy principle of \citet{ZiBaDe10,ZiBaDe13}, which restores well-posedness by replacing trajectory-distribution maximisation with a causality-constrained problem reformulated over state--action occupation measures \citep{ZhBlBa18}; see \citet{GlTo22} for a unified exposition. In the single-agent average-reward setting, \citet{WuKeWu23} develop a maximum-entropy IRL framework with linear reward parametrisations and provide convergence guarantees, but their analysis is confined to ordinary MDPs and does not address the population-consistency constraint that distinguishes mean-field equilibria from MDP optima.

%When the reward function is known, equilibria can be computed using classical MFG methods or learned through reinforcement learning techniques \citep{LaPePeGiMuElGePi24}. In many applications, however, the reward function is unknown and must be inferred from observed behaviour. {\color{red} Inverse reinforcement learning (IRL) addresses this problem by recovering a reward function that explains expert demonstrations, thereby improving interpretability and enabling generalisation to new environments \citep{AdCoBe22}. Classical single-agent formulations include maximum-margin apprenticeship learning \citep{abbeel2004apprenticeship} and the maximum-entropy framework of \citet{Ziebart2008}. Beyond these early approaches, richer non-linear and non-parametric reward representations have been explored through Gaussian-process IRL \citep{levine2011nonlinear} as well as adversarial formulations such as GAIL \citep{ho2016generative} and AIRL \citep{fu2018airl}.}

For mean-field games (MFGs), several recent works have extended inverse reinforcement learning beyond the classical single-agent setting. \cite{YaYeTrXuZh18} reduce a cooperative MFG to an equivalent single-agent MDP and solve the resulting problem using maximum-entropy IRL. However, this reduction relies on the assumption that all agents share a common reward structure and therefore does not apply to genuinely decentralised competitive interactions. In contrast, \cite{YaLiLiHu22} study non-cooperative decentralised MFGs through a maximum-margin formulation, while \cite{ChZhLiWi23} combine entropy regularisation with the adversarial imitation-learning framework introduced in \cite{ho2016generative,fu2018airl}.

An alternative line of work focuses on imitation learning, where the objective is to directly recover policies from demonstrations rather than infer an underlying reward function implicitly or explicitly. Such approaches have also been investigated in the MFG setting; see \cite{GiPaOlNiMaMa23}. Although imitation learning often leads to computationally simpler formulations, IRL provides a more informative description of the underlying strategic behaviour by recovering the latent reward structure that explains the observed equilibria.

Closer in spirit to the present work, in our earlier paper \cite{AnKaSa24} we studied the IRL problem for discounted mean-field games under a finite-dimensional linear reward parametrisation. In that work, we developed an occupation-measure formulation of maximum causal entropy IRL and derived a convex dual problem involving a log-partition objective. We further established smoothness and strong convexity properties of the dual objective, which in turn yielded the consistency and convergence guarantees of the associated gradient-based optimisation algorithm. More recently, in \cite{AnKaKaSa25Kernel}, we extended this framework to RKHS reward models in the discounted setting, thereby allowing the reward function to belong to an infinite-dimensional non-parametric function class. In contrast to the finite-dimensional linear parametrisation, the RKHS formulation requires a different variational treatment. In particular, we introduced a Lagrangian relaxation whose dual characterisation is equivalent to the existence of a stationary point of a log-likelihood objective. To compute this stationary solution, we again developed a gradient-based optimisation scheme and established its consistency through the smoothness properties of the log-likelihood functional. This RKHS-based formulation significantly increases modelling flexibility. A preliminary version of the finite-dimensional average-reward formulation appeared in \citet{AlkirSaldiCDC25}, where only the linear reward model was considered.

The present paper extends both of the above approaches to the average-reward setting. Although the overall methodological framework remains conceptually similar to the discounted case, the average-reward regime introduces several substantial technical difficulties that require a fundamentally different analysis. In particular, the Bellman and soft Bellman optimality operators are no longer strict contractions in their standard form. To recover contraction properties, one must either work in the span semi-norm under suitable minorisation conditions or modify the transition kernel using the minorising measure to obtain a sup-norm contraction. Moreover, value functions and Q-functions are only characterised up to additive constants, which introduces additional invariances into both the primal and dual formulations. Another major difference is that, unlike the discounted setting where stability and occupation-measure representations are automatically induced by discounting, the average-reward framework requires additional ergodicity-type assumptions on the controlled dynamics to ensure the existence of invariant occupation measures. Consequently, several analytical properties that arise naturally in the discounted setting require a separate and more delicate treatment in the average-reward regime. Overall, the main goal of the present work is to develop a rigorous maximum causal entropy IRL framework for stationary mean-field games under the average-reward criterion, together with contraction-based methods for analysing the associated soft Bellman systems and establishing the consistency of the resulting gradient-based optimisation procedures. The following section provides a detailed overview of our main contributions.

\subsection{Contributions}\label{sec:contributions}
 
We develop a unified framework for IRL in discrete-time stationary MFGs under the long-run average-reward criterion. Starting from the maximum causal entropy formulation, we derive equivalent optimisation problems for two reward classes, finite-dimensional linear features and infinite-dimensional RKHS rewards, and analyse them within a single occupation-measure perspective. The specific contributions are as follows.
 
\begin{enumerate}[label=(\roman*)]
    \item \textbf{Linear rewards.} We give a convex dual reformulation of the maximum causal entropy problem over long-run occupation measures, with an explicit log-partition objective, and establish smoothness and strong-convexity properties of the dual that justify gradient descent with a constant step size. This extends \citet{WuKeWu23} from MDPs to MFGs.
    \item \textbf{RKHS rewards.} A Lagrangian relaxation leads to an unconstrained maximum log-likelihood formulation, where the induced policies satisfy a soft Bellman equation corresponding to reward functions parametrised by the Lagrange multipliers. We prove Fr\'{e}chet differentiability of the relevant operators with respect to the Lagrange multipliers and establish Lipschitz smoothness of the log-likelihood score, enabling first-order optimisation in function space. This is the average-reward counterpart of \citet{AnKaKaSa25Kernel}; the standard discounted contraction is unavailable here and is replaced by the construction in (iii).
    \item \textbf{Sub-stochastic contraction.} We construct a minorisation-based sub-stochastic transition kernel and prove that the associated soft Bellman operator is a strict contraction. This circumvents the additive-constant ambiguity of average-reward value functions and provides the analytical backbone for (ii).
    \item \textbf{Gradient-based algorithms.} Combining (i)--(iii), we obtain gradient methods with explicit constant step sizes for both reward classes, together with convergence guarantees. We validate the framework numerically on a malware-spread MFG and on an RKHS-based consumer-choice model, recovering policies that closely match the expert equilibrium in both cases.
\end{enumerate}

\subsection{Organisation}\label{sec:organisation}
 
Section~\ref{sec:preliminaries} fixes notation and recalls preliminaries on stationary MFGs, occupation measures, and the maximum causal entropy formulation. Section~\ref{sect:linear_reward} develops the finite-dimensional linear reward model: a convex occupation-measure reformulation, the explicit log-partition dual, and the smoothness/strong-convexity analysis underpinning constant-step-size gradient descent. Section~\ref{sec:rkhs} develops the RKHS reward model: the Lagrangian relaxation, the sub-stochastic-kernel construction that restores soft Bellman contraction, Fr\'echet differentiability of the Bellman fixed point, and the smoothness of the log-likelihood score. Section~\ref{sec:numerics} presents two numerical illustrations: a malware-spread MFG and an RKHS-based consumer-choice model. Section~\ref{sec:conclusion} concludes.

\section{Preliminaries}
\subsection{Mean-Field Games}\label{sec:preliminaries}

Mean-field games (MFGs) model decision-making in large populations where individual behaviour depends on the collective state. Throughout, we consider discrete-time stationary MFGs under the long-run average-reward criterion. Such a game is a tuple
$(\X, \A, p, r),$
where $\X$ and $\A$ are the finite \emph{state} and \emph{action} spaces; the transition kernel $p : \X \times \A \times \mathcal{P}(\X) \to \mathcal{P}(\X)$ governs the evolution of the next state; and the one-stage reward $r : \X \times \A \times \mathcal{P}(\X) \to \R$ specifies the immediate reward at each stage. Since $\X \times \A$ is finite and $r$ is continuous in $\mu$ on the compact simplex $\mathcal{P}(\X)$, the reward is uniformly bounded:
\[
\|r\|_\infty \coloneqq \sup_{(x,a,\mu)} |r(x, a, \mu)| < \infty.
\]

%The one-step transition probability \( p: \mathcal{X} \times \mathcal{A} \times \mathcal{P}(\mathcal{X}) \to \mathcal{P}(\mathcal{X}) \) determines how the next state evolves, while the one-stage reward function {\color{red} \( r: \mathcal{X} \times \mathcal{A} \times \mathcal{P}(\mathcal{X}) \to \mathbb{R} \)} specifies the immediate reward received. {\color{red} The reward function is uniformly bounded, i.e.\ $\|r\|_\infty := \sup_{(x,a,\mu)}|r(x,a,\mu)| < \infty$ due to finiteness of $\X\times\A$ and continuity of $r$ in $\mu$ on the compact simplex $\mathcal{P}(\X)$. } 

At each time $t$, the agent in state $x_t$ samples an action $a_t \sim \pi(\cdot \mid x_t)$, receives reward $r(x_t, a_t, \mu)$, and transitions according to $x_{t+1} \sim p(\cdot \mid x_t, a_t, \mu)$. A (stationary) \emph{policy} is a kernel $\pi : \X \to \mathcal{P}(\A)$; we denote the set of all such policies by $\Pi$. The \emph{mean-field term} $\mu \in \mathcal{P}(\X)$ represents the limiting empirical state distribution of the population as the number of agents tends to infinity \citep{CaDe18,AdJoWe15}.

In a \emph{time-dependent} MFG, the population distribution evolves as a sequence $\{\mu_t\}_{t \geq 0}$ in response to the agents' collective behaviour, leading to a coupled forward--backward system
\citep[][]{CaDe18,CaDe19}. In a \emph{stationary} MFG, the distribution is time-invariant, $\mu_t = \mu$ for all $t$, and the equilibrium is a joint fixed point of an optimality condition and an invariance condition. The invariance condition reads

\begin{equation}  
\mu(z) = \sum_{x \in \mathcal{X}} \sum_{a \in \mathcal{A}} \mu(x)\, \pi(a \mid x)\, p(z \mid x,a,\mu), \qquad z \in \mathcal{X}.
\label{stationarity}
\end{equation}
%The stationary framework provides several advantages: it simplifies analysis by reducing the problem to a single fixed-point equation; it aligns naturally with long-run performance criteria such as the \emph{average reward}.

%For a fixed population measure \( \mu \), a policy \( \pi \), and an {\color{red}initial state distribution} \( x_0 \sim \mu_0 \), we define the long-run \textit{average reward} as
For a fixed mean-field term $\mu$, a policy $\pi$, and initial distribution $x_0 \sim \mu_0$, the long-run \emph{average reward} is

\begin{equation}
\label{eq:cum-reward}
\mathcal{J}_{\mu}(\pi,\mu_0)=\liminf_{T\to\infty}\frac{1}{T}\sum_{t=0}^{T-1}\mathbb{E}^{\pi}\bigl[r(x_t,a_t,\mu)\bigr].
\end{equation}

This criterion will serve throughout as our reference for optimality.

%This criterion will serve as our reference for optimality.  

To formalise the equilibrium, we introduce two set-valued maps. The \emph{optimality map} $\Phi : \mathcal{P}(\X) \to 2^\Pi$ assigns to each population state the set of policies that maximise the long-run average reward against it,
\[
\Phi(\mu) \coloneqq \left\{ \pi \in \Pi : \J_\mu(\pi, \mu) = \sup_{\tilde\pi \in \Pi} \J_\mu(\tilde\pi, \mu) \right\}.
\]
The \emph{invariance map} $\Lambda : \Pi \to 2^{\mathcal{P}(\X)}$ assigns to each policy the set of population states that are stationary under~\eqref{stationarity},
\[
\Lambda(\pi) \coloneqq \left\{ \mu \in \mathcal{P}(\X) : \mu(x) = \sum_{(y, a) \in \X \times \A} p(x \mid y, a, \mu)\, \pi(a \mid y)\, \mu(y),\ \forall x \in \X \right\}.
\]

%To define the stationary mean-field equilibrium, we first introduce two set-valued mappings. The first one is $\Phi: \mathcal{P(\X)} \rightarrow 2^\Pi$, where $2^\Pi$ is the collection of all subsets of $\Pi$,  and is defined as
%\[\Phi(\mu) = \left\{ \pi \in \Pi : \mathcal{J}_{\mu}(\pi,\mu) = \sup_{\tilde{\pi} \in \Pi} \mathcal{J}_{\mu}(\tilde{\pi}, \mu) \right\}\]
%namely, the optimal policies for a fixed $\mu$. We also define the mapping $\Lambda : \Pi \to 2^{\mathcal{P}(\X)}$ by
%\[\Lambda(\pi)\coloneqq\left\{\mu \in \mathcal{P}(\X) :\mu(x)=\sum_{(y,a)\in \X\times \A} p(x\mid y,a,\mu)\pi(a\mid y)\mu(y),\ \forall x\in \X\right\}.\]

The equilibrium notion in a stationary MFG is obtained by combining
the optimality condition encoded by $\Phi$ with the invariance condition encoded
by $\Lambda$. 

\begin{definition}\label{def:MFE}
A pair $(\pi^\star, \mu^\star) \in \Pi \times \mathcal{P}(\X)$ is a \emph{mean-field equilibrium} (MFE) if $\pi^\star \in \Phi(\mu^\star)$ and $\mu^\star \in \Lambda(\pi^\star)$.
\end{definition}

 Under standard assumptions, existence of an MFE under the average-reward criterion follows from Kakutani's fixed-point theorem.

%The existence of a mean-field equilibrium can be established under standard assumptions for the average-reward criterion by applying {\color{red}Kakutani's fixed-point theorem}.}

\subsection{Inverse Reinforcement Learning}\label{sect:IRL}

The stationary MFG framework is well-suited to inverse reinforcement learning (IRL), because long-run state statistics under the expert equilibrium are governed by the invariant distribution $\mu_E$ and can be expressed in static form under a stationary policy. In \emph{forward} reinforcement learning, the reward $r$ is known and one seeks a policy maximising~\eqref{eq:cum-reward} for a given $\mu$. In \emph{inverse} reinforcement learning, the reward is unknown: the learner observes expert trajectories generated under some equilibrium $(\pi_E, \mu_E)$ and seeks a reward under which the expert's behaviour itself constitutes a MFE, equivalently a policy that reproduces the expert's long-run statistics under the inferred reward.

%The stationary mean-field game framework is particularly well-suited for inverse reinforcement learning (IRL), where observed trajectories often reflect equilibrium behaviour. In such cases, long-run state statistics are governed by the invariant distribution $\mu$, and the corresponding state–action behaviour can be represented in a static form under a stationary policy. This property allows for well-posed and tractable formulations of IRL objectives. While time-dependent formulations remain relevant for finite-horizon settings, this work focuses on the stationary case. In \emph{forward reinforcement learning}, the reward function $r$ is known and the objective is to learn a policy that maximises~\eqref{eq:cum-reward} given $\mu$. In contrast, \emph{inverse reinforcement learning} assumes that the reward is unknown. Instead, the agent observes the expert's trajectories under some mean-field equilibrium and seeks a reward under which the expert’s behaviour {\color{red}constitute a mean-field equilibrium}. In other words, the expert’s actions provide information about the reward function, helping the IRL {\color{red}agent understand} the main goal and replicate the expert behaviour under the inferred reward model.

To ensure that the long-run quantities appearing in the IRL formulation are well defined, we impose the following standard ergodicity assumption.

\begin{assumption}[Ergodicity]\label{ass:ergodicity}
For every policy $\pi \in \Pi$ and every mean-field term $\mu \in \mathcal{P}(\X)$, the induced Markov chain on $\X$ is positive Harris recurrent and aperiodic.
\end{assumption}
 
 Under Assumption~\ref{ass:ergodicity}, time-averaged quantities converge to expectations under the unique invariant distribution \citep[Theorem~13.3.3]{MeynTweedie}. Consequently, in what follows we write ``$\lim$'' rather than ``$\liminf$'' whenever long-run averages are taken under the fixed expert mean-field term $\mu_E$.

%\begin{assumption}
%The Markov process induced by any policy \( \pi \in \Pi \) and mean-field term \( \mu\) is assumed to be positive Harris recurrent and aperiodic. This ensures that time-averaged quantities converge to expectations under the invariant distribution \citep[Theorem~13.3.3]{MeynTweedie}.\label{ergodicity}
%\end{assumption}

%{\color{red} \begin{assumption}[Existence of Ces\`aro limits] For every stationary policy $\pi\in\Pi$, the chain $P_{\pi,\mu_E}(x'\mid x)\coloneqq\sum_{a\in\A}\pi(a\mid x)\,p(x'\mid x,a,\mu_E)$ admits a unique Ces\`aro-invariant probability measure $\eta_\pi$: $$\eta_\pi(x)\;=\;\lim_{T\to\infty}\frac{1}{T}\sum_{t=0}^{T-1}\mathbb{P}^{\pi,\mu_E}_{\rho_0}(x_t=x),$$ which exists and is independent of the initial distribution $\rho_0$. \label{ergodicity} \end{assumption} }

%\begin{remark} Under Assumption~\ref{ergodicity}, the time-average limits appearing below are well-defined for every stationary policy $\pi \in \Pi$ under the expert mean-field term $\mu_E$. Hence, in what follows, we write ``$\lim$'' instead of ``$\liminf$'' whenever we consider long-run averages with fixed mean-field term $\mu_E$. \end{remark}}

The learner observes a collection of expert-generated trajectories $\{(x_t^{(i)}, a_t^{(i)})_{t \geq 0}\}_{i=1}^d$ drawn from the expert MFE $(\pi_E, \mu_E)$. To formulate the inverse problem, we fix a feature map
\[
\varphi : \X \times \A \times \mathcal{P}(\X) \to \mathcal{E},
\]
where $\mathcal{E}$ is a real Hilbert space, instantiated as $\R^k$ in the linear case (Section~\ref{sect:linear_reward}) and as an RKHS $\mathcal{H}_k$ in the kernel case (Section~\ref{sec:rkhs}). The feature map summarises the reward-relevant content of behaviour; matching its long-run statistics will be the central constraint of the IRL formulation.

\begin{definition}\label{def:feature_exp_vec}
For a stationary policy $\pi$, the \emph{feature expectation vector} under the expert mean-field term $\mu_E$ is
\[
\langle \varphi \rangle_{\pi, \mu_E} \coloneqq \lim_{T \to \infty} \frac{1}{T} \sum_{t=0}^{T-1} \E^{\pi, \mu_E} \big[ \varphi(x_t, a_t, \mu_E) \big], \qquad x_0 \sim \mu_E.
\]
\end{definition}
 
\noindent Under Assumption~\ref{ass:ergodicity}, both $\mu_E$ and $\langle \varphi \rangle_{\pi_E, \mu_E}$ admit consistent empirical estimators from expert trajectories: the ergodic theorem \citep[Theorem~13.3.3]{MeynTweedie} guarantees that the empirical state distribution and feature average computed over a single long trajectory converge almost surely to their stationary counterparts as $T \to \infty$, and using multiple independent trajectories reduces finite-sample variance. As is standard in the MFG-IRL literature \citep{YaLiLiHu22, ChZhLiWi23}, we treat these long-run statistics as given in the optimisation problem.
 
\begin{assumption}[Known expert statistics]\label{ass:known_stats}
The expert mean-field term $\mu_E$ and the corresponding feature expectation vector $\langle \varphi \rangle_{\pi_E, \mu_E}$ are known.
\end{assumption}

\subsection{Occupation Measures}\label{sect:occupation_measures}

To express the long-run constraints of the IRL problem in static form, we work with state--action occupation measures. These objects encode the asymptotic visitation frequencies induced by a stationary policy under the expert mean-field term, and they are used throughout Sections~\ref{sect:linear_reward} and~\ref{sec:rkhs}: in the linear model to reformulate the IRL problem as an optimisation over measures, and in the kernel model to compactly express expert statistics and gradient identities.

%To express the long-run constraints of the IRL problem in a static form, it is convenient to work with \emph{occupation measures}. These objects encode the asymptotic state-action frequencies induced by a stationary policy under the expert mean-field term $\mu_E$. They will be used in both parts of the paper. In the linear model, they allow us to reformulate the problem as an optimisation over measures. In the kernel-based model, they provide a compact way to represent expert statistics and will later simplify several gradient identities.

%We therefore begin by introducing the state--action occupation measure$\nu_\pi$ and its state marginal $\nu_\pi^{\mathcal X}$ for a stationary policy$\pi \in \Pi$. Roughly speaking, $\nu_\pi(x,a)$ measures the long-run frequency with which the pair $(x,a)$ is visited, while $\nu_\pi^{\mathcal X}(x)$ records the corresponding state frequency. Under Assumption~\ref{ass:ergodicity}, these limits are well defined.

\begin{definition}\label{def:occupation_measures}
For $\pi \in \Pi$, the \emph{state--action occupation measure} under $\mu_E$ is
\[
\nu_\pi(x, a) \coloneqq \lim_{T \to \infty} \frac{1}{T} \sum_{t=0}^{T-1} \E^{\pi, \mu_E} \big[ \1_{\{(x_t, a_t) = (x, a)\}} \big],
\]
and its state marginal is $\nu_\pi^\X(x) \coloneqq \sum_{a \in \A} \nu_\pi(x, a)$.
\end{definition}

Under Assumption~\ref{ass:ergodicity}, the limit exists and is non-negative for each $(x, a)$. Moreover, $\nu_\pi$ defines a probability measure on $\X \times \A$ (see Appendix~\ref{app:nu-prob}).

%\begin{definition}
 %   \label{def:occupation_measures}
  %  The state--action occupation measure $\nu_\pi$ for any policy $\pi \in \Pi$ is defined as
  %\begin{equation*} \nu_\pi(x,a) \coloneqq \limT \E^{\pi,\mu_E}\left[\1_{\left\{ (x_t,a_t) =(x,a) \right\}}  \right]. \end{equation*}
 %   The corresponding state occupation measure is defined by  \[  \nu_\pi^{\X}(x) \coloneqq \sum_{a\in \A} \nu_\pi(x,a). \]
%\end{definition}
%The next observation shows that \(\nu_\pi\) is a probability measure on \(\X\times\A\).  Indeed, Assumption~\ref{ass:ergodicity} ensures that, for each $(x,a)\in\mathcal{X}\times\mathcal{A}$, the limit
%\[ \nu_\pi(x,a) = \lim_{T\to\infty}\frac{1}{T}\sum_{t=0}^{T-1} \mathbb{P}^{\pi,\mu_E}\big((x_t,a_t)=(x,a)\big) \]
%exists and is non-negative. Summing over all $(x,a)$ gives
%\[ \sum_{x\in\mathcal{X}}\sum_{a\in\mathcal{A}}\nu_\pi(x,a) = \lim_{T\to\infty}\frac{1}{T}\sum_{t=0}^{T-1} \sum_{x\in\mathcal{X}}\sum_{a\in\mathcal{A}} \mathbb{P}^{\pi,\mu_E}\big((x_t,a_t)=(x,a)\big) =1, \]
%since for each $t$ the inner sum equals one. Therefore, $\nu_\pi$ is non-negative and has unit total mass.

The three identities below decompose $\nu_\pi$ into its policy and state-flow components, and express the feature expectation vector in static form. Each statement is a routine consequence of stationarity and the finiteness of $\X \times \A$. We record them here for later use in Sections~\ref{sect:linear_reward} and~\ref{sec:rkhs}, and defer the proofs to Appendix~\ref{app:occupation-proofs}.

The first identity states that the occupation measure factorises into a state marginal and a conditional action distribution in the natural way.

\begin{lemma}[Policy factorisation]\label{lem:occ-factorization}
For every $\pi \in \Pi$ and every $(x, a) \in \X \times \A$,
\[
\nu_\pi(x, a) = \pi(a \mid x)\, \nu_\pi^\X(x).
\]
\end{lemma}

The next lemma gives the corresponding balance equation for the state marginal, induced by the transition kernel.

\begin{lemma}[Flow balance]\label{lem:flow}
For every $\pi \in \Pi$ and every $x \in \X$,
\[
\nu_\pi^\X(x) = \sum_{(y, a) \in \X \times \A} p(x \mid y, a, \mu_E)\, \nu_\pi(y, a).
\]
\end{lemma}

Together, these identities show that $\nu_\pi$ encodes both the policy and the induced state flow. We now use the same representation to rewrite the long-run feature average in static form.

\begin{lemma}[Static feature representation]\label{lem:feature-exp}
For every $\pi \in \Pi$,
\[
\langle \varphi \rangle_{\pi, \mu_E} = \sum_{(x, a) \in \X \times \A} \varphi(x, a, \mu_E)\, \nu_\pi(x, a).
\]
\end{lemma}

\subsection{Maximum Causal Entropy Principle}\label{sect:MaxEnt}

In IRL, the expert's trajectories generally do not uniquely determine a reward: many reward functions can explain the same observed behaviour. Rather than estimating the reward directly, we follow the maximum causal entropy principle of \citet{Ziebart2008,ZiBaDe10,ZiBaDe13}. Among all policies whose long-run feature statistics match the expert's, we select the one of maximum causal entropy, thereby choosing the least biased explanation of the data. The reward is then recovered implicitly through the feature-matching constraint, and the resulting policy provides a consistent and tractable representation of the expert's behaviour. The qualifier \emph{causal} indicates that each action $a_t$ is conditioned on the past states $x_{0:t}$ only, ruling out the use of future trajectory information.

%{\color{blue}We formulate the inverse problem as a constrained entropy maximization problem. In IRL, the reward function is unknown, and the observed expert behavior typically does not uniquely determine it; that is, many reward functions can explain the same trajectories. Rather than estimating the reward directly, we proceed indirectly through the maximum causal entropy principle of \citet{Ziebart2008}. Specifically, we match the expected feature statistics of the model to those observed from the expert, which implicitly restricts the set of admissible reward functions under a given feature map. Among all policies that satisfy these moment-matching constraints, we select the one that maximizes causal entropy, thereby choosing the least biased explanation of the data. In this way, the reward is inferred implicitly through the feature expectations, while the resulting policy provides a consistent and tractable representation of the expert’s behavior.}

\begin{definition}\label{def:entropy}
The \emph{average-reward causal entropy} of $\pi \in \Pi$ is
\[
H(\pi) \coloneqq \lim_{T \to \infty} \frac{1}{T} \sum_{t=0}^{T-1} \E^{\pi, \mu_E} \big[ -\log \pi(a_t \mid x_t) \big].
\]
\end{definition}
 We adopt the conventions $0 \log 0 \coloneqq 0$ and $0 \log(0 / 0) \coloneqq 0$ throughout. Under Assumption~\ref{ass:ergodicity}, the limit exists and is finite.

We define the \emph{average-reward maximum causal entropy IRL problem} $\opt$ as:

\[
\begin{array}{ll}
\opt \quad \underset{\pi \in \Pi}{\text{maximise}} & H(\pi) \\[0.4em]
\phantom{\opt \quad} \text{subject to} & \displaystyle \mu_E(x) = \sum_{(y, a) \in \X \times \A} p(x \mid y, a, \mu_E)\, \pi(a \mid y)\, \mu_E(y), \quad \forall x \in \X, \\[0.9em]
& \displaystyle \lim_{T \to \infty} \frac{1}{T} \sum_{t=0}^{T-1} \E^{\pi, \mu_E} \big[\varphi(x_t, a_t, \mu_E) \big] = \langle \varphi \rangle_{\pi_E, \mu_E}.
\end{array}
\]
The first constraint enforces the invariance of the expert mean-field term $\mu_E$ under the policy $\pi$; the second matches the long-run feature expectation induced by $(\pi, \mu_E)$ to that of the expert pair $(\pi_E, \mu_E)$.

Two consequences of feasibility will be used repeatedly in the sequel. The first identifies the state marginal of $\nu_\pi$ with $\mu_E$; the second rewrites the causal entropy in static occupation-measure form. Proofs are deferred to Appendix~\ref{app:opt-proofs}.
\begin{lemma}\label{lem:muE_nu_marginal}
If $\pi \in \Pi$ is feasible for $\opt$, then $\nu_\pi^\X(x) = \mu_E(x)$ for all $x \in \X$.
\end{lemma}
 
\begin{lemma}\label{lem:re-writing-entropy}
If $\pi \in \Pi$ is feasible for $\opt$, then
\[
H(\pi) = \sum_{(x, a) \in \X \times \A} -\log \!\left( \frac{\nu_\pi(x, a)}{\mu_E(x)} \right) \nu_\pi(x, a).
\]
\end{lemma}
 
The formulation $\opt$ is now specialised in two directions. Section~\ref{sect:linear_reward} considers the finite-dimensional linear reward model and uses the occupation-measure representation to derive a convex dual problem with an explicit log-partition objective. Section~\ref{sec:rkhs} considers the RKHS reward model and pursues a Lagrangian-relaxation route based on soft Bellman equations.

\section{Finite-Dimensional Linear Reward Model}\label{sect:linear_reward}

We now specialise the optimisation problem $\opt$ introduced in Section~\ref{sect:MaxEnt} to the finite-dimensional linear setting. Specifically, we assume the expert's long-run statistics are represented by a feature map taking values in $\R^k$, and that the unknown reward function is linear with respect to this feature vector.

\begin{assumption}[Finite-dimensional linear reward model]\label{ass:linear_reward}
The feature map of Section~\ref{sect:IRL} takes values in $\R^k$, i.e.\ $\varphi : \X \times \A \times \mathcal{P}(\X) \to \R^k$, and the unknown reward belongs to the linearly parametrised class
\[
\mathcal{R} \coloneqq \big\{ r(x, a, \mu) = \langle \theta, \varphi(x, a, \mu) \rangle : \theta \in \R^k \big\}.
\]
\end{assumption}
 Under Assumption~\ref{ass:linear_reward}, the expert feature expectation vector $\langle \varphi \rangle_{\pi_E, \mu_E}$ lies in $\R^k$, and the problem $\opt$ inherits a finite-dimensional feature representation. We begin by showing that any optimal solution of $\opt$ induces a MFE.

\iffalse
\begin{assumption}[Finite-dimensional linear reward model]
\label{ass:linear_reward}
Throughout this section, we specialise the feature map introduced in
Section~\ref{sect:IRL} to a map $\varphi : \X \times \A \times \mathcal P(\X) \to \R^k,$ and assume that the unknown reward belongs to the linearly parametrised class
\[
\mathcal{R} \coloneqq \left\{ r(x,a,\mu) = \langle \theta, \varphi(x,a,\mu) \rangle \mid \theta \in \mathbb{R}^k,\ \varphi : \mathcal{X} \times \mathcal{A} \times \mathcal{P}(\mathcal{X}) \to \mathbb{R}^k \right\}.
\]
\end{assumption}
\fi

\begin{proposition}\label{prop:linear-mfe}
Suppose Assumption~\ref{ass:linear_reward} holds and let $\pi^\star$ be an optimal solution of $\opt$. Then $(\pi^\star, \mu_E)$ is a mean-field equilibrium.
\end{proposition}

\begin{proof}
Since $\pi^\star$ is feasible for $\opt$, the first constraint of $\opt$ gives $\mu_E \in \Lambda(\pi^\star)$. 
Let $r_E(x, a, \mu) = \langle \theta_E, \varphi(x, a, \mu) \rangle$ denote the expert's (unknown) reward under Assumption~\ref{ass:linear_reward}. Since $(\pi_E, \mu_E)$ is a MFE, $\pi_E \in \Phi(\mu_E)$, i.e.\ $\mathcal{J}_{\mu_E}(\pi_E,\mu_E)=\sup_{\pi\in\Pi}\mathcal{J}_{\mu_E}(\pi,\mu_E).$ By linearity of the reward and the feature-matching constraint in $\opt$,
\[
\mathcal{J}_{\mu_E}(\pi^\star,\mu_E)
=
\left\langle \theta_E,\langle \varphi\rangle_{\pi^\star,\mu_E}\right\rangle
=
\left\langle \theta_E,\langle \varphi\rangle_{\pi_E,\mu_E}\right\rangle
=
\mathcal{J}_{\mu_E}(\pi_E,\mu_E),
\]
so $\pi^\star \in \Phi(\mu_E)$. Combined with $\mu_E \in \Lambda(\pi^\star)$, this proves the claim.
\end{proof}

\iffalse
\begin{proof}
Since $\pi^\star$ is feasible for $\opt$, it satisfies
\[
\mu_E(x)=\sum_{(y,a)\in \X\times \A} p(x\mid y,a,\mu_E)\pi^\star(a\mid y)\mu_E(y),
\qquad \forall x\in \X.
\]
Hence, $\mu_E\in \Lambda(\pi^\star)$. 

Now let $r_E(x,a,\mu)=\langle \theta_E,\varphi(x,a,\mu)\rangle$ denote the expert's unknown reward under Assumption~\ref{ass:linear_reward}. Since $(\pi_E,\mu_E)$ is a
mean-field equilibrium, we have $\pi_E\in \Phi(\mu_E)$, that is,
\[
\mathcal{J}_{\mu_E}(\pi_E,\mu_E)=\sup_{\pi\in\Pi}\mathcal{J}_{\mu_E}(\pi,\mu_E).
\]
Moreover, by the linearity of the reward and the feature-matching constraint in $\opt$,
\[
\mathcal{J}_{\mu_E}(\pi^\star,\mu_E)
=
\left\langle \theta_E,\langle \varphi\rangle_{\pi^\star,\mu_E}\right\rangle
=
\left\langle \theta_E,\langle \varphi\rangle_{\pi_E,\mu_E}\right\rangle
=
\mathcal{J}_{\mu_E}(\pi_E,\mu_E).
\]
Therefore,
\[
\mathcal{J}_{\mu_E}(\pi^\star,\mu_E)=\sup_{\pi\in\Pi}\mathcal{J}_{\mu_E}(\pi,\mu_E),
\]
which implies that $\pi^\star\in\Phi(\mu_E)$. Consequently,
$(\pi^\star,\mu_E)$ is a mean-field equilibrium.
\end{proof}
\fi

\subsection{Occupation-Measure Reformulation}\label{sec:linear-occupation} 
The identities of Section~\ref{sect:occupation_measures} allow us to pass from the policy formulation $\opt$ to an equivalent optimisation over occupation measures. By Lemmas~\ref{lem:feature-exp}, \ref{lem:muE_nu_marginal}, and~\ref{lem:re-writing-entropy}, every feasible policy for $\opt$ induces an occupation measure satisfying the corresponding feature-matching, flow-balance, and entropy identities. This yields:

%We now derive an equivalent occupation-measure formulation of the linear problem. The identities from Section~\ref{sect:occupation_measures} allow us to pass {\color{red} from the original policy formulation} 
%from the root policy formulation  to an optimisation problem over measures.

%By Lemma~\ref{lem:feature-exp}, Lemma~\ref{lem:muE_nu_marginal}, and Lemma~\ref{lem:re-writing entropy}, every feasible policy for $\opt$
%induces an occupation measure satisfying the corresponding feature-matching, flow-balance, and entropy identities. This leads to the following reformulation:

\[
\begin{array}{rl}
\optt \quad \displaystyle \max_{\nu \in \R^{\X \times \A}}
& \displaystyle \sum_{(x,a) \in \X \times \A} -\log\!\left( \frac{\nu(x,a)}{\mu_E(x)} \right) \nu(x,a) \\[1.0em]
\text{subject to}
& \displaystyle \sum_{(x,a) \in \X \times \A} \varphi(x, a, \mu_E)\, \nu(x, a) = \langle \varphi \rangle_{\pi_E, \mu_E}, \\[0.7em]
& \displaystyle \mu_E(x) = \sum_{(y,a) \in \X \times \A} p(x \mid y, a, \mu_E)\, \nu(y, a), \quad \forall x \in \X, \\[0.5em]
& \nu^{\X}(x) = \mu_E(x), \quad \forall x \in \X, \\[0.3em]
& \nu(x, a) \geq 0, \quad \forall (x, a) \in \X \times \A.
\end{array}
\]

\iffalse
\[
\begin{array}{lll}
\optt \,\, \underset{\nu \in \R^{\X \times \A}}{\text{maximise}} \text{ } 
& \displaystyle \sum_{(x,a) \in \X \times \A} -\log\left(\frac{\nu(x,a)}{\mu_E(x)}\right) \, \nu(x,a)
\\
\phantom{\mathbf{(OPT_2)} \,\, } \text{subject to} 
& \displaystyle \sum_{(x,a) \in \X \times \A} \varphi(x,a,\mu_E) \, \nu(x,a) = \langle \varphi \rangle_{\pi_E,\mu_E}
\\
\phantom{x} 
& \mu_E(x) = \sum_{(y,a) \in \X \times \A} p(x\mid y,a,\mu_E) \, \nu(y,a), \qquad \forall x \in \X
\\
\phantom{x} 
& \nu^{\X}(x) = \mu_E(x), \qquad\forall x \in \X
\\
\phantom{x} 
& \nu(x,a) \geq 0, \qquad \forall (x,a) \in \X \times \A.
\end{array}
\]
\fi

The problem above is the occupation-measure counterpart of
the original maximum causal entropy formulation $\opt$ from Section~\ref{sect:MaxEnt}.

We next show that $\optt$ is convex and is equivalent to $\opt$.

%We next show that it is a convex optimisation problem and that it is equivalent to $\opt$.

\begin{theorem}\label{thm:linear_convex}
    $\optt$ is a convex optimisation problem.
\end{theorem}

\begin{proof}
Rewrite the objective as
\[
\sum_{(x, a) \in \X \times \A} -\log(\nu(x, a))\, \nu(x, a)
+ \sum_{(x, a) \in \X \times \A} \log(\mu_E(x))\, \nu(x, a).
\]
The first term is the entropy of $\nu$, which is strongly concave in $\nu$ \citep[p.~37]{alajaji2018}; the second is linear in $\nu$. Hence the objective is strongly concave. Since all constraints are linear in $\nu$, the problem is convex.
\end{proof}

\begin{theorem}\label{thm:linear_correspondence}
The map $\pi \mapsto \nu_{\pi}$ defined by $\nu_{\pi}(x, a) = \mu_{E}(x)\pi(a \mid x)$ is a bijection from the feasible set of $\opt$ onto the feasible set of $\optt$, with inverse $\nu \mapsto \pi_{\nu}(a \mid x) = \nu(x, a)/\nu^{ X}(x)$ on $\{\nu^{ X} > 0\}$; under this correspondence, the objective values of $\opt$ and $\optt$ coincide.
\end{theorem}

The proof is a routine application of the occupation-measure identities of Section~\ref{sect:occupation_measures} together with the feasibility-induced marginal identity of Lemma~\ref{lem:muE_nu_marginal}; the full argument is given in Appendix~\ref{app:proof-linear-correspondence}.

\begin{remark}\label{rem:linear-correspondence}
The correspondence between feasible policies of $\opt$ and feasible occupation measures of $\optt$ is given by
\begin{align*}
    &\pi \mapsto \nu_\pi \coloneqq \mu_E(x)\, \pi(a \mid x) \in \mathcal{P}(\mathcal{X} \times \mathcal{A}) \quad\text{and} \quad\nu \mapsto \pi_\nu(a \mid x) \coloneqq \frac{\nu(x,a)}{\nu^{\mathcal{X}}(x)} \in \Pi,
\end{align*}
whenever $\nu^{\X}(x)>0$. That is, $\nu_\pi$ is the occupation measure induced
by the policy $\pi$ under the invariant distribution $\mu_E$, while $\pi_\nu$
is obtained from $\nu$ by normalising with respect to its state marginal.

Thus the two formulations are equivalent at the level of feasible solutions, not only in objective value.
\end{remark}

\noindent In view of Theorem~\ref{thm:linear_correspondence}, we work entirely with $\optt$ in the remainder of this section.

\subsection{Dual Formulation}\label{sec:linear-dual}

 We now derive a min--max representation of $\optt$ by introducing Lagrange multipliers for its linear constraints. Solving the inner maximisation in closed form will then yield the dual problem.

%We next derive a min--max representation of $\optt$ by introducing Lagrange multipliers for its linear constraints. This will serve as the starting point for the dual analysis.

\begin{proposition}
\label{prop:minmax-linear}
Let $\boldsymbol{\alpha} \in \R^k$ and $\boldsymbol{\beta},\boldsymbol{\theta} \in \R^\X$ denote the Lagrange
multipliers associated with the feature-matching, flow-balance, and marginal
constraints in $\optt$, respectively. Then 
\small
\begin{align*}
\optt
=
\max_{\nu \in \mathcal P(\X \times \A)}
\min_{\substack{\boldsymbol{\alpha} \in \R^k\\ \boldsymbol{\beta},\boldsymbol{\theta} \in \R^\X}}
\bigg\{
H(\nu)
+
\sum_{(x,a)\in \X\times \A}
\ell_{\boldsymbol{\alpha},\boldsymbol{\beta},\boldsymbol{\theta}}(x,a)\nu(x,a) -
\langle \boldsymbol{\alpha},\langle \varphi\rangle_{\pi_E,\mu_E}\rangle
-
\sum_{x\in \X}\boldsymbol{\theta}_x\mu_E(x)
\bigg\},
\end{align*}
\normalsize
where
\[
\ell_{\boldsymbol{\alpha},\boldsymbol{\beta},\boldsymbol{\theta}}(x,a)
\coloneqq
\log\mu_E(x)
+
\langle \boldsymbol{\alpha},\varphi(x,a,\mu_E)\rangle
+
\boldsymbol{\theta}_x
+
\sum_{z\in \X}\boldsymbol{\beta}_z\bigl(p(z\mid x,a,\mu_E)-\mu_E(z)\bigr).
\]
\end{proposition}

The derivation expands the Lagrangian, collects the terms linear in $\nu$, and absorbs the constant $-\sum_z \bm{\beta}_z\, \mu_E(z)$ into the coefficient of $\nu(x, a)$ using $\sum_{(x, a)} \nu(x, a) = 1$; the bookkeeping is recorded in Appendix~\ref{app:proof-minmax-linear}. For fixed $(\bm{\alpha}, \bm{\beta}, \bm{\theta})$, the displayed expression is concave in $\nu$ (since $H(\nu)$ is concave and the remaining terms are linear); for fixed $\nu$, it is affine, hence convex, in $(\bm{\alpha}, \bm{\beta}, \bm{\theta})$. This convex--concave structure is the input to Sion's minimax theorem in the next result.

\begin{theorem}\label{thm:linear_dual}
The dual of $\optt$ is
\[
\min_{\substack{\bm{\alpha} \in \R^k \\ \bm{\beta}, \bm{\theta} \in \R^\X}}
\bigg\{
\log \sum_{(x, a) \in \X \times \A} e^{\ell_{\bm{\alpha}, \bm{\beta}, \bm{\theta}}(x, a)}
- \langle \bm{\alpha}, \langle \varphi \rangle_{\pi_E, \mu_E} \rangle
- \sum_{x \in \X} \bm{\theta}_x\, \mu_E(x)
\bigg\},
\]
with $\ell_{\bm{\alpha}, \bm{\beta}, \bm{\theta}}$ as in Proposition~\ref{prop:minmax-linear}. For each $(\bm{\alpha}, \bm{\beta}, \bm{\theta})$, the inner maximisation is attained at the Boltzmann distribution
\[
\nu^\star(x, a) \coloneqq \frac{e^{\ell_{\bm{\alpha}, \bm{\beta}, \bm{\theta}}(x, a)}}{Z_{\bm{\alpha}, \bm{\beta}, \bm{\theta}}},
\qquad
Z_{\bm{\alpha}, \bm{\beta}, \bm{\theta}} \coloneqq \sum_{(x, a) \in \X \times \A} e^{\ell_{\bm{\alpha}, \bm{\beta}, \bm{\theta}}(x, a)}.
\]
Furthermore, there is no duality gap.
\end{theorem}

\iffalse
\begin{remark}
The representation above will be the starting point of the dual analysis. For
fixed $(\boldsymbol{\alpha},\boldsymbol{\beta},\boldsymbol{\theta})$, the function is concave in $\nu$, since $H(\nu)$
is concave and the remaining terms are linear in $\nu$. On the other hand, for
fixed $\nu$, it is affine, hence convex, in $(\boldsymbol{\alpha},\boldsymbol{\beta},\boldsymbol{\theta})$.
\end{remark}
The min--max representation above allows us to derive the dual problem
explicitly. The next theorem shows that the dual objective is obtained by
solving the inner maximisation over occupation measures in closed form.

\begin{theorem}\label{thm:linear_dual}
    The dual of $\optt$ is 
    \[
    \min\limits_{\substack{\boldsymbol{\alpha} \in \mathbb{R}^k \\ \boldsymbol{\beta}, \boldsymbol{\theta} \in \mathbb{R}^{\mathcal{X}}}} \bigg\{ \log \sum_{(x,a)\in\X \times \A}e^{\ell_{\boldsymbol{\alpha,\beta,\theta}}(x,a)} - \langle \boldsymbol{\alpha}, \, \langle \varphi\rangle_{\pi_E,\mu_E}  \rangle - \sum_{x \in \X} \boldsymbol{\theta_x} \mu_E(x) \bigg \}
    \]
    where
    \[
    \ell_{\boldsymbol{\alpha,\beta,\theta}}(x,a) \coloneqq \log(\mu_E(x)) + \langle\boldsymbol{\alpha}, \, \varphi(x,a,\mu_E)\rangle + \boldsymbol{\theta_x} + \sum_{z\in\X} \boldsymbol{\beta_z}(p(z \; |\; x,a,\mu_E) - \mu_E(z)) .
    \]
\end{theorem}
\fi

\begin{proof}
By Proposition~\ref{prop:minmax-linear}, $\optt$ has the displayed min--max representation. Since $\mathcal{P}(\X \times \A)$ is compact and convex and the Lagrangian is concave in $\nu$ and affine in $(\bm{\alpha}, \bm{\beta}, \bm{\theta})$, Sion's minimax theorem~\citep{Sion1958} permits the interchange of $\max$ and $\min$:
\begin{align*}
\optt &= \min_{\substack{\boldsymbol{\alpha} \in \R^k\\ \boldsymbol{\beta},\boldsymbol{\theta} \in \R^\X}} \bigg\{
\max_{\nu \in \mathcal{P}(\X \times \A)} \Big[ H(\nu) + \sum_{(x, a)} \ell_{\bm{\alpha}, \bm{\beta}, \bm{\theta}}(x, a)\, \nu(x, a) \Big] \\
&\hspace{4.20cm}
- \langle \bm{\alpha}, \langle \varphi \rangle_{\pi_E, \mu_E} \rangle
- \sum_{x} \bm{\theta}_x\, \mu_E(x) \bigg\}.
\end{align*}
The variational identity \citep[Prop.~1.4.2]{PaulDupuis1997}
\[
\log \sum_{t \in T} e^{\ell(t)}
= \max_{\nu \in \mathcal{P}(T)} \left[ H(\nu) + \sum_{t \in T} \ell(t)\, \nu(t) \right]
\]
evaluates the inner maximum to $\log \sum_{(x, a)} e^{\ell_{\bm{\alpha}, \bm{\beta}, \bm{\theta}}(x, a)}$, which gives the displayed dual. The same identity shows that the maximiser is the Boltzmann distribution $\nu^\star$ in the statement. The absence of a duality gap is the conclusion of Sion's theorem.
\end{proof}

Theorem~\ref{thm:linear_dual} reduces the linear inverse problem to a finite-dimensional dual optimisation
problem. For convenience, we denote the dual objective by
\begin{equation}\label{def:linear_dual_h}
h(\boldsymbol{\alpha},\boldsymbol{\beta},\boldsymbol{\theta})
\coloneqq
\log \sum_{(x,a)\in \X\times\A} e^{\ell_{\boldsymbol{\alpha},\boldsymbol{\beta},\boldsymbol{\theta}}(x,a)}
-\big\langle \boldsymbol{\alpha},\langle \varphi\rangle_{\pi_E,\mu_E}\big\rangle
-\sum_{x\in\X}\boldsymbol{\theta}_x\,\mu_E(x).
\end{equation}
The regularity of $h$ is the subject of the next subsection.

\subsection{Regularity of the Dual Objective}\label{sec:linear-regularity}

We now establish smoothness and strong convexity of $h$. Together, these two properties justify gradient descent with a constant step size.

\begin{theorem}[Lipschitz smoothness of $h$]\label{thm:linear_smoothness}
The function $h$ is $L$-smooth with
\[
L \coloneqq 6 M^2 \sqrt{|\X| \cdot |\A|}, \qquad M \coloneqq \max\{M_{\bm{\alpha}}, M_{\bm{\beta}}, M_{\bm{\theta}}\},
\]
where
\[
M_{\bm{\alpha}} \coloneqq \sup_{(x, a)} \|\varphi(x, a, \mu_E)\|, \quad
M_{\bm{\beta}} \coloneqq \sup_{(x, a)} \|p(\cdot \mid x, a, \mu_E) - \mu_E(\cdot)\|, \quad
M_{\bm{\theta}} \coloneqq \sup_{x} \|e_x\|,
\]
and $e_x(y) \coloneqq \1_{\{x = y\}}$. Each of these constants is finite by the finiteness of $\X \times \A$.
\end{theorem}

\begin{proof}
We first compute the partial gradients of \(h\):
\begin{align*}
\nabla_{\boldsymbol{\alpha}}h(\boldsymbol{\alpha},\boldsymbol{\beta},\boldsymbol{\theta})
&=
\sum_{(x,a)\in\X\times\A}
\nu^\star(x,a)\varphi(x,a,\mu_E)
-
\langle\varphi\rangle_{\pi_E,\mu_E},\\
\nabla_{\boldsymbol{\beta}}h(\boldsymbol{\alpha},\boldsymbol{\beta},\boldsymbol{\theta})
&=
\sum_{(x,a)\in\X\times\A}
\nu^\star(x,a)
\big(p(\cdot\mid x,a,\mu_E)-\mu_E(\cdot)\big),\\
\nabla_{\boldsymbol{\theta}}h(\boldsymbol{\alpha},\boldsymbol{\beta},\boldsymbol{\theta})
&=
\sum_{(x,a)\in\X\times\A}
\nu^\star(x,a)e_x(\cdot)
-
\mu_E(\cdot),
\end{align*}
where \(e_x(y)\coloneqq \1_{\{x=y\}}\). Thus, the dependence of
\(\nabla h\) on the parameters enters through the Boltzmann distribution
\(\nu^\star\). For
\(\gamma\in\{\boldsymbol{\alpha},\boldsymbol{\beta},\boldsymbol{\theta}\}\),
differentiating \(\nu^\star\) gives
\[
\nabla_{\gamma}\nu^\star(x,a)
=
\nu^\star(x,a)
\left(
\nabla_{\gamma}\ell_{\boldsymbol{\alpha},\boldsymbol{\beta},\boldsymbol{\theta}}(x,a)
-
\left\langle
\nabla_{\gamma}\ell_{\boldsymbol{\alpha},\boldsymbol{\beta},\boldsymbol{\theta}}
\right\rangle_{\nu^\star}
\right).
\]

\noindent Notice that
\begin{align*}
\sup_{(x,a)\in\X\times\A}
\left\|
\nabla_{\boldsymbol{\alpha}}
\ell_{\boldsymbol{\alpha},\boldsymbol{\beta},\boldsymbol{\theta}}(x,a)
\right\|
&=
\sup_{(x,a)\in\X\times\A}
\|\varphi(x,a,\mu_E)\|
=:M_{\boldsymbol{\alpha}}<\infty,\\
\sup_{(x,a)\in\X\times\A}
\left\|
\nabla_{\boldsymbol{\beta}}
\ell_{\boldsymbol{\alpha},\boldsymbol{\beta},\boldsymbol{\theta}}(x,a)
\right\|
&=
\sup_{(x,a)\in\X\times\A}
\|p(\cdot\mid x,a,\mu_E)-\mu_E(\cdot)\|
=:M_{\boldsymbol{\beta}}<\infty,\\
\sup_{x\in\X}
\left\|
\nabla_{\boldsymbol{\theta}}
\ell_{\boldsymbol{\alpha},\boldsymbol{\beta},\boldsymbol{\theta}}(x,a)
\right\|
&=
\sup_{x\in\X}
\|e_x\|
=:M_{\boldsymbol{\theta}}<\infty.
\end{align*}
Let $M\coloneqq
\max\{M_{\boldsymbol{\alpha}},M_{\boldsymbol{\beta}},M_{\boldsymbol{\theta}}\}$. Since \(0\leq \nu^\star(x,a)\leq 1\), the previous display implies
\[
\|\nabla_{\gamma}\nu^\star(x,a)\|
\leq
\left\|
\nabla_{\gamma}\ell_{\boldsymbol{\alpha},\boldsymbol{\beta},\boldsymbol{\theta}}(x,a)
\right\|
+
\left\|
\left\langle
\nabla_{\gamma}\ell_{\boldsymbol{\alpha},\boldsymbol{\beta},\boldsymbol{\theta}}
\right\rangle_{\nu^\star}
\right\|
\leq 2M.
\]
Combining the three parameter blocks, the full parameter derivative of
\(\nu^\star(x,a)\) is bounded by \(2\sqrt{3}M\). Hence, by the Mean Value
Theorem, for every \((x,a)\in\X\times\A\),
\[
|\nu^\star(x,a)-\nu^{\star\prime}(x,a)|
\leq
2\sqrt{3}M
\left\|
\Delta(\boldsymbol{\alpha},\boldsymbol{\beta},\boldsymbol{\theta})
\right\|,
\]
where \(\Delta(\cdot)\) denotes the difference between two parameter points.

We now transfer this bound to the gradient. For
\(\gamma\in\{\boldsymbol{\alpha},\boldsymbol{\beta},\boldsymbol{\theta}\}\),
the constant terms in \(\nabla_\gamma h\) cancel when taking differences, and
we obtain
\begin{align*}
\left\|
\Delta\big(
\nabla_{\gamma}h(\boldsymbol{\alpha},\boldsymbol{\beta},\boldsymbol{\theta})
\big)
\right\|
&\leq
M_{\gamma}
\sqrt{
\sum_{(x,a)\in\X\times\A}
\big(
\nu^\star(x,a)-\nu^{\star\prime}(x,a)
\big)^2
}\\
&\leq
2\sqrt{3}M_{\gamma}M\sqrt{|\X||\A|}
\left\|
\Delta(\boldsymbol{\alpha},\boldsymbol{\beta},\boldsymbol{\theta})
\right\|\\
&\leq
2\sqrt{3}M^2\sqrt{|\X||\A|}
\left\|
\Delta(\boldsymbol{\alpha},\boldsymbol{\beta},\boldsymbol{\theta})
\right\|.
\end{align*}
Combining the three gradient blocks gives
\[
\left\|
\Delta\big(
\nabla h(\boldsymbol{\alpha},\boldsymbol{\beta},\boldsymbol{\theta})
\big)
\right\|
\leq
6M^2\sqrt{|\X||\A|}
\left\|
\Delta(\boldsymbol{\alpha},\boldsymbol{\beta},\boldsymbol{\theta})
\right\|.
\]
Therefore \(h\) is \(L\)-smooth with $L\coloneqq 6M^2\sqrt{|\X||\A|}$.
\end{proof}

We next turn to the strong-convexity properties of the dual objective. Unlike smoothness, strong convexity cannot hold on the full parameter space because the dual variables contain constant-shift invariances. Let $\mathbf 1\in\R^\X$ denote the vector whose every component is equal to one, and define
\[
\R^\X_0
\coloneqq
\{v\in\R^\X:\langle v,\mathbf 1\rangle=0\}.
\]
For any $c\in\R$, adding $c\mathbf 1$ to $\boldsymbol{\beta}$ leaves the score unchanged, since
\[
\left\langle
\boldsymbol{\beta}+c\mathbf 1,
p(\cdot\mid x,a,\mu_E)-\mu_E(\cdot)
\right\rangle
=
\left\langle
\boldsymbol{\beta},
p(\cdot\mid x,a,\mu_E)-\mu_E(\cdot)
\right\rangle,
\]
because both $p(\cdot\mid x,a,\mu_E)$ and $\mu_E$ have total mass one.
Similarly, replacing $\boldsymbol{\theta}$ by
$\boldsymbol{\theta}+c\mathbf 1$ adds the constant $c$ to the log-partition term,
while the linear term changes by
\[
-c\sum_{x\in\X}\mu_E(x)=-c.
\]
Hence these two contributions cancel. Therefore,
\[
h(\boldsymbol{\alpha},\boldsymbol{\beta}+c\mathbf 1,\boldsymbol{\theta})
=
h(\boldsymbol{\alpha},\boldsymbol{\beta},\boldsymbol{\theta}),
\qquad
h(\boldsymbol{\alpha},\boldsymbol{\beta},\boldsymbol{\theta}+c\mathbf 1)
=
h(\boldsymbol{\alpha},\boldsymbol{\beta},\boldsymbol{\theta}).
\]
Thus, the full parameter space contains directions along which $h$ is constant. We therefore carry out the strong-convexity analysis on the gauge-fixed space $\R^k\times\R^\X_0\times\R^\X_0$.

This restriction only removes the constant-shift ambiguity in
$\boldsymbol{\beta}$ and $\boldsymbol{\theta}$; it does not change the Boltzmann occupation measure $\nu^\star$ or the recovered policy. On this space, we prove a compact-set strong-convexity result under a non-degeneracy condition on the corresponding feature-transition-state-indicator difference vectors.

\begin{theorem}[Strong convexity of $h$]\label{thm:linear_constant_step}
Suppose that, for some fixed $(x_0,a_0)\in\X\times\A$, the set
{\small
\[
\Big\{
\big(
\varphi(x,a,\mu_E)-\varphi(x_0,a_0,\mu_E),
p(\cdot\mid x,a,\mu_E)-p(\cdot\mid x_0,a_0,\mu_E),
e_x-e_{x_0}
\big)
: (x,a)\in\X\times\A
\Big\}
\]}
spans $\R^k\times\R^\X_0\times\R^\X_0$, where $e_x\in\R^\X$ is defined by
$e_x(y)\coloneqq \1_{\{x=y\}}$ for each $y\in\X$.

Then, for every compact subset
$D\subset\R^k\times\R^\X_0\times\R^\X_0$, the function $h$ is
$\rho(D)$-strongly convex on $D$ with respect to the variables $(\boldsymbol{\alpha},\boldsymbol{\beta},\boldsymbol{\theta})
\in
\R^k\times\R^\X_0\times\R^\X_0$, where $\rho(D)>0$.
\end{theorem}

\begin{remark}
The spanning hypothesis is a non-degeneracy condition on the gauge-fixed parameter space: it asks that the joint feature-transition-state-indicator difference vectors, relative to the reference pair $(x_0,a_0)$, span the identifiable space $\R^k\times\R^\X_0\times\R^\X_0$. The restriction to $\R^k\times\R^\X_0\times\R^\X_0$ removes the constant-shift invariances in the $\boldsymbol{\beta}$ and $\boldsymbol{\theta}$ variables. Under the hypothesis, the Hessian of $h$ restricted to this space is uniformly positive definite on any compact set, with a strong-convexity modulus equal to the minimum eigenvalue of the restricted Hessian over that set.
\end{remark}
\begin{proof}
Our aim is to establish the strong convexity of \(h\) on
\(\R^k\times\R^\X_0\times\R^\X_0\). It suffices to show that the Hessian of
\(h\), restricted to this space, is uniformly positive definite on compact
subsets \citep[Sect.~9.1.2]{BoLi04}.

We first compute the Hessian blocks. From the expression of the Boltzmann
occupation measure \(\nu^\star\), we have
\[
\nabla_\gamma \nu^\star(x,a)
=
\nu^\star(x,a)
\left(
\nabla_\gamma \ell_{\boldsymbol{\alpha},\boldsymbol{\beta},\boldsymbol{\theta}}(x,a)
-
\left\langle
\nabla_\gamma \ell_{\boldsymbol{\alpha},\boldsymbol{\beta},\boldsymbol{\theta}}
\right\rangle_{\nu^\star}
\right),
\qquad
\gamma\in\{\boldsymbol{\alpha},\boldsymbol{\beta},\boldsymbol{\theta}\}.
\]
Therefore,
\begin{align*}
\nabla_{\boldsymbol{\alpha},\boldsymbol{\alpha}}^2 h(\cdot)
&=
\operatorname{Cov}_{(x,a)\sim\nu^\star}
\big(\varphi(x,a,\mu_E)\big),\\
\nabla_{\boldsymbol{\beta},\boldsymbol{\beta}}^2 h(\cdot)
&=
\operatorname{Cov}_{(x,a)\sim\nu^\star}
\big(p(\cdot\mid x,a,\mu_E)\big),\\
\nabla_{\boldsymbol{\theta},\boldsymbol{\theta}}^2 h(\cdot)
&=
\operatorname{Cov}_{(x,a)\sim\nu^\star}
\big(e_x\big).
\end{align*}
Similarly, the cross terms are
\begin{align*}
\nabla_{\boldsymbol{\alpha},\boldsymbol{\beta}}^2 h(\cdot)
&=
\operatorname{Cov}_{(x,a)\sim\nu^\star}
\big(p(\cdot\mid x,a,\mu_E),\varphi(x,a,\mu_E)\big),\\
\nabla_{\boldsymbol{\alpha},\boldsymbol{\theta}}^2 h(\cdot)
&=
\operatorname{Cov}_{(x,a)\sim\nu^\star}
\big(e_x,\varphi(x,a,\mu_E)\big),\\
\nabla_{\boldsymbol{\beta},\boldsymbol{\theta}}^2 h(\cdot)
&=
\operatorname{Cov}_{(x,a)\sim\nu^\star}
\big(e_x,p(\cdot\mid x,a,\mu_E)\big).
\end{align*}
Here we use that subtracting the deterministic vector \(\mu_E(\cdot)\) does not
change covariances. Define
\[
\textbf{X}(x,a)
\coloneqq
\big(
\varphi(x,a,\mu_E),\,
p(\cdot\mid x,a,\mu_E),\,
e_x
\big).
\]
Then, as quadratic forms on
\(\R^k\times\R^\X_0\times\R^\X_0\), we have
\begin{equation}
\operatorname{Hes}_{\R^k\times\R^\X_0\times\R^\X_0}(h)
(\boldsymbol{\alpha},\boldsymbol{\beta},\boldsymbol{\theta})
=
\operatorname{Cov}(\textbf{X}).
\label{hes=cov}
\end{equation}

We now show that this restricted covariance is positive definite. Suppose, for
contradiction, that it is not. Then there exists a nonzero vector
\[
\boldsymbol{v}\in \R^k\times\R^\X_0\times\R^\X_0
\]
such that
\[
\left\langle
\boldsymbol{v},
\operatorname{Cov}(\textbf{X})\boldsymbol{v}
\right\rangle
=0.
\]
Equivalently,
\[
0
=
\operatorname{Var}_{\nu^\star}
\left(
\left\langle \boldsymbol{v},\textbf{X}(x,a)\right\rangle
\right).
\]
Hence there exists \(c\in\R\) such that
\[
\left\langle \boldsymbol{v},\textbf{X}(x,a)\right\rangle=c
\]
for \(\nu^\star\)-almost every \((x,a)\). Since \(\nu^\star\) is a Boltzmann
distribution, it assigns positive probability to every state-action pair in
\(\X\times\A\). Therefore,
\[
\left\langle \boldsymbol{v},\textbf{X}(x,a)\right\rangle=c,
\qquad
\forall (x,a)\in\X\times\A.
\]
In particular, for the fixed reference pair \((x_0,a_0)\),
\[
\left\langle
\boldsymbol{v},
\textbf{X}(x,a)-\textbf{X}(x_0,a_0)
\right\rangle
=0,
\qquad
\forall (x,a)\in\X\times\A.
\]
By the spanning assumption, the vectors
\[
\textbf{X}(x,a)-\textbf{X}(x_0,a_0),
\qquad (x,a)\in\X\times\A,
\]
span \(\R^k\times\R^\X_0\times\R^\X_0\). Thus
\(\boldsymbol{v}\) is orthogonal to all of
\(\R^k\times\R^\X_0\times\R^\X_0\). Since
\(\boldsymbol{v}\) itself belongs to this space, we obtain
\(\boldsymbol{v}=0\), contradicting the choice of \(\boldsymbol{v}\). Therefore,
\(\operatorname{Cov}(\textbf{X})\) is positive definite on
\(\R^k\times\R^\X_0\times\R^\X_0\).

For each parameter value
\((\boldsymbol{\alpha},\boldsymbol{\beta},\boldsymbol{\theta})\), let
\(\lambda_{\min}(\boldsymbol{\alpha},\boldsymbol{\beta},\boldsymbol{\theta})\)
denote the smallest eigenvalue of the restricted covariance operator on
\(\R^k\times\R^\X_0\times\R^\X_0\). The map $(\boldsymbol{\alpha},\boldsymbol{\beta},\boldsymbol{\theta})
\mapsto
\lambda_{\min}(\boldsymbol{\alpha},\boldsymbol{\beta},\boldsymbol{\theta})$ is continuous, because \(\nu^\star\), and hence \(\operatorname{Cov}(\textbf{X})\), depends continuously on the parameters.
Therefore, for every compact subset \(D\subset\R^k\times\R^\X_0\times\R^\X_0\), the minimum
\[
\rho(D)
\coloneqq
\min_{(\boldsymbol{\alpha},\boldsymbol{\beta},\boldsymbol{\theta})\in D}
\lambda_{\min}(\boldsymbol{\alpha},\boldsymbol{\beta},\boldsymbol{\theta})
\]
is attained. Since the restricted covariance is positive definite at every
point, we have \(\rho(D)>0\). Hence, for all
\((\boldsymbol{\alpha},\boldsymbol{\beta},\boldsymbol{\theta})\in D\),
\[
\operatorname{Hes}_{\R^k\times\R^\X_0\times\R^\X_0}(h)
(\boldsymbol{\alpha},\boldsymbol{\beta},\boldsymbol{\theta})
\succeq
\rho(D) I.
\]
Thus the restriction of \(h\) to
\(\R^k\times\R^\X_0\times\R^\X_0\) is \(\rho(D)\)-strongly convex on \(D\).
\end{proof}

The smoothness and strong convexity properties established above justify the application of gradient descent with a constant step size to the dual problem. The resulting procedure is summarised in the algorithm below.

\begin{algorithm}[H]
\caption{Gradient Descent for the Linear Dual}
\label{alg:linear-gd}
\begin{algorithmic}[1]
\REQUIRE Initial parameters
$(\boldsymbol{\alpha}_0,\boldsymbol{\beta}_0,\boldsymbol{\theta}_0)
\in\R^k\times\R^\X_0\times\R^\X_0$, step size
$\delta\in(0,1/L]$, number of iterations $K$
\STATE $(\bm{\alpha}, \bm{\beta}, \bm{\theta}) \leftarrow (\bm{\alpha}_0, \bm{\beta}_0, \bm{\theta}_0)$
\FOR{$k = 0, 1, \ldots, K - 1$}
    \STATE $(\bm{\alpha}, \bm{\beta}, \bm{\theta}) \leftarrow (\bm{\alpha}, \bm{\beta}, \bm{\theta}) - \delta\, \nabla h(\bm{\alpha}, \bm{\beta}, \bm{\theta})$
\ENDFOR
\STATE Compute $\nu^\star_{\bm{\alpha}, \bm{\beta}, \bm{\theta}}$ at the final parameters.
\RETURN $(\bm{\alpha}, \bm{\beta}, \bm{\theta})$ and $\nu^\star_{\bm{\alpha}, \bm{\beta}, \bm{\theta}}$.
\end{algorithmic}
\end{algorithm}

\begin{lemma}[Invariance of the gauge-fixed parameter space]
\label{lem:gauge_fixed_invariance}
Consider the gradient descent iteration
\[
(\bm{\alpha}_{n+1},\bm{\beta}_{n+1},\bm{\theta}_{n+1})
=
(\bm{\alpha}_{n},\bm{\beta}_{n},\bm{\theta}_{n})
-
\delta
\nabla h(\bm{\alpha}_{n},\bm{\beta}_{n},\bm{\theta}_{n}),
\]
where $\delta>0$. If $(\bm{\alpha}_{0},\bm{\beta}_{0},\bm{\theta}_{0})
\in
\R^k\times\R^\X_0\times\R^\X_0,$ then
\[
(\bm{\alpha}_{n},\bm{\beta}_{n},\bm{\theta}_{n})
\in
\R^k\times\R^\X_0\times\R^\X_0
\]
for every $n\geq 0$.
\end{lemma}
Lemma \ref{lem:gauge_fixed_invariance} ensures that gradient descent is compatible with the gauge-fixed formulation used in Theorem~\ref{thm:linear_constant_step}. In particular, once
the algorithm is initialised in $\R^k\times\R^\X_0\times\R^\X_0$, the iterates never leave this space, so the strong-convexity estimate can be applied along the whole trajectory. The proof is
elementary: the $\boldsymbol{\beta}$- and $\boldsymbol{\theta}$-components of $\nabla h$ have zero inner product with $\mathbf 1$, and hence belong to $\R^\X_0$; since $\R^\X_0$ is a linear subspace, the claim follows by induction. The details are given in Appendix \ref{app:linear-proofs}.

\iffalse
\begin{algorithm}[H]
\caption{Gradient Descent for Parameter Optimisation}
\label{GD}
\begin{algorithmic}[1]
    \REQUIRE Initial parameters $(\boldsymbol{\alpha}_0,\boldsymbol{\beta}_0,\boldsymbol{\theta}_0)$, step size $\delta > 0$, number of iterations $K$
    \STATE Set $(\boldsymbol{\alpha},\boldsymbol{\beta},\boldsymbol{\theta}) \leftarrow (\boldsymbol{\alpha}_0,\boldsymbol{\beta}_0,\boldsymbol{\theta}_0)$
    \FOR{$k=0,1,\dots,K-1$} 
        \STATE $(\boldsymbol{\alpha},\boldsymbol{\beta},\boldsymbol{\theta}) \leftarrow (\boldsymbol{\alpha},\boldsymbol{\beta},\boldsymbol{\theta}) - \delta\, \nabla h(\boldsymbol{\alpha},\boldsymbol{\beta},\boldsymbol{\theta})$
    \ENDFOR
    \STATE Compute $\nu^\star_{\boldsymbol{\alpha},\boldsymbol{\beta},\boldsymbol{\theta}}$ from the final parameters.
    \RETURN $(\boldsymbol{\alpha},\boldsymbol{\beta},\boldsymbol{\theta})$ and $\nu^\star_{\boldsymbol{\alpha},\boldsymbol{\beta},\boldsymbol{\theta}}$
\end{algorithmic}
\end{algorithm}

\fi

\paragraph*{Convergence rate.}
Let $z_\star \coloneqq (\bm{\alpha}_\star, \bm{\beta}_\star, \bm{\theta}_\star) \in \R^k\times\R^\X_0\times\R^\X_0$ be the unique minimiser of $h$ on $\R^k\times\R^\X_0\times\R^\X_0$, and let $z_0 \coloneqq (\bm{\alpha}_0, \bm{\beta}_0, \bm{\theta}_0)$. Define the compact set
\[
D \coloneqq \big\{ z \in \R^k\times\R^\X_0\times\R^\X_0 : \|z - z_\star\| \leq \|z_0 - z_\star\| \big\},
\]
and let $\rho(D) > 0$ be the strong-convexity constant of $h$ on $D$ from Theorem~\ref{thm:linear_constant_step}. By the standard analysis of gradient descent under smoothness and strong convexity \citep[Theorem~3.6]{GaGo24}, for any $\delta \in (0, 1/L]$ the iterates $\{z_k\}_{k \geq 0}$ converge linearly to $z_\star$:
\[
\|z_k - z_\star\|^2 \leq (1 - \delta\, \rho(D))^k\, \|z_0 - z_\star\|^2, \qquad k \geq 0.
\]
For $\delta = 1/L$, the contraction factor is $1 - \rho(D)/L$.

\paragraph*{Convergence of the recovered occupation measure.}
Since each coordinate $\nu^\star_{\bm{\alpha}, \bm{\beta}, \bm{\theta}}(x, a)$ is $2\sqrt{3}M$-Lipschitz in $(\bm{\alpha}, \bm{\beta}, \bm{\theta})$ by Theorem~\ref{thm:linear_smoothness}, the parameter convergence rate transfers to the occupation-measure level:
\[
\big\| \nu^\star_{\bm{\alpha}_k, \bm{\beta}_k, \bm{\theta}_k}
       - \nu^\star_{\bm{\alpha}_\star, \bm{\beta}_\star, \bm{\theta}_\star} \big\|
\leq 2\sqrt{3}M\sqrt{|\X| \cdot |\A|}\, (1 - \delta\, \rho(D))^{k/2}\, \|z_0 - z_\star\|.
\]

\paragraph*{Recovery of the optimal policy.}
At the dual minimiser $z_\star$, the optimal occupation measure of $\optt$ is the Boltzmann distribution
\[
\nu^\star_{\bm{\alpha}_\star, \bm{\beta}_\star, \bm{\theta}_\star}(x, a)
= \frac{e^{\ell_{\bm{\alpha}_\star, \bm{\beta}_\star, \bm{\theta}_\star}(x, a)}}
       {\sum_{(x, a) \in \X \times \A} e^{\ell_{\bm{\alpha}_\star, \bm{\beta}_\star, \bm{\theta}_\star}(x, a)}}.
\]
The corresponding policy, obtained via Remark~\ref{rem:linear-correspondence}, is
\[
\pi_{\nu^\star_{\bm{\alpha}_\star, \bm{\beta}_\star, \bm{\theta}_\star}}(a \mid x)
= \frac{\nu^\star_{\bm{\alpha}_\star, \bm{\beta}_\star, \bm{\theta}_\star}(x, a)}
       {\nu^{\star, \X}_{\bm{\alpha}_\star, \bm{\beta}_\star, \bm{\theta}_\star}(x)}.
\]
By Theorem~\ref{thm:linear_correspondence}, this policy is the optimiser of the original maximum causal entropy problem $\opt$; by Proposition~\ref{prop:linear-mfe}, the pair $(\pi_{\nu^\star_{\bm{\alpha}_\star, \bm{\beta}_\star, \bm{\theta}_\star}}, \mu_E)$ is a mean-field equilibrium under the recovered reward.

\section{Reproducing Kernel Hilbert Space Reward Model}\label{sec:rkhs}

We now specialise the optimisation problem $\opt$ from Section~\ref{sect:MaxEnt} to a functional reward model based on reproducing kernel Hilbert spaces. In contrast to the finite-dimensional linear case of Section~\ref{sect:linear_reward}, the reward is no longer restricted to a finite-dimensional feature family; instead, it is assumed to belong to a reproducing kernel Hilbert space (RKHS), giving a flexible non-parametric representation while preserving the analytical structure required for first-order optimisation.

Let $\sZ \coloneqq \X \times \A \times \mathcal{P}(\X)$
denote the domain of the reward function. We fix a symmetric, measurable, positive semidefinite kernel $k : \sZ \times \sZ \to \R$, and let $\mathcal{H}_k \subset \R^{\sZ}$ denote the associated RKHS, equipped with inner product $\langle \cdot, \cdot \rangle_{\mathcal{H}_k}$.\footnote{An RKHS is a Hilbert space of functions on $\sZ$ in which every point-evaluation map $f \mapsto f(z)$ is continuous; this is equivalent to the existence of a reproducing kernel $k(\cdot, z)$. See \citet{PaVe16} for an accessible introduction.} We specialise the generic feature map of Section~\ref{sect:IRL} to the canonical RKHS feature
\[
\varphi : \sZ \to \mathcal{H}_k, \qquad \varphi(z) \coloneqq k(\cdot, z),
\]
which satisfies the following standard properties.

\iffalse
To formalize this reward model, we first introduce the underlying RKHS structure. Let
\[
\sZ \coloneqq \X \times \A \times \mathcal{P}(\X)
\]
denote the domain of the reward function. We fix a positive semidefinite kernel
\[
k : \sZ \times \sZ \to \mathbb{R},
\]
symmetric and measurable, and let $\mathcal{H}_k \subset \mathbb{R}^{\sZ}$
denote the associated reproducing kernel Hilbert space, equipped with inner
product $\langle \cdot,\cdot\rangle_{\mathcal{H}_k}$.\footnote{%
An RKHS is a Hilbert space of functions on $\sZ$ in which every point-evaluation
map $f\mapsto f(z)$ is continuous; this is equivalent to the existence of a
reproducing kernel $k(\cdot,z)$. We refer the reader to \citet{PaVe16} for an accessible introduction to RKHS theory.}
\fi

%In this section, we specialise the generic feature map introduced in Section~\ref{sect:IRL} to the canonical RKHS feature map \[ \varphi : \sZ \to \mathcal{H}_k, \qquad \varphi(z)\coloneqqk(\cdot,z). \]

\begin{lemma}\label{lem:rkhs-properties}
The canonical feature map $\varphi(z) = k(\cdot, z)$ satisfies:
\begin{enumerate}[label=(\roman*), leftmargin=2.5em]
\item \emph{Reproducing property:} $f(z) = \langle f, \varphi(z) \rangle_{\mathcal{H}_k}$ for every $f \in \mathcal{H}_k$ and $z \in \sZ$;
\item \emph{Kernel as inner product:} $k(z, z') = \langle \varphi(z), \varphi(z') \rangle_{\mathcal{H}_k}$ for every $z, z' \in \sZ$;
\item \emph{Density of the feature span:} $\mathcal{H}_k = \overline{\operatorname{span}} \{ \varphi(z) : z \in \sZ \}$.
\end{enumerate}
\end{lemma}

%Below, we summarise some basic properties of the canonical RKHS. These results are standard and can be found in most references on RKHS theory \citep[see, e.g.,][]{PaVe16}. 

%\begin{enumerate}[label=(\roman*), leftmargin=2.5em]
%\item[ ] \emph{\textbf{Reproducing property:}} For every $f\in\mathcal{H}_k$, $f(z)=\langle f,\varphi(z)\rangle_{\mathcal{H}_k},
%\qquad \forall z\in \sZ.$
%\item[ ] \emph{\textbf{Kernel as inner product:}} For every $z,z'\in \sZ$, $k(z,z')=\langle \varphi(z),\varphi(z')\rangle_{\mathcal{H}_k}.$

%\item[ ] \emph{\textbf{Density of the feature span:}} The linear span of the feature vectors is dense in $\mathcal{H}_k$:
%\[ \mathcal{H}_k=\overline{\operatorname{span}}\{\varphi(z):z\in \sZ\}. \]
%\end{enumerate} }

The kernel representation of the inner product yields a convenient formula for finite linear combinations of feature vectors. For \(n,m\in\mathbb N\), points
\(z_1,\ldots,z_n,y_1,\ldots,y_m\in \mathcal Z\), and coefficients
\(c_1,\ldots,c_n,d_1,\ldots,d_m\in\mathbb R\), bilinearity together with
the second property yields
\[
\left\langle
\sum_{i=1}^{n}c_i\varphi(z_i),
\sum_{j=1}^{m}d_j\varphi(y_j)
\right\rangle_{\mathcal{H}_k}
=
\sum_{i=1}^{n}\sum_{j=1}^{m}c_i d_j\,k(z_i,y_j).
\]
This identity will be used repeatedly in the following analysis.

To guarantee well-posedness, we restrict the reward class by replacing the linear assumption (Assumption~\ref{ass:linear_reward}) with the requirement that the reward resides in the aforementioned RKHS. In what follows, we leverage this functional framework to analyse the inverse problem.
 
\begin{assumption}[RKHS reward model]\label{ass:rkhs_reward}
The unknown reward function belongs to $\mathcal{H}_k$, i.e.\ $r \in \mathcal{H}_k$ and $r(z) = \langle r, \varphi(z) \rangle_{\mathcal{H}_k}$ for all $z \in \sZ$.
\end{assumption}

%To make the IRL problem well posed, we now restrict the class of admissible reward functions. In particular, we no longer work under Assumption~\ref{ass:linear_reward} from the previous section, where the reward was assumed to belong to a finite-dimensional linear family. Instead, throughout the remainder of this section, we replace Assumption~\ref{ass:linear_reward} with the following functional requirement: the reward belongs to the RKHS introduced above.

\begin{remark}[On realisability]\label{rem:rkhs-realizability}
Assumption~\ref{ass:rkhs_reward} is the RKHS analogue of the linear-realisability hypothesis in Assumption~\ref{ass:linear_reward}: it asserts that the expert reward lies in the chosen function class. The hypothesis is essential for the equilibrium-recovery statement of Proposition~\ref{prop:rkhs-mfe}. If $r_E \notin \mathcal{H}_k$, the algorithm still solves the projected problem: it returns the reward in $\mathcal{H}_k$ whose long-run feature expectation under $\pi_E$ coincides with that of $r_E$. A formal analysis of misspecified rewards is beyond the scope of this paper.
\end{remark}

\begin{remark}[Linear rewards as a special case]\label{rem:rkhs-linear-embedding}
If $r \in \operatorname{span}\{\varphi(z_1), \ldots, \varphi(z_k)\}$, then there exist $\alpha_1, \ldots, \alpha_k \in \R$ with $r = \sum_{i=1}^k \alpha_i\, \varphi(z_i)$. By the reproducing property,
\[
r(z) = \sum_{i=1}^k \alpha_i\, k(z, z_i) = \langle \bm{\alpha}, \xi(z) \rangle_{\R^k},
\quad \text{where } \quad \xi(z) \coloneqq (k(z, z_1), \ldots, k(z, z_k)) \in \R^k.
\]
Hence the finite-dimensional linear reward model of Section~\ref{sect:linear_reward} embeds into the RKHS model as the span of finitely many canonical features.
\end{remark}

%Under Assumption~\ref{ass:rkhs_reward}, the expert feature expectation vector $\langle \varphi\rangle_{\pi_E,\mu_E}$ takes values in $\mathcal H_k$, and the original problem $\opt$ becomes a maximum causal entropy problem with a possibly infinite-dimensional constraint. 

\iffalse

The following observation makes explicit how finite-dimensional linear rewards can be represented within the RKHS model.
\begin{remark}
If $r \in \operatorname{span}\{\varphi(z_1),\dots,\varphi(z_k)\}$, then there exist coefficients $\alpha_1,\dots,\alpha_k\in\mathbb R$ such that
\[
r(\cdot)=\sum_{i=1}^k \alpha_i\,\varphi(z_i).
\]
Hence, by the reproducing property,
\[
r(z)=\sum_{i=1}^k \alpha_i 
\langle \varphi(z),\varphi(z_i)\rangle_{\mathcal H_k} = \sum_{i=1}^k \alpha_i k(z,z_i)  
= \left\langle {\bf \alpha}, \big(k(z,z_1),\ldots,k(z,z_k)\big) \right\rangle_{\R^k}.
\]
\end{remark}

\fi

Under Assumption~\ref{ass:rkhs_reward}, the expert feature expectation vector $\langle \varphi \rangle_{\pi_E, \mu_E}$ takes values in $\mathcal{H}_k$.\footnote{The Cesàro average defining $\langle \varphi \rangle_{\pi_E, \mu_E}$ is a Bochner integral of $\mathcal{H}_k$-valued random elements; it is well-defined because $\varphi(x, a, \mu_E) = k(\cdot, (x, a, \mu_E))$ has bounded $\mathcal{H}_k$-norm on the finite set $\X \times \A$, and the relevant evaluations lie in the finite-dimensional subspace $\operatorname{span}\{\varphi(x, a, \mu_E) : (x, a) \in \X \times \A\} \subset \mathcal{H}_k$} The problem $\opt$ becomes a maximum causal entropy problem with a possibly infinite-dimensional constraint. We next show that, as in the linear case, an optimal solution of $\opt$ induces an MFE.

%We next {\color{blue} state} that, exactly as in the finite-dimensional linear case, an optimal solution of the original maximum causal entropy problem $\opt$ induces a mean-field equilibrium under the RKHS reward model.

\begin{proposition}\label{prop:rkhs-mfe}
Suppose Assumption~\ref{ass:rkhs_reward} holds and let $\pi^\star$ be an optimal solution of $\opt$. Then $(\pi^\star, \mu_E)$ is a mean-field equilibrium.
\end{proposition}

%\begin{proposition}\label{prop:rkhs-mfe}
%Suppose Assumption~\ref{ass:rkhs_reward} holds, and let $\pi^\star$ be anoptimal solution of the original maximum causal entropy problem $\opt$. Then the pair $(\pi^\star,\mu_E)$ is a mean-field equilibrium.
%\end{proposition}

\begin{proof}
Since $\pi^\star$ is feasible for $\opt$, the first constraint of $\opt$ gives $\mu_E \in \Lambda(\pi^\star)$.
 
Let $r_E \in \mathcal{H}_k$ denote the true expert reward. By the reproducing property, for every $(x, a) \in \X \times \A$,
\[
r_E(x, a, \mu_E) = \langle r_E, \varphi(x, a, \mu_E) \rangle_{\mathcal{H}_k}.
\]
Hence, for any stationary policy $\pi$,
\[
\J_{\mu_E}(\pi, \mu_E)
= \lim_{T \to \infty} \frac{1}{T} \sum_{t=0}^{T-1} \E^{\pi, \mu_E} \big[ r_E(x_t, a_t, \mu_E) \big]
= \big\langle r_E,\, \langle \varphi \rangle_{\pi, \mu_E} \big\rangle_{\mathcal{H}_k}.
\]
Since $\pi^\star$ is feasible for $\opt$, its feature expectation matches the expert's: $\langle \varphi \rangle_{\pi^\star, \mu_E} = \langle \varphi \rangle_{\pi_E, \mu_E}$. Therefore,
\[
\J_{\mu_E}(\pi^\star, \mu_E)
= \big\langle r_E,\, \langle \varphi \rangle_{\pi_E, \mu_E} \big\rangle_{\mathcal{H}_k}
= \J_{\mu_E}(\pi_E, \mu_E).
\]
Since $(\pi_E, \mu_E)$ is a mean-field equilibrium, $\pi_E \in \Phi(\mu_E)$, i.e., $\J_{\mu_E}(\pi_E, \mu_E) = \sup_{\pi \in \Pi} \J_{\mu_E}(\pi, \mu_E)$. Combined with the previous identity, this gives $\pi^\star \in \Phi(\mu_E)$. Together with $\mu_E \in \Lambda(\pi^\star)$, this proves the claim.
\end{proof}

Proposition~\ref{prop:rkhs-mfe} shows that $\opt$ retains its equilibrium interpretation under the RKHS reward model. In the remainder of this section, we analyse $\opt$ through a Lagrangian relaxation and the associated soft Bellman equations.

%{\color{blue}
%Proposition~\ref{prop:rkhs-mfe} shows that the optimisation problem $\opt$ retains its equilibrium interpretation under the RKHS reward model. In the remainder of this section, we analyse this optimisation problem through a dual relaxation and the associated soft Bellman equations.
%}

\subsection{Maximum Causal Entropy Formulation and Dual Relaxation}\label{sec:rkhs-formulation}

 It is convenient to rewrite $\opt$ in a form where the invariance of the expert mean-field term is expressed directly through long-run state frequencies. This yields the equivalent problem
\[
\begin{array}{ll}
\opttH \quad \underset{\pi \in \mathcal{P}(\A \mid \X)}{\text{maximise}} & H(\pi) \\[0.4em]
\phantom{\opttH \quad} \text{subject to}
& \displaystyle \lim_{T \to \infty} \frac{1}{T} \sum_{t=0}^{T-1} \E^{\pi, \mu_E} \left[ \1_{\{x_t = x\}} \right] = \mu_E(x), \quad \forall x \in \X, \\[0.7em]
& \displaystyle \lim_{T \to \infty} \frac{1}{T} \sum_{t=0}^{T-1} \E^{\pi, \mu_E} \left[ \varphi(x_t, a_t, \mu_E) \right] = \langle \varphi \rangle_{\pi_E, \mu_E}.
\end{array}
\]

\iffalse
For the analysis that follows, it is convenient to rewrite the original maximum
causal entropy formulation $\opt$ in a form where the invariance of the expert
mean-field term is expressed directly through long-run state frequencies. This leads to the following equivalent problem.

\[
\begin{array}{lll}
\optH \,\, \underset{\pi \in \mathcal{P}(\A \mid \X)}{\text{maximise}} \text{ } &H(\pi) 
\\
\phantom{\mathbf{(OPT_1)} \,\, } \text{subject to}  & \lim_{T\to \infty}\frac{1}{T}\sum_{t=0}^{T-1} \E^{\pi,\mu_E} \left[ \1_{ \{ x_t= x \} } \right] = \mu_E(x) \,\, \forall x \in \X \\
 \phantom{x} & \lim_{T \to \infty} \frac{1}{T} \sum_{t=0}^{T-1} \, \E^{\pi,\mu_E}[\varphi(x_t,a_t,\mu_E)] = \langle \varphi \rangle_{\pi_E,\mu_E}
\end{array}
\]

\fi

%The next result shows that the reformulated problem is equivalent to the original maximum causal entropy formulation under the RKHS reward model.
\begin{proposition}\label{prop:RKHS_equivilance}
$\opt$ and $\optH$ are equivalent.
\end{proposition}

 The two formulations share the same objective and feature-matching constraint, so equivalence reduces to showing that the first constraint of $\opt$ (invariance) is equivalent to the first constraint of $\opttH$ (matching the Cesàro state averages to $\mu_E$). The forward direction follows from Assumption~\ref{ass:ergodicity} and the ergodic theorem; the converse follows from a Cesàro telescoping argument. The full proof is given in Appendix~\ref{app:proof-rkhs-equivalence}.

We now analyse $\opttH$ through a Lagrangian relaxation of its constraints. Introduce the dual parameter
\[
\vartheta \coloneqq (\zeta,\psi) \in \mathbb{R}^{\X}\times \mathcal{H}_k,
\]
where $\zeta$ corresponds to the state-frequency constraint and $\psi$ to the feature-matching constraint. The Lagrangian relaxation of $\optH$ is therefore given by
\begin{align*}
    \mathcal{D}(\vartheta)  \coloneqq \max_{\pi \in \mathcal{P}(\A \mid \X)}  \Bigg\{& H(\pi)  + \Big\langle \zeta, \lim_{T\to \infty}\frac{1}{T} \sum_{t=0}^{T-1} \E^{\pi,\mu_E} \left[ \1_{ \{ x_t= \cdot \} } \right] - \mu_E( \cdot) \Big\rangle_{\R^\X} \\
    & + \Big\langle \psi, \lim_{T \to \infty} \frac{1}{T} \sum_{t=0}^{T-1} \, \E^{\pi,\mu_E}[\varphi(x_t,a_t,\mu_E)] - \langle \varphi \rangle_{\pi_E,\mu_E}  \Big\rangle_{\mathcal{H}_k} \Bigg\}
\end{align*}
Let us denote the term inside the $\max$ as $\mathcal{L}(\pi,\vartheta)$. Then, the dual function is 
$$\displaystyle \mathcal{D}(\vartheta) = \max_{\pi \in \mathcal{P}(\A \mid \X)} \mathcal{L}(\pi,\vartheta).$$ 
For fixed $\vartheta$, the terms $\langle \zeta,\mu_E\rangle_{\mathbb{R}^{\X}}$
and $\langle \psi,\langle \varphi\rangle_{\pi_E,\mu_E}\rangle_{\mathcal H_k}$ do not depend
on $\pi$ and therefore do not affect the maximiser. 

Consequently, the policy
that maximises $\mathcal{L}(\pi,\vartheta)$ is equivalently characterised by the reduced
optimisation problem below.
{\small
\begin{align*}
\underset{\pi\in\mathcal{P}(\A\mid \X)}{\arg\max} \ \mathcal{L}(\pi, \vartheta) 
=\underset{\pi\in\mathcal{P}(\A\mid \X)}{\arg\max}
\bigg\{ 
H(\pi)
& +\Big\langle \zeta,\ \lim_{T\to\infty}\frac{1}{T}\sum_{t=0}^{T-1}
\E^{\pi,\mu_E}\!\big[ \1\{x_t=\cdot\}\big]\Big\rangle_{\R^\X} \nonumber\\ &\quad
+\Big\langle \psi,\ \lim_{T\to\infty}\frac{1}{T}\sum_{t=0}^{T-1}
\E^{\pi,\mu_E}\!\big[\varphi(x_t,a_t,\mu_E)\big]\Big\rangle_{\mathcal H_k}
\bigg\}.
\end{align*}}
To characterise the maximising policy \(\pi_\vartheta\in\arg\max_{\pi}\mathcal{L}(\pi,\vartheta)\), it suffices to solve the following equivalent maximisation problem:
\begin{align*}
&\max_{\pi \in \mathcal{P}(\A \mid \X)} \Bigg\{  H(\pi) 
+ \Big\langle \zeta, \lim_{T\to \infty}\frac{1}{T} \sum_{t=0}^{T-1} \E^{\pi,\mu_E} \left[ \1_{ \{ x_t= \cdot \} } \right] \Big\rangle_{\R^\X} \\ &\hspace{5em} + \Big\langle \psi, \lim_{T \to \infty} \frac{1}{T} \sum_{t=0}^{T-1} \, \E^{\pi,\mu_E}[\varphi(x_t,a_t,\mu_E)]  \Big\rangle_{\mathcal{H}_k} \Bigg\} \\
&= \max_{\pi \in \mathcal{P}(\A \mid \X)} \Bigg\{  H(\pi) 
+ \limT \E^{\pi,\mu_E} \left[\zeta_{x_t}\right] + \limT \E^{\pi,\mu_E} \left[\psi(x_t,a_t,\mu_E) \right] \Bigg\}\\
&=\max_{\pi \in \mathcal{P}(\A \mid \X)} \Bigg\{  H(\pi) + \limT \E^{\pi,\mu_E}\left[\zeta_{x_t} +\psi(x_t,a_t,\mu_E) \right]\Bigg\},
\end{align*}
where $\zeta_{x_t}$ denotes the coordinate of $\zeta \in \R^\X$ evaluated at $x_t$, and the final term arises from the reproducing property of the RKHS inner product. 

Thus, for each fixed dual parameter $\vartheta$, the inner maximisation reduces to
an entropy-regularised average-reward control problem under the expert
mean-field term, with effective reward
\[
r_\vartheta(x,a,\mu)\coloneqq\zeta_x+\psi(x,a,\mu).
\]
We next characterise the maximiser of this relaxed control problem through soft
Bellman equations.

\subsection{Soft Bellman Equations and Differentiability}\label{sect:RKHS_Bellman_equations}

To obtain a strict contraction for the soft Bellman optimality operator in the average-reward setting, we introduce the following minorisation condition.

\begin{assumption}[Minorisation]\label{ass:minorization}
There exists a sub-probability measure \(\xi\) on \(\mathcal X\) such that, for all
\((x,a,\mu)\in\mathcal X\times\mathcal A\times\mathcal P(\mathcal X)\),
$$p(\cdot \mid x, a, \mu) \geq \xi(\cdot)$$ pointwise, with $0<\xi(\X) \coloneqq \sum_{y \in \X} \xi(y) < 1$.
\end{assumption}

\iffalse
\begin{assumption}\label{ass:minorization}
    For all $(x,a,\mu)\in\X\times\A\times\mathcal{P}(\mathcal{X})$ there exists a sub-probability measure $\xi(\cdot)$ such that
    \[
    p(\cdot\mid x,a,\mu) \geq \xi(\cdot),
    \]
    where $\xi(\mathcal{\X})\coloneqq \sum_{y\in\X} \xi(y) < 1$.
\end{assumption}
\fi

\noindent Under Assumption~\ref{ass:minorization}, define the sub-stochastic kernel $\Tilde{p}$ by
\begin{equation}\label{def:sub_prob_measure}
\Tilde{p}(y \mid x,a,\mu) \coloneqq p(y \mid x,a,\mu)-\xi(y).
\end{equation}
and set $\kappa \coloneqq 1 - \xi(\X) \in (0, 1)$. The kernel $\Tilde{p}$ has total mass $\sum_y \Tilde{p}(y \mid x, a, \mu) = \kappa$ uniformly in $(x, a, \mu)$. 

%\noindent and set $\kappa \coloneqq 1- \xi(\mathcal{\X}) \in(0,1)$. 

%{\color{blue} Moreover, the optimal policy of the relaxed problem admits the following representation via soft Bellman operators based on the sub-stochastic kernel $\tilde p$.} More precisely, for each fixed $\vartheta$, the maximising policy is characterised through the following soft Bellman system \citep{NeJoVi17}:5\begin{align*} Q_\vartheta(x,a) &\coloneqq r_\vartheta(x,a,\mu_E)  + \sum_{y \in \mathcal{X}} \Tilde{p}(y \mid x,a, \mu_E) \, V_{Q_\vartheta}(y), \\ V_Q(x)  &\coloneqq \log \Big( \sum_{b \in \mathcal{A}} \, e^{  Q(x,b)} \Big), \\ A^\vartheta(x,a) &\coloneqq Q_\vartheta(x,a) - V_{Q_\vartheta}(x). \end{align*}

\begin{proposition}[Soft Bellman characterisation; cf.\ \citealp{NeJoVi17}]\label{prop:soft-bellman}
For each fixed $\vartheta = (\zeta, \psi) \in \R^\X \times \mathcal{H}_k$, the inner maximiser $\pi_\vartheta \in \arg\max_\pi \mathcal{L}(\pi, \vartheta)$ admits the softmax representation
\begin{equation}\label{eq:policy_theta}
\pi_\vartheta(a \mid x) = \exp\big( A^\vartheta(x, a) \big),
\end{equation}
where $A^\vartheta(x, a) \coloneqq Q_\vartheta(x, a) - V_{Q_\vartheta}(x)$ and $Q_\vartheta \in \R^{\X \times \A}$ is the unique fixed point of the soft Bellman system
\begin{align}
Q_\vartheta(x, a) &= r_\vartheta(x, a, \mu_E) + \sum_{y \in \X} \Tilde{p}(y \mid x, a, \mu_E)\, V_{Q_\vartheta}(y), \label{eq:soft-Q} \\
V_Q(x) &= \log \sum_{b \in \A} e^{Q(x, b)}. \label{eq:soft-V}
\end{align}
\end{proposition}

We encapsulate the right-hand side of~\eqref{eq:soft-Q} into the soft Bellman operator $\mathcal{T}^\vartheta$, defined as follows.

\begin{definition}[RKHS Bellman Operator]
\label{RKHS_operatorT}
For any $Q \in \mathbb{R}^{\mathcal{X} \times \mathcal{A}}$ and $\vartheta \in \mathbb{R}^\mathcal{X} \times \mathcal{H}_k$, the operator $\mathcal{T}^\vartheta: \mathbb{R}^{\mathcal{X} \times \mathcal{A}} \to \mathbb{R}^{\mathcal{X} \times \mathcal{A}}$ is given by
\begin{equation*}
\label{eq:RKHS_operatorT}
(\mathcal{T}^\vartheta Q)(x, a) \coloneqq r_\vartheta(x, a, \mu_E) + \sum_{y \in \mathcal{X}} \tilde{p}(y \mid x, a, \mu_E) V_Q(y)
\end{equation*}
for all $(x,a) \in \mathcal{X} \times \mathcal{A}$.
\end{definition}

To compute the optimal policy, one first needs to solve the soft Bellman system. In what follows, we show that the optimal soft $Q$-function can be obtained as the unique fixed point of a contraction mapping, namely the soft Bellman optimality operator.

\begin{remark}\label{rem:soft-Bellman-form}
The standard soft Bellman characterisation in \citet{NeJoVi17} uses the original kernel $p$ together with an additive constant representing the average-reward gain. Under the decomposition $p = \xi + \Tilde{p}$, the $\xi$-contribution to $\sum_y p(y \mid x, a, \mu_E)\, V_Q(y)$ is the scalar $\sum_y \xi(y)\, V_Q(y)$, which is independent of $(x, a)$. Since the induced policy depends on $Q$ only through the advantage $A_Q = Q - V_Q$ (and is therefore invariant to additive constants in $Q$), this scalar can be absorbed into the representative of $Q$, and we may work with the equivalent operator $\mathcal{T}^\vartheta$ written in terms of $\Tilde{p}$. The total mass $\kappa < 1$ of $\Tilde{p}$ is the mechanism behind the strict contraction in the next theorem.
\end{remark}

The key advantage of the representation above is that the associated Bellman operator becomes a strict contraction in the sup-norm.

\begin{theorem}\label{thm:Q_contraction}
Under Assumption~\ref{ass:minorization}, the operator $\mathcal{T}^\vartheta$ is a strict contraction in the sup-norm:
\[
\| \mathcal{T}^\vartheta Q_1 - \mathcal{T}^\vartheta Q_2 \|_\infty \leq \kappa\, \| Q_1 - Q_2 \|_\infty.
\]

In particular, $\mathcal{T}^\vartheta$ admits a unique fixed point $Q_\vartheta$ by the Banach fixed-point theorem.
\end{theorem}

\begin{proof}
    Fix $x\in \X$. For every $b\in \A$, by the definition of the sup-norm,
    \begin{equation*}
        Q_1(x,b) \leq Q_2(x,b) + \left\| Q_1 - Q_2 \right\|_\infty.
    \end{equation*}
    Exponentiating, summing over $b \in \A$ and taking $\log$ yields
{\small 
\begin{equation*}
    \begin{aligned}
             V_{Q_1}(x) = \log \sum_{b \in \A} e^{Q_1(x,b)}  \leq  \log \sum_{b \in \A} e^{Q_2(x,b)} + \left\| Q_1 - Q_2\right\|_\infty   = V_{Q_2}(x) + \left\| Q_1 - Q_2\right\|_\infty.
    \end{aligned}
\end{equation*}}
Equivalently,
    \begin{equation}\label{RKHS_contraction_1}
        V_{Q_1}(x)-V_{Q_2}(x) \leq  \left\| Q_1 - Q_2 \right\|_\infty.
    \end{equation}
    The symmetric argument starting from $Q_2 \leq Q_1 + \|Q_1 - Q_2\|_\infty$ yields the reverse inequality, hence
    %We also have, by the same argument starting from
    %\begin{equation*}
     %   Q_2(x,b) \leq Q_1(x,b) + \left\| Q_1 - Q_2 \right\|_\infty \qquad \forall b \in \A, \end{equation*} that
    \begin{equation}\label{RKHS_contraction_2}
        V_{Q_2}(x) -  V_{Q_1}(x)  \leq \left\| Q_1 - Q_2 \right\|_\infty.
    \end{equation}
    Combining~\eqref{RKHS_contraction_1} and~\eqref{RKHS_contraction_2} implies
    \begin{equation*}
        \left| V_{Q_1}(x)- V_{Q_2}(x) \right| \leq \left\| Q_1-Q_2 \right\|_{\infty} \qquad \forall x \in \X, 
    \end{equation*}
    and therefore    \begin{equation}\label{ineq:thm5_V_Q}
        \left\| V_{Q_1} - V_{Q_2} \right\|_\infty \leq \left\| Q_1 - Q_2 \right\|_\infty .
    \end{equation}
    Now fix $(x,a)\in \X \times \A$. Using Definition~\ref{RKHS_operatorT}, the reward term cancels in the difference:
    {\small
    \begin{equation*}
        \left( \mathcal{T}^\vartheta Q_1 - \mathcal{T}^\vartheta Q_2 \right)(x,a) = \sum_{y \in \X} \Tilde{p}(y \mid x,a, \mu_E) \big( V_{Q_1}(y) - V_{Q_2}(y)\big).
    \end{equation*}}
    Therefore,
    {\small
    \begin{equation*}
        \left| \left( \mathcal{T}^\vartheta Q_1 - \mathcal{T}^\vartheta Q_2 \right)(x,a) \right| \leq \sum_{y \in \X} \Tilde{p}(y \mid x,a, \mu_E) \left\| V_{Q_1} - V_{Q_2}\right\|_\infty.
    \end{equation*}}
    Since $\Tilde{p}$ is sub-stochastic with total mass $\kappa \in (0,1)$, we get
    \begin{equation*}
        \begin{aligned}
            \left| \left( \mathcal{T}^\vartheta Q_1 - \mathcal{T}^\vartheta Q_2 \right)(x,a) \right|  \leq \kappa \left\| V_{Q_1} - V_{Q_2}\right\|_\infty  \overset{\eqref{ineq:thm5_V_Q}}{\leq} \kappa \left\| Q_1 - Q_2 \right\|_\infty .
        \end{aligned}
    \end{equation*}
    Taking supremum over $(x,a)$ gives
    \begin{equation*} \left\|\mathcal{T}^\vartheta Q_1 - \mathcal{T}^\vartheta Q_2  \right\|_\infty  \leq \kappa \left\| Q_1 -Q_2 \right\|_\infty.
    \end{equation*}
    This proves that $\mathcal{T}^\vartheta$ is a $\kappa$-contraction in $\| \cdot\|_\infty$. Since $\kappa \in (0,1)$ and $\left
    ( \R^{\X \times \A} , \; \| \cdot \|_\infty \right)$ is complete, $\mathcal{T}^\vartheta$ admits a unique fixed point by the Banach fixed point theorem.
\end{proof}

%{\color{blue} As a consequence, for every fixed $\vartheta$, the Bellman equation admits a unique solution $Q_\vartheta$.  We next study how this fixed point varies with respect to the dual parameter. To this end, we rewrite the fixed-point equation as a non-linear equation in $(Q,\vartheta)$ and invoke the implicit function theorem. }

As a consequence, for every fixed dual variable \(\vartheta\), the Bellman equation admits a unique solution \(Q_\vartheta\). We now investigate the dependence of this fixed point \(Q_\vartheta\) on  \(\vartheta\). Setting
\[
F(Q, \vartheta) \coloneqq Q - \mathcal{T}^\vartheta Q,
\]
the equation $F(Q, \vartheta) = 0$ has $Q_\vartheta$ as its unique solution by Theorem~\ref{thm:Q_contraction}. The implicit function theorem will yield Fréchet differentiability of $\vartheta \mapsto Q_\vartheta$ once we verify that $F$ is $C^1$ and that the Jacobian $\nabla_Q F(Q, \vartheta) = I - D_Q \mathcal{T}^\vartheta(Q)$ is invertible.
 
\begin{proposition}\label{prop:Q-differentiable}
Under Assumption~\ref{ass:minorization}, the map $\vartheta \mapsto Q_\vartheta$ is Fréchet differentiable on $\R^\X \times \mathcal{H}_k$. The Fréchet derivative of $\mathcal{T}^\vartheta(Q)$ with respect to $Q$ in direction $H \in \ell_\infty(\X \times \A)$ satisfies
\begin{equation}\label{eq:DT-bound}
\| D_Q \mathcal{T}^\vartheta(Q)[H] \|_\infty \leq \kappa\, \|H\|_\infty,
\end{equation}
so the operator $I - D_Q \mathcal{T}^\vartheta(Q)$ is invertible via the Neumann series
\begin{equation}\label{eq:F_inverse}
\big( I - D_Q \mathcal{T}^\vartheta(Q) \big)^{-1} = \sum_{n = 0}^\infty \big( D_Q \mathcal{T}^\vartheta(Q) \big)^n,
\end{equation}
with the uniform bound
\begin{equation}\label{eq:resolvent-bound}
\big\| \big( I - D_Q \mathcal{T}^\vartheta(Q) \big)^{-1} \big\|_\infty \leq \frac{1}{1 - \kappa}.
\end{equation}
\end{proposition}
 
\noindent A direct computation shows that $V_Q$ is $C^\infty$ as a map $\ell_\infty(\X \times \A) \to \ell_\infty(\X)$, with Fréchet derivative
\[
D_Q V_Q[H](x) = \sum_{b \in \A} \frac{e^{Q(x, b)}}{\sum_{b' \in \A} e^{Q(x, b')}}\, H(x, b),
\]
satisfying $|D_Q V_Q[H](x)| \leq \|H\|_\infty$. The sub-stochasticity of $\Tilde{p}$ then yields~\eqref{eq:DT-bound}, the Neumann series~\eqref{eq:F_inverse} converges, and the implicit function theorem gives the Fréchet differentiability of $Q_\vartheta$. Full computational details are in Appendix~\ref{app:proof-rkhs-Q-differentiable}.

\subsection{Score Function and Primal Recovery}\label{sec:rkhs-score}

We now define a score function that compares the policy $\pi_\vartheta$ induced by the dual parameter with the expert occupation measure. 
\begin{equation}\label{eq:score_fn}
\mathcal{V}(\vartheta) \coloneqq \sum_{(x, a) \in \X \times \A} \nu_{\pi_E}(x, a)\, \log \pi_\vartheta(a \mid x).
\end{equation}
Its stationary points will be shown to encode the constraints of the primal problem. By Definition~\ref{def:occupation_measures}, $\mathcal{V}$ admits the equivalent long-run representation
\begin{equation}\label{eq:score_fn_2}
\mathcal{V}(\vartheta) = \lim_{T \to \infty} \frac{1}{T} \sum_{t=0}^{T-1} \E^{\pi_E, \mu_E} \big[ \log \pi_\vartheta(a_t \mid x_t) \big].
\end{equation}

%We now return to the occupation measures introduced in Section~\ref{sect:occupation_measures} and use them todefine a score function for recovering primal feasibility. The basic idea is to compare the policy induced by the dual parameter $\vartheta$ with the expert occupation measure. This leads to a score whose stationary points encode the constraints of the primal problem. Motivated by this observation, we define the score function by

\begin{remark}[Maximum-likelihood interpretation]\label{rem:rkhs-MLE}
$\mathcal{V}(\vartheta)$ is the expected log-likelihood of the expert trajectory under the dual-parametrised policy $\pi_\vartheta$. Stationary points of $\mathcal{V}$ are therefore maximum-likelihood estimators of $\vartheta$ given the expert demonstrations, and Theorem~\ref{thm:rkhs-score-stationarity} below shows that these stationary points additionally solve the primal--dual problem.
\end{remark}

%{\color{blue}
%The next result shows that stationary points of the score recover the primalconstraints and, at the same time, identify optimal solutions of the primal and dual problems.}

The following identity is the key technical ingredient in the score-stationarity proof; it makes explicit how the sub-stochastic kernel $\Tilde{p}$ interacts with occupation measures.
 
\begin{lemma}[Sub-stochastic flow identity]\label{lem:substochastic-flow}
Under Assumption~\ref{ass:minorization}, for any stationary policy $\pi$ whose occupation measure is taken under the expert mean-field term $\mu_E$, and any $y \in \X$,
\[
\sum_{(x, a) \in \X \times \A} \nu_\pi(x, a)\, \Tilde{p}(y \mid x, a, \mu_E) = \nu_\pi^\X(y) - \xi(y).
\]
\end{lemma}

\begin{proof}
By the decomposition $p = \xi + \Tilde{p}$ from~\eqref{def:sub_prob_measure} and the flow identity (Lemma~\ref{lem:flow}),
\[
\nu_\pi^\X(y) = \sum_{(x, a) \in \X \times \A} \nu_\pi(x, a)\, \big( \xi(y) + \Tilde{p}(y \mid x, a, \mu_E) \big).
\]
Since $\nu_\pi$ is a probability measure on $\X \times \A$, we have $\sum_{(x, a)} \nu_\pi(x, a) = 1$, so the $\xi$-contribution equals $\xi(y)$. Rearranging yields the claim.
\end{proof}

\begin{theorem}\label{thm:rkhs-score-stationarity}
    If $\nabla \mathcal{V}(\vartheta^\star)=0$, then
    \[
    \vartheta^\star \in \operatorname*{arg\,min}_{\vartheta \in \R^\X \times \mathcal{H}_k} \mathcal{D}(\vartheta) \quad \text{and} \quad \pi_{\vartheta^\star} \in \operatorname*{arg\,max}_{\pi \in \mathcal{P}(\A \mid \X)} \, \opttH.
    \]
\end{theorem}

The proof proceeds in two steps. First, differentiating the soft Bellman equations and applying Lemma~\ref{lem:substochastic-flow} yields the gradient identity
\[
\nabla_\vartheta \mathcal{V}(\vartheta)
= \langle f \rangle_{\pi_E, \mu_E} - \lim_{T \to \infty} \frac{1}{T} \sum_{t=0}^{T-1} \E^{\pi_\vartheta, \mu_E}[f(x_t, a_t)],
\]
where $f(x, a) \coloneqq (e_x,\ \varphi(x, a, \mu_E)) \in \R^\X \times \mathcal{H}_k$. Hence $\nabla \mathcal{V}(\vartheta^\star) = 0$ if and only if $\pi_{\vartheta^\star}$ satisfies both constraints of $\opttH$. Second, a strong-duality argument using the Lagrangian shows that under feasibility, $\opttH = \mathcal{D}(\vartheta^\star)$, simultaneously identifying $\vartheta^\star$ as a dual minimiser and $\pi_{\vartheta^\star}$ as a primal optimum. The full argument is in Appendix~\ref{app:proof-rkhs-score-stationarity}.
 
The stationary points of $\mathcal{V}$ thus encode the optimal solutions of both the primal and dual problems. This characterisation is the basis for the gradient-based algorithm of the next subsection, whose justification rests on smoothness of $\mathcal{V}$.

\subsection{Smoothness of the Score Function and Gradient-Based Solution}\label{sec:rkhs-smoothness}
 
We now establish a Lipschitz smoothness bound for the score function. This bound provides the step-size scale for the gradient-based update introduced below.

%We now establish a smoothness bound for the score function. This will provide anatural step-size scale for the gradient-based update considered below. The argument proceeds by deriving a uniform Lipschitz bound for the score gradient through the Bellman {\color{blue} equations} introduced in Section~\ref{sect:RKHS_Bellman_equations}.

\begin{theorem}[Lipschitz smoothness of $\mathcal{V}$]\label{RKHS_thm_smoothness}
The score function $\mathcal{V} : \R^\X \times \mathcal{H}_k \to \R$ is Fréchet differentiable, and its gradient is $L$-Lipschitz with respect to the Hilbert-space norm $\|\cdot\|_{\R^\X \times \mathcal{H}_k}$:
\[
\big\| \nabla \mathcal{V}(\vartheta_1) - \nabla \mathcal{V}(\vartheta_2) \big\|_{\R^\X \times \mathcal{H}_k}
\leq L\, \|\vartheta_1 - \vartheta_2\|_{\R^\X \times \mathcal{H}_k},
\]
with
\[
L = \frac{|\A|\, K^2\, (\kappa + 1)}{(1 - \kappa)^3},
\qquad
K \coloneqq \max_{(x, a) \in \X \times \A} \|f(x, a)\|_{\R^\X \times \mathcal{H}_k},
\]
where $f(x, a) \coloneqq (e_x,\ \varphi(x, a, \mu_E))$.
\end{theorem}

%\begin{theorem}\label{RKHS_thm_smoothness}
 %   Score function $\mathcal{V}$ is $L$-smooth where
  %  \[
   % L = \frac{\left| \A \right| K^2 (\kappa+1)}{(1-\kappa)^3} \quad \text{and} \quad  K= \max_{(x,a)\in \X \times \A} \left\| f(x,a) \right\|_{\R^\X \times \mathcal{H}_k}
   % \]
%\end{theorem}

\begin{proof}[Proof outline]
The score gradient depends on $\vartheta$ only through the soft $Q$-function $Q_\vartheta$ and the induced policy $\pi_\vartheta$. We establish, in order:
\begin{enumerate}[label=(\roman*), leftmargin=2.5em]
\item a $K/(1 - \kappa)$-Lipschitz bound for $Q_\vartheta$ in $\vartheta$, obtained from the contraction property of $\mathcal{T}^\vartheta$ (Theorem~\ref{thm:Q_contraction});
\item an $|\A|\, K / (1 - \kappa)$-Lipschitz bound for the policy-averaging operator $\Pi(\vartheta)$, using the softmax-Lipschitz estimate of \citet[Proposition~4]{GaoPav18} together with an $\ell_2$-to-$\ell_\infty$ conversion that costs a factor of $\sqrt{|\A|}$ on each side;
\item a resolvent representation $\nabla_\vartheta Q_\vartheta = (I - \Tilde{P}\, \Pi(\vartheta))^{-1} F$, where $\Tilde{P}$ is the linear map induced by $\Tilde{p}$ and $F(x, a) = f(x, a)$, which converts smoothness of $\nabla_\vartheta Q_\vartheta$ into a question about the dependence of the resolvent on $\Pi(\vartheta)$;
\item the explicit constant follows from the identity $B^{-1} - A^{-1} = A^{-1}(A - B) B^{-1}$ applied to the two resolvents, the uniform bound $\|(I - \Tilde{P}\Pi(\vartheta))^{-1}\|_\infty \leq 1 / (1 - \kappa)$ from Proposition~\ref{prop:Q-differentiable}, and the Lipschitz bound from (ii).
\end{enumerate}
Combining (i)--(iv) yields a Lipschitz bound for $\nabla_\vartheta \log \pi_\vartheta$ that propagates through the expectation in~\eqref{eq:score_fn} to give the stated constant. The detailed operator computations are in Appendix~\ref{app:proof-rkhs-smoothness}.
\end{proof}

\noindent The smoothness bound suggests the natural step-size range $\gamma \in (0, 1/L]$ for the constant-step-size gradient-ascent procedure below.

%Having an explicit $L$-smoothness bound for $\mathcal{V}$ is useful because it allows us to control the change of the score under perturbations of $\vartheta$. In particular, it provides a natural step-size scale for a gradient-based update, which we use to define a simple ascent algorithm on $\mathcal{V}$ as follows. {\color{red} Accordingly, we adopt the gradient-ascent procedure below.}

\begin{algorithm}[H]
\caption{Gradient Ascent for the Score Function}
\label{alg:rkhs_mll_ga}
\begin{algorithmic}[1]
\REQUIRE Initial point $\vartheta_0\in\R^\X\times\mathcal{H}_k$, step size $\gamma\in(0,1/L]$, gradient tolerance $\varepsilon>0$, number of iterations $K\in\mathbb{N}$
\STATE $\vartheta\gets\vartheta_0$
\FOR{$k=0,1,\dots,K-1$}
    \IF{$\|\nabla\mathcal{V}(\vartheta)\|_{\R^\X\times\mathcal{H}_k}\le\varepsilon$}
        \STATE \textbf{break}
    \ENDIF
    \STATE $\vartheta \gets \vartheta + \gamma\,\nabla\mathcal{V}(\vartheta)$
\ENDFOR
\STATE Recover the policy $\pi_\vartheta$ from $Q_\vartheta$ via \eqref{eq:policy_theta}.
\RETURN $\vartheta$ and $\pi_\vartheta$
\end{algorithmic}
\end{algorithm}

%%\begin{algorithm}[H]
%\caption{Gradient Ascent for the Score Function}
%\label{alg:rkhs_mll_ga}
%\begin{algorithmic}[1]
 %   \REQUIRE  $\vartheta_0 \in \R^{\X}\times \mathcal{H}_k$, $\gamma \in \left(0,\, \frac{1}{L}\right]$, and $K\in \mathbb{N}$
  %  \STATE Set $\vartheta \leftarrow \vartheta_0$
   % \FOR{$k=0,1,\dots,K-1$}
    %    \STATE $\vartheta_{k+1} \leftarrow \vartheta_{k} + \gamma\, \nabla \mathcal{V}(\vartheta_k)$
    %\ENDFOR
    %\STATE Compute the policy $\pi_{\vartheta_K}$ induced by $\vartheta_K$
   % \RETURN $\vartheta_K$ and $\pi_{\vartheta_K}$
%\end{algorithmic}
%\end{algorithm}

%It remains to justify the choice of constant step-size in the update. The smoothness property established in Theorem~\ref{RKHS_thm_smoothness} suggests restricting to $\gamma \in \left( 0,\ \frac{1}{L}\right] $ as a natural range for a fixed step size. The following theorem formalizes the resulting convergence statement. The argument relies on the standard \emph{ascent lemma}. We still include it here for completeness.  

\begin{theorem}[Best-iterate gradient bound]
\label{thm:rkhs_const_stepsize}
Let $\mathcal{V}^\star \coloneqq \sup_\vartheta \mathcal{V}(\vartheta) \;\le\; 0$. For any fixed step-size $\gamma\in(0,1/L]$, the iterates of Algorithm~\ref{alg:rkhs_mll_ga} satisfy, for every $K\ge 1$,
\[
\min_{0\le k<K}\bigl\|\nabla \mathcal{V}(\vartheta_k)\bigr\|_{\R^\X\times\mathcal{H}_k}^{2}
\;\le\;
\frac{2\bigl(\mathcal{V}^\star-\mathcal{V}(\vartheta_0)\bigr)}{\gamma K}.
\]
In particular, $\|\nabla\mathcal{V}(\vartheta_k)\|_{\R^\X\times\mathcal{H}_k}\to 0$ as $k\to\infty$.
\end{theorem}

%\begin{theorem}
%\label{thm:rkhs_const_stepsize}
%For fixed step-size $\gamma \in \left( 0,\ \frac{1}{L}\right]$, the sequence generated by Algorithm~\ref{alg:rkhs_mll_ga} satisfies
%\[\|\nabla \mathcal{V}(\vartheta_k)\| \underset{k \to \infty}{\xrightarrow{\hspace*{0.8cm}}} 0.\]
%\end{theorem}

\begin{proof}
By the $L$-smoothness of $\mathcal{V}$ (Theorem~\ref{RKHS_thm_smoothness}), the standard ascent lemma gives, for every $\vartheta, \vartheta' \in \R^\X \times \mathcal{H}_k$,
%Since $\mathcal{V}$ is $L$-smooth, the standard quadratic lower bound holds for all $\vartheta,\vartheta'\in \R^{\X}\times \mathcal{H}_k$,
\begin{equation}
\label{eq:rkhs_ascent_lemma}
\mathcal{V}(\vartheta') \ge \mathcal{V}(\vartheta)
+ \left\langle \left.\nabla_\vartheta \mathcal{V}(\vartheta)\right|_{\vartheta=\vartheta},\,\vartheta'-\vartheta\right\rangle
- \frac{L}{2}\,\|\vartheta'-\vartheta\|^2 .
\end{equation}
Applying \eqref{eq:rkhs_ascent_lemma} at $\vartheta=\vartheta_k$, and taking
\(
\vartheta'=\vartheta_{k+1}
=\vartheta_k+\gamma\left.\nabla_\vartheta \mathcal{V}(\vartheta)\right|_{\vartheta=\vartheta_k},
\) we obtain
\begin{align*}
\mathcal{V}(\vartheta_{k+1})
&\ge \mathcal{V}(\vartheta_k)
+ \left\langle
\left.\nabla_\vartheta \mathcal{V}(\vartheta)\right|_{\vartheta=\vartheta_k},
\,\vartheta_{k+1}-\vartheta_k
\right\rangle
 \\ & \qquad- \frac{L}{2}\,\|\vartheta_{k+1}-\vartheta_k\|^2 \\
&= \mathcal{V}(\vartheta_k)
+ \left\langle
\left.\nabla_\vartheta \mathcal{V}(\vartheta)\right|_{\vartheta=\vartheta_k},
\,\gamma\left.\nabla_\vartheta \mathcal{V}(\vartheta)\right|_{\vartheta=\vartheta_k}
\right\rangle
\\ & \qquad- \frac{L}{2}\,\left\|\gamma\left.\nabla_\vartheta \mathcal{V}(\vartheta)\right|_{\vartheta=\vartheta_k}\right\|^2 \\
&= \mathcal{V}(\vartheta_k)
+ \Big(\gamma-\frac{L\gamma^2}{2}\Big)
\left\|
\left.\nabla_\vartheta \mathcal{V}(\vartheta)\right|_{\vartheta=\vartheta_k}
\right\|^2.
\end{align*}
Using $\gamma\le 1/L$ gives $\gamma-\frac{L\gamma^2}{2}\ge \gamma/2$, hence
\begin{equation}
\label{eq:rkhs_increment_eval}
\mathcal{V}(\vartheta_{k+1}) \ge \mathcal{V}(\vartheta_k)
+ \frac{\gamma}{2}
\left\|
\left.\nabla_\vartheta \mathcal{V}(\vartheta)\right|_{\vartheta=\vartheta_k}
\right\|^2.
\end{equation}
Therefore $\big(\mathcal{V}(\vartheta_k)\big)_{k\ge 0}$ is non-decreasing. Next, $\mathcal{V}$ is bounded above. Indeed, by definition
\[
\mathcal{V}(\vartheta)=\sum_{(x,a)\in \X \times \A}\nu_{\pi_E}(x,a)\,\log \pi_\vartheta(a\mid x),
\]
and since $0<\pi_\vartheta(a\mid x)\le 1$ we have $\log \pi_\vartheta(a\mid x)\le 0$. Because $\nu_{\pi_E}$ is a probability measure on $\X\times\A$, it follows that $\mathcal{V}(\vartheta)\le 0$ for all $\vartheta$. Hence $\big(\mathcal{V}(\vartheta_k)\big)_{k\geq 0}$ is monotone and bounded above, and thus convergent.
Summing~\eqref{eq:rkhs_increment_eval} from $k=0$ to $K-1$ and using $\mathcal{V}(\vartheta_K)\le \mathcal{V}^\star$,
\[
\frac{\gamma}{2}\sum_{k=0}^{K-1}\bigl\|\nabla\mathcal{V}(\vartheta_k)\bigr\|^2
\;\le\;
\mathcal{V}(\vartheta_K)-\mathcal{V}(\vartheta_0)
\;\le\;
\mathcal{V}^\star-\mathcal{V}(\vartheta_0).
\]
The minimum of a non-negative sequence is bounded by its average, so
\[
\min_{0\le k<K}\bigl\|\nabla\mathcal{V}(\vartheta_k)\bigr\|^{2}
\;\le\;
\frac{1}{K}\sum_{k=0}^{K-1}\bigl\|\nabla\mathcal{V}(\vartheta_k)\bigr\|^{2}
\;\le\;
\frac{2\bigl(\mathcal{V}^\star-\mathcal{V}(\vartheta_0)\bigr)}{\gamma K},
\]
proving the displayed bound. Letting $K\to\infty$ in the summed inequality yields $\displaystyle \sum_{k=0}^{\infty}\|\nabla\mathcal{V}(\vartheta_k)\|^2<\infty$, hence $\|\nabla\mathcal{V}(\vartheta_k)\|\to 0$.
\end{proof}

%\noindent Summing~\eqref{eq:rkhs_increment_eval} from $k=0$ to $K-1$ yields
%\[\frac{\gamma}{2}\sum_{k=0}^{K-1} \left\| \left.\nabla_\vartheta \mathcal{V}(\vartheta)\right|_{\vartheta=\vartheta_k} \right\|^2 \le \mathcal{V}(\vartheta_K)-\mathcal{V}(\vartheta_0) \le 0-\mathcal{V}(\vartheta_0). \]
%Letting $K\to\infty$ gives
%\begin{equation}\label{RKHS_const_step_size_V_div}  \sum_{k=0}^{\infty} \left\| \left.\nabla_\vartheta \mathcal{V}(\vartheta)\right|_{\vartheta=\vartheta_k} \right\|^2 <\infty,  \end{equation}
%which forces
%\[ \left\| \left.\nabla_\vartheta \mathcal{V}(\vartheta)\right|_{\vartheta=\vartheta_k} \right\| \underset{k \to \infty}{\xrightarrow{\hspace*{0.8cm}}} 0, \]
%since otherwise there would exist $\varepsilon>0$ and infinitely many $k$ with $\left\| \left.\nabla_\vartheta \mathcal{V}(\vartheta)\right|_{\vartheta=\vartheta_k} \right\|^2 \geq \varepsilon$, making the series diverge, which contradicts with~\eqref{RKHS_const_step_size_V_div}. This proves the claim.
%\end{proof}

Theorem~\ref{thm:rkhs_const_stepsize} justifies the use of constant-step-size gradient ascent and completes the RKHS-based solution framework. Combining Theorems~\ref{thm:rkhs-score-stationarity} and~\ref{thm:rkhs_const_stepsize}, the limit point $\vartheta_\star$ of Algorithm~\ref{alg:rkhs_mll_ga} yields, via the soft Bellman equation~\eqref{eq:policy_theta}, a policy $\pi_{\vartheta_\star}$ that solves the maximum causal entropy problem; by Proposition~\ref{prop:rkhs-mfe}, the pair $(\pi_{\vartheta_\star}, \mu_E)$ is a mean-field equilibrium under the recovered reward.

%{\color{blue}
%Theorem~\ref{thm:rkhs_const_stepsize} justifies the use of constant-step-size gradient ascent for optimising the score function and thereby completes the RKHS-based solution framework. In particular, the gradient ascent algorithm can be used to compute a stationary point of the score function, from which one can recover solutions to both the primal and dual problems.
%}
\section{Numerical Illustrations} \label{sec:numerics}

We now present two numerical illustrations of the proposed framework. The first example
concerns the finite-dimensional linear reward formulation of Section~\ref{sect:linear_reward}, while the second
illustrates the RKHS-based formulation of Section~\ref{sec:rkhs}. In each case, we first specify a
stationary mean-field game and construct an expert equilibrium, which provides the expert
mean-field term and the corresponding long-run statistics used in the inverse problem.
This evaluation procedure is in line with the MFG-IRL literature, where expert behaviour is
generated from an equilibrium and the inverse method is assessed through its ability to recover
the induced behaviour from expert statistics \citep{YaLiLiHu22,ChZhLiWi23}. The inverse procedure
is then applied \emph{without access} to the underlying reward parameters.

%The inverse procedures are then applied without using the underlying reward parameters.
The results demonstrate that the proposed methods recover policies that closely match the
expert behaviour and reproduce the associated equilibrium statistics.

\subsection{Malware-Spread Model}

We consider an average-reward stationary mean-field game motivated by the malware-spread model studied by \citet{SuAd19}. The discounted counterpart of this type of numerical experiment was considered in our earlier work~\citep{AnKaSa24}.

The state space is $\X=\{0,0.1,0.2,\ldots,0.9\}$,
where \(x\in \X\) represents the malware severity level of an agent. The action space is
\(\A=\{0,1\}\), where \(0\) corresponds to doing nothing and \(1\) corresponds to repair.

The transition kernel is given by
\[
p(y\mid x,a,\mu)
=
\begin{cases}
\dfrac{\1_{\{y\geq x\}}}
{\left|\{z\in \X:z\geq x\}\right|}, & a=0,\\[3mm]
\1_{\{y=0\}}, & a=1.
\end{cases}
\]
Hence, choosing repair resets the state to \(0\), whereas choosing to do nothing moves the
agent uniformly among the current and worse severity states.

In the infinite-population limit, the one-stage reward for each agent is given by
\[
    r(x,a,\mu)=\theta_1x+\theta_2x\mu_{\mathrm{av}}+\theta_3a, \quad \text{ where} \quad \mu_{\mathrm{av}}\coloneqq\sum_{z\in \X}z\mu(z).
\]
The quantity \(\mu_{\mathrm{av}}\) denotes the average malware severity in the population.
The coefficient \(\theta_1\) determines the reward loss due to an agent's own malware
severity, \(\theta_2\) determines the additional reward loss caused by the interaction between
the agent's severity and the population average severity, and \(\theta_3\) represents the
reward loss associated with choosing repair. In the numerical experiment, we take
\[
    \theta_1=-0.1,\qquad \theta_2=-1,\qquad \theta_3=-0.4.
\]
In the inverse problem, the parameter vector \(\theta=(\theta_1,\theta_2,\theta_3)\) is
unknown. Therefore, we assume that the reward belongs to the finite-dimensional linear class
\[
    \mathcal R
    =
    \left\{
    r(x,a,\mu)=\langle \theta,\varphi(x,a,\mu)\rangle:
    \theta\in\mathbb R^3
    \right\},
\]
where the feature map is $    \varphi(x,a,\mu)=\bigl(x,\;x\mu_{\mathrm{av}},\;a\bigr)$.

We first solve the forward stationary mean-field game under the known reward structure
above. The computed expert policy is
\[
\pi_E =
\begin{bmatrix}
1 & 1 & 1 & 1 & 1 & 0 & 0 & 0 & 0 & 0\\
0 & 0 & 0 & 0 & 0 & 1 & 1 & 1 & 1 & 1
\end{bmatrix},
\]
where the columns correspond to the states \(0,0.1,\ldots,0.9\), and the rows
correspond to the actions \(0\) and \(1\). The corresponding mean-field term is
\[
\begin{aligned}
\mu_E =&
\Big[0.404683,\;0.045527,\;0.052031,\;0.060702,\;0.072843,\\
&\qquad \qquad 0.072843,\;0.072843,\;0.072843,\;0.072843,\;0.072843\Big].
\end{aligned}
\]
The pair \((\pi_E,\mu_E)\) is a MFE: \(\pi_E\) is optimal for the fixed
mean-field term \(\mu_E\), and \(\mu_E\) is invariant under the transition kernel induced
by \(\pi_E\).
The average malware severity under this equilibrium is $\mu_{\mathrm{av}}= 0.317257$.

The expert occupation measure is given by $\nu_E(x,a)=\mu_E(x)\pi_E(a\mid x)$. For the MFE above, this gives
{\small
\[
\nu_E
=
\begin{bmatrix}
0.4047 & 0.0455 & 0.0520 & 0.0607 & 0.0728
& 0 & 0 & 0 & 0 & 0\\
0 & 0 & 0 & 0 & 0
& 0.0728 & 0.0728 & 0.0728 & 0.0728 & 0.0728
\end{bmatrix}.
\]
}
Using this occupation measure, the expert feature expectation vector is computed as
\[
    \langle \varphi\rangle_{\pi_E,\mu_E}
    =
    \sum_{(x,a)\in \X\times \A}
    \varphi(x,a,\mu_E)\nu_E(x,a)
    =
    \begin{bmatrix}
    0.317257 & 0.100652 & 0.364214
    \end{bmatrix}.
\]
We now solve the average-reward maximum causal entropy IRL problem using the expert
mean-field term \(\mu_E\) and the expert feature expectation vector
\(\langle\varphi\rangle_{\pi_E,\mu_E}\). The dual problem derived in
Section~\ref{sec:linear-dual} is solved by gradient descent. At the minimiser of the dual
objective, the optimal occupation measure \(\nu^\star\) is recovered, and the corresponding
policy $\pi$ is obtained from
\[
    \pi_{\nu^\star}(a\mid x)
    =
    \frac{\nu^\star(x,a)}{\nu^{\star,\X}(x)},
    \qquad
    \nu^{\star,\X}(x)=\sum_{a\in \A}\nu^\star(x,a).
\]

Starting from zero initial dual variables, we run the algorithm for \(8\times 10^4\)
iterations with step size \(\gamma=0.05<1/L\approx 0.0591\), using the smoothness
bound from Section~\ref{sec:linear-regularity}. The recovered policy is
{\small
\[
   \pi_{\nu^\star}
=
\begin{bmatrix}
0.9948 & 0.9949 & 0.9905 & 0.9761 & 0.8902
& 0.1361 & 0.0183 & 0.0040 & 0.0011 & 0.0003\\
0.0052 & 0.0051 & 0.0095 & 0.0239 & 0.1098
& 0.8639 & 0.9817 & 0.9960 & 0.9989 & 0.9997
\end{bmatrix}. 
\]
}
The recovered policy is close to the expert policy. For malware severity levels below
\(0.5\), the probability of repair is close to zero, whereas for malware severity levels above
\(0.5\), the probability of repair is close to one. At the boundary state \(x=0.5\), the
recovered repair probability is \(0.8639\). Quantitatively, the largest entrywise policy error is
\[
    \|\pi_{\nu^\star}-\pi_E\|_\infty
    =
    \max_{(x,a)\in \X\times \A}
    |\pi_{\nu^\star}(a\mid x)-\pi_E(a\mid x)|
    = 0.1361.
\]
Figure~\ref{fig:linear-malware-ten-state} visualises the numerical results. The left panel
shows that the dual objective decreases steadily and stabilises as the iterations progress.
The right panel compares the expert and recovered repair probabilities, showing that the
recovered policy follows the same behaviour as the expert policy across the malware severity
levels.

\begin{figure}[H]
    \centering
    \begin{minipage}[t]{0.48\columnwidth}
        \centering

        \includegraphics[width=\linewidth]{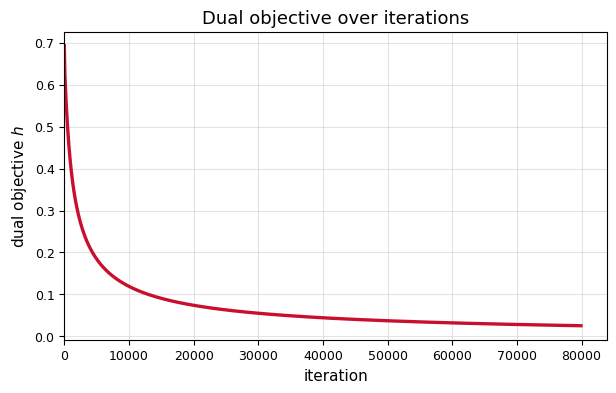}

        \vspace{0.4em}
        \small (a) Convergence of the dual objective
    \end{minipage}
    \hfill
    \begin{minipage}[t]{0.48\columnwidth}
        \centering

        \includegraphics[width=\linewidth]{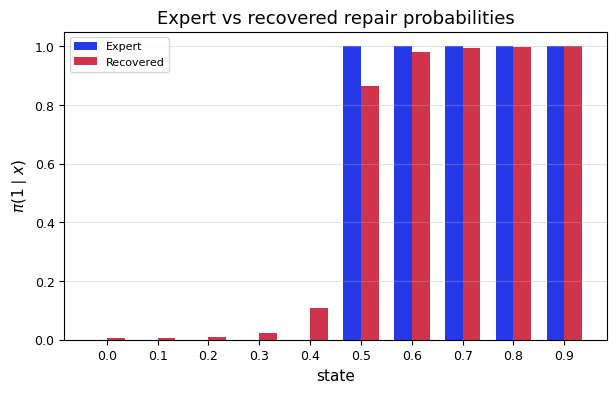}

        \vspace{0.4em}
        \small (b) Expert and recovered repair probabilities
    \end{minipage}
    \caption{Numerical results for the ten-state malware spread model.}
    \label{fig:linear-malware-ten-state}
\end{figure}

\subsection{RKHS-Based Consumer-Choice Model}
\label{sec:rkhs-simulation}

We next illustrate the RKHS-based formulation on an average-reward stationary consumer-choice model motivated by the consumer-choice example of \citet{Neumann20}. In this example, agents choose between two service providers through stay/change actions, and the inverse procedure is used to recover the expert equilibrium behaviour from the corresponding long-run statistics.

A state is written as $x=(i,j)$, where $i\in\{1,2\}$ denotes the
provider currently used by the agent and $j\in\{1,2\}$ denotes the preferred provider.
Thus
\[
\mathcal X=\{(1,1),(1,2),(2,1),(2,2)\},
\qquad
\mathcal A=\{\mathrm{stay},\mathrm{change}\}.
\]
For $x=(i,j)$, define the provider used after action $a$ by
\[
i^+(x,a)
\coloneqq
\begin{cases}
i, & a=\mathrm{stay},\\
3-i, & a=\mathrm{change},
\end{cases}
\qquad
x^+(x,a) \coloneqq (i^+(x,a),j).
\]
Hence, the action changes only the provider component, while the preference component
remains fixed. To ensure ergodicity and to satisfy the minorisation condition in
Assumption~\ref{ass:minorization}, we use the transition kernel
\[
p(y\mid x,a,\mu)
=
(1-\varepsilon)\1_{\{y=x^+(x,a)\}}
+
\varepsilon \nu_{\mathrm{noise}}(y),
\]
where $\nu_{\mathrm{noise}}$ is the uniform distribution on $\mathcal X$. In the experiment,
we set $\varepsilon=0.2$. Therefore, each transition probability is bounded below by
$0.05$, so that
\[
\xi=(0.05,0.05,0.05,0.05),
\qquad
\kappa=1-\sum_{y\in\mathcal X}\xi(y)=0.8.
\]

Let
\[
m_1(\mu)=\mu(1,1)+\mu(1,2),
\qquad
m_2(\mu)=\mu(2,1)+\mu(2,2)
\]
denote the population shares of providers $1$ and $2$, respectively. The expert data are
generated from the true non-linear reward
\[
\begin{aligned}
r(x,a,\mu)
={}&
0.1\log\bigl(m_{i^+(x,a)}(\mu)+\delta\bigr)
-
0.05\,m_{i^+(x,a)}^2(\mu)
+
0.05\,\1_{\{i^+(x,a)=1\}}  \\
&-
0.3\,\1_{\{a=\mathrm{change}\}}
-
0.1\,\1_{\{i^+(x,a)\neq j\}},
\end{aligned}
\]
where \(\delta=10^{-20}\) is used only to avoid the singularity of the logarithm at zero. This reward combines a non-linear mean-field dependence, a small bias
towards provider $1$, a switching cost, and a mismatch penalty.

Solving the forward stationary mean-field game with this reward yields the expert
mean-field equilibrium
\[
\mu_E=\big[0.45,0.25,0.05,0.25\big].
\]
 The expert policy is given by
\[
\pi_E
=
\begin{bmatrix}
1 & 1 & 0 & 1\\
0 & 0 & 1 & 0
\end{bmatrix},
\]
where the columns correspond to the states $(1,1),(1,2),(2,1),(2,2)$, respectively,
and the rows correspond to the actions $\mathrm{stay}$ and $\mathrm{change}$. Thus, the expert stays in states $(1,1)$, $(1,2)$, and $(2,2)$, while
changing provider in state $(2,1)$. The corresponding provider shares are
$m_1(\mu_E)=0.7$ and $m_2(\mu_E)=0.3$.

The inverse problem is solved without using the true reward parameters. The reward term
in the Lagrangian relaxation is modelled in an RKHS with a Gaussian kernel. For
$x=(i,j)$, define
\[
m_1(\mu)=\mu(1,1)+\mu(1,2),
\qquad
m_2(\mu)=\mu(2,1)+\mu(2,2),
\]
and encode $(x,a,\mu)$ as
\[
\eta(x,a,\mu)
=
\bigl(
\1_{\{i=1\}},
\1_{\{i=2\}},
\1_{\{j=1\}},
\1_{\{j=2\}},
\1_{\{a=\mathrm{stay}\}},
\1_{\{a=\mathrm{change}\}},
m_1(\mu),
m_2(\mu)
\bigr)\in\mathbb R^8 .
\]
The Gaussian kernel is then given by
\[
k\bigl((x,a,\mu),(x',a',\mu')\bigr)
=
\exp\left(
-\frac{\|\eta(x,a,\mu)-\eta(x',a',\mu')\|_2^2}{2\sigma^2}
\right),
\qquad
\sigma=0.9.
\]
We use the eight anchor points corresponding to all state-action pairs under the fixed
expert mean-field term $\mu_E=(0.45,0.25,0.05,0.25)$:
\[
\begin{aligned}
z_1 &= ((1,1),\mathrm{stay},\mu_E),&
z_2 &= ((1,1),\mathrm{change},\mu_E),\\
z_3 &= ((1,2),\mathrm{stay},\mu_E),&
z_4 &= ((1,2),\mathrm{change},\mu_E),\\
z_5 &= ((2,1),\mathrm{stay},\mu_E),&
z_6 &= ((2,1),\mathrm{change},\mu_E),\\
z_7 &= ((2,2),\mathrm{stay},\mu_E),&
z_8 &= ((2,2),\mathrm{change},\mu_E).
\end{aligned}
\]
At the expert mean-field term, the provider-share coordinates are
\[
m_1(\mu_E)=0.45+0.25=0.70,
\qquad
m_2(\mu_E)=0.05+0.25=0.30.
\]
With these anchors, the reward term used in the relaxed soft Bellman equation is
parametrised as
\[
r_\vartheta(x,a,\mu_E)
=
\zeta_x
+
\sum_{n=1}^{8}
c_n k\bigl((x,a,\mu_E),z_n\bigr),
\qquad
\vartheta=(\zeta,c_1,\ldots,c_8).
\]
Thus, the total number of parameters is
$|\mathcal X|+|\mathcal X||\mathcal A|=4+8=12$.

For each parameter $\vartheta$, the policy $\pi_\vartheta$ is obtained from the reduced
soft Bellman equation in Section~\ref{sect:RKHS_Bellman_equations}, using the frozen expert
mean-field term $\mu_E$ and the reward term $r_\vartheta$ above. We solve for
$\vartheta$ by gradient ascent on the log-likelihood score
\[
\mathcal V(\vartheta)
=
\sum_{(x,a)\in\mathcal X\times\mathcal A}
\nu_{\pi_E}(x,a)\log \pi_\vartheta(a\mid x).
\]
Starting from the zero parameter, we run the algorithm for $80000$ iterations with
step-size $\gamma=9\times 10^{-4}<1/L\approx 9.76\times 10^{-4}$, using the smoothness bound from Theorem~\ref{RKHS_thm_smoothness}. The final score
and gradient norm are
\[
\mathcal V(\vartheta)=-0.01969,
\qquad
\|\nabla_\vartheta \mathcal V(\vartheta)\|=0.01619.
\]

The recovered mean-field term is compared with the expert mean-field term in
Table~\ref{tab:rkhs-mean-field-comparison}. The resulting error is $\|\mu_\vartheta-\mu_E\|_1=0.007892$.

\begin{table}[htbp]
\centering
\caption{Comparison of the expert and recovered mean-field terms.}
\label{tab:rkhs-mean-field-comparison}
\begin{tabular}{c|c|c|c}
\hline
State $x$ & $\mu_E(x)$ & $\mu_\vartheta(x)$ & Absolute error \\
\hline
$(1,1)$ & $0.45$ & $0.446522$ & $0.003478$ \\
$(1,2)$ & $0.25$ & $0.250468$ & $0.000468$ \\
$(2,1)$ & $0.05$ & $0.053478$ & $0.003478$ \\
$(2,2)$ & $0.25$ & $0.249532$ & $0.000468$ \\
\hline
\end{tabular}
\end{table}

The recovered policy is reported in Table~\ref{tab:rkhs-policy-comparison}. The largest
absolute policy error is $0.06524$.

\begin{table}[htbp]
\centering
\caption{Comparison of the expert policy and the recovered RKHS policy.}
\label{tab:rkhs-policy-comparison}
\begin{tabular}{c|c|c|c|c}
\hline
State $x$ & Action $a$ & $\pi_E(a\mid x)$ & $\pi_\vartheta(a\mid x)$ & Absolute error \\
\hline
$(1,1)$ & stay   & $1$ & $0.998078$ & $0.001922$ \\
$(1,1)$ & change & $0$ & $0.001922$ & $0.001922$ \\
$(1,2)$ & stay   & $1$ & $0.969867$ & $0.030133$ \\
$(1,2)$ & change & $0$ & $0.030133$ & $0.030133$ \\
$(2,1)$ & stay   & $0$ & $0.065240$ & $0.065240$ \\
$(2,1)$ & change & $1$ & $0.934760$ & $0.065240$ \\
$(2,2)$ & stay   & $1$ & $0.969284$ & $0.030716$ \\
$(2,2)$ & change & $0$ & $0.030716$ & $0.030716$ \\
\hline
\end{tabular}
\end{table}

Figure~\ref{fig:rkhs-four-panel} summarises the numerical behaviour of the RKHS-based
inverse procedure. The score stabilises, the gradient norm decreases, the state-marginal
and feature-moment errors become small, and the recovered policy probabilities move
towards the expert actions.

\begin{figure}[H]
\centering

\begin{subfigure}{0.49\textwidth}
\centering
\includegraphics[width=\textwidth]{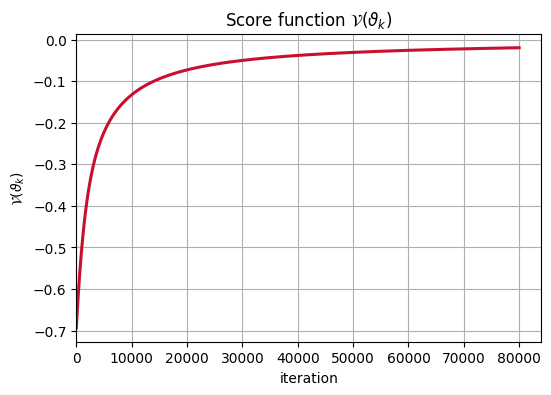}
\caption{Score function $\mathcal V(\vartheta_k)$}
\end{subfigure}
\hfill
\begin{subfigure}{0.49\textwidth}
\centering
\includegraphics[width=\textwidth]{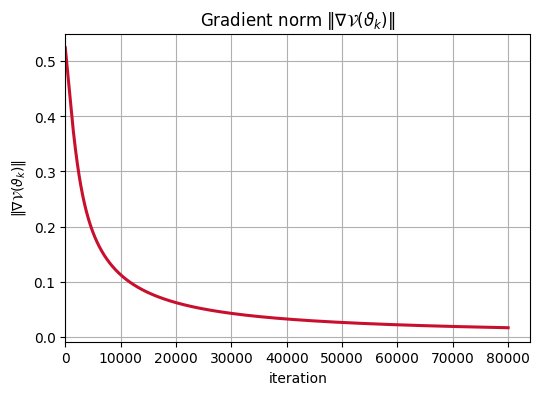}
\caption{Gradient norm $\|\nabla_\vartheta \mathcal V(\vartheta_k)\|$}
\end{subfigure}

\vspace{0.35cm}

\begin{subfigure}{0.49\textwidth}
\centering
\includegraphics[width=\textwidth]{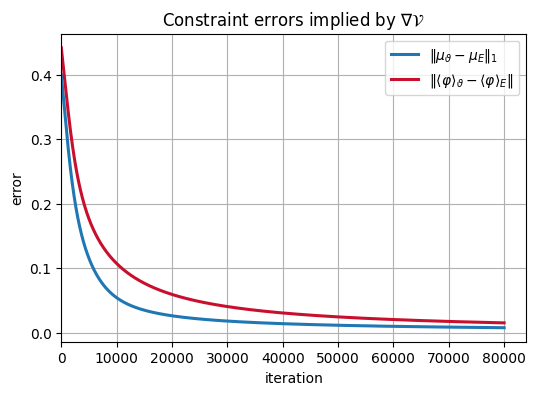}
\caption{Constraint errors}
\end{subfigure}
\hfill
\begin{subfigure}{0.49\textwidth}
\centering
\includegraphics[width=\textwidth]{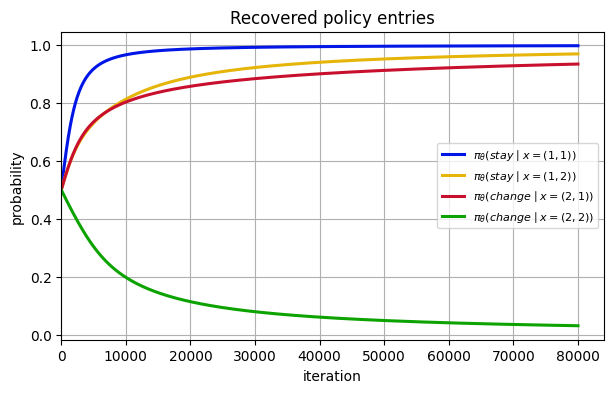}
\caption{Recovered policy entries}
\end{subfigure}

\caption{Numerical results for the RKHS-based consumer-choice model.}
\label{fig:rkhs-four-panel}
\end{figure}

\section{Conclusion}\label{sec:conclusion}
In this paper, we developed a unified maximum causal entropy framework for IRL in discrete-time, infinite-horizon stationary MFGs under the average-reward criterion. Within a single occupation-measure formulation, we treated both finite-dimensional linear rewards and infinite-dimensional rewards in an RKHS. For linear rewards, the problem reduces to a convex log-partition dual whose Lipschitz smoothness and strong convexity yield gradient descent with an explicit constant step size and linear convergence. For RKHS rewards, a Lagrangian relaxation gives a soft Bellman characterisation; the lack of a discount-factor contraction is resolved by a minorisation-based sub-stochastic kernel that restores a strict contraction, and Fr\'echet differentiability together with Lipschitz smoothness of the log-likelihood score yields a gradient-ascent algorithm with quantitative convergence guarantees in function space. Two numerical experiments, a malware-spread MFG and an RKHS-based consumer-choice model with
a non-linear expert reward, show that the recovered policies and mean-field terms closely match
the expert equilibrium.

Several directions remain open. Extending the framework to continuous state and action spaces, where the finite-dimensional log-partition representation is no longer available, would require a measure-theoretic treatment of the dual problem and of the associated Bellman fixed point in suitable function spaces. Finite-sample analysis from finitely many expert trajectories, as opposed to access to the exact expert occupation measure, is another natural next step, along with stochastic and variance-reduced gradient variants suitable for high-dimensional state and action spaces. Finally, the framework developed here is tied to the stationary, single-population, fully observed regime; extensions to time-dependent mean-field games, multi-population and heterogeneous-agent models, and partially observed dynamics are promising avenues for future research.

%In this paper, we studied inverse reinforcement learning for discrete-time, infinite-horizon mean-field games under an average-reward criterion. We formulated the inverse problem through the maximum causal entropy principle by enforcing consistency with both the expert mean-field interaction and the long-run feature statistics generated by expert demonstrations. Within a unified occupation-measure framework, we considered both finite-dimensional linear reward models and infinite-dimensional RKHS reward classes.

\acks{This work was supported by the Scientific and Technological Research Council of Turkey (TUBITAK), under Grant no: 1001-124F134.}

%\newpage

\appendix

\section{Auxiliary Results on Occupation Measures}\label{app:occupation-proofs}

This appendix collects the standard occupation-measure identities used throughout the paper. Each statement was given in Section~\ref{sect:occupation_measures}; here we record the proofs. All statements are taken under Assumption~\ref{ass:ergodicity}, which ensures that the relevant time-averaged limits exist.

\subsection*{The occupation measure is a probability measure}\label{app:nu-prob}

We first verify in detail the claim made after Definition~\ref{def:occupation_measures}: $\nu_\pi$ is a probability measure on $\X \times \A$. Under Assumption~\ref{ass:ergodicity}, for each $(x, a) \in \X \times \A$ the limit
\[
\nu_\pi(x, a)
=
\lim_{T \to \infty} \frac{1}{T} \sum_{t=0}^{T-1}
\mathbb{P}^{\pi, \mu_E}\big((x_t, a_t) = (x, a)\big)
\]
exists and is non-negative. Summing over all $(x, a) \in \X \times \A$,
\[
\sum_{x \in \X} \sum_{a \in \A} \nu_\pi(x, a)
=
\lim_{T \to \infty} \frac{1}{T} \sum_{t=0}^{T-1}
\sum_{x \in \X} \sum_{a \in \A}
\mathbb{P}^{\pi, \mu_E}\big((x_t, a_t) = (x, a)\big)
=
1,
\]
since for each $t$ the inner sum equals one. Hence $\nu_\pi$ is non-negative and has unit total mass.

\subsection*{Proof of Lemma~\ref{lem:occ-factorization} (Policy factorisation)}\label{app:proof-occ-factorization}
\begin{qedblock}
Fix $(x, a) \in \X \times \A$. Since $\pi$ is stationary, for each $t \geq 0$,
\[
\E^{\pi, \mu_E}\!\Big[\1_{\{(x_t, a_t) = (x, a)\}}\Big]
=
\E^{\pi, \mu_E}\!\Big[
\E^{\pi, \mu_E}\!\left[\1_{\{(x_t, a_t) = (x, a)\}} \mid x_t\right]
\Big].
\]
The inner conditional expectation evaluates as
\[
\E^{\pi, \mu_E}\!\left[\1_{\{(x_t, a_t) = (x, a)\}} \mid x_t\right]
=
\1_{\{x_t = x\}}\, \mathbb{P}^{\pi, \mu_E}(a_t = a \mid x_t)
=
\1_{\{x_t = x\}}\, \pi(a \mid x),
\]
and therefore
\[
\E^{\pi, \mu_E}\!\left[\1_{\{(x_t, a_t) = (x, a)\}}\right]
=
\pi(a \mid x)\,
\E^{\pi, \mu_E}\!\left[\1_{\{x_t = x\}}\right].
\]
Averaging over $t = 0, \ldots, T - 1$ and passing to the limit gives
\[
\nu_\pi(x, a)
=
\pi(a \mid x)\,
\lim_{T \to \infty} \frac{1}{T} \sum_{t=0}^{T-1}
\E^{\pi, \mu_E}\!\left[\1_{\{x_t = x\}}\right].
\]
By Definition~\ref{def:occupation_measures},
\[
\lim_{T \to \infty} \frac{1}{T} \sum_{t=0}^{T-1}
\E^{\pi, \mu_E}\!\left[\1_{\{x_t = x\}}\right]
=
\sum_{b \in \A} \nu_\pi(x, b)
=
\nu_\pi^\X(x).
\]
Hence
\[
\nu_\pi(x, a) = \pi(a \mid x)\, \nu_\pi^\X(x). 
\]
\end{qedblock}

\subsection*{Proof of Lemma~\ref{lem:flow} (Flow Balance)}\label{app:proof-flow}
\begin{qedblock}
Fix $x \in \X$. By Definition~\ref{def:occupation_measures},

{\small
\[
\sum_{(y, a) \in \X \times \A} p(x \mid y, a, \mu_E)\, \nu_\pi(y, a)
=
\lim_{T \to \infty} \frac{1}{T} \sum_{t=0}^{T-1}
\sum_{(y, a) \in \X \times \A}
p(x \mid y, a, \mu_E)\,
\E^{\pi, \mu_E}\!\left[\1_{\{(x_t, a_t) = (y, a)\}}\right].
\]}
Since $\X \times \A$ is finite, we may interchange the sum and the expectation:
{\small
\[
\sum_{(y, a) \in \X \times \A} p(x \mid y, a, \mu_E)\, \nu_\pi(y, a)
=
\lim_{T \to \infty} \frac{1}{T} \sum_{t=0}^{T-1}
\E^{\pi, \mu_E}\!\Bigg[
\sum_{(y, a) \in \X \times \A}
p(x \mid y, a, \mu_E)\, \1_{\{(x_t, a_t) = (y, a)\}}
\Bigg].
\]}
The indicator picks out the unique pair $(y, a) = (x_t, a_t)$, so
\[
\sum_{(y, a) \in \X \times \A} p(x \mid y, a, \mu_E)\, \nu_\pi(y, a)
=
\lim_{T \to \infty} \frac{1}{T} \sum_{t=0}^{T-1}
\E^{\pi, \mu_E}\!\left[p(x \mid x_t, a_t, \mu_E)\right].
\]
By the tower property,
\[
\E^{\pi, \mu_E}\!\left[p(x \mid x_t, a_t, \mu_E)\right]
=
\E^{\pi, \mu_E}\!\left[
\E^{\pi, \mu_E}\!\left[\1_{\{x_{t+1} = x\}} \mid x_t, a_t\right]
\right]
=
\E^{\pi, \mu_E}\!\left[\1_{\{x_{t+1} = x\}}\right].
\]
Therefore,
\[
\sum_{(y, a) \in \X \times \A} p(x \mid y, a, \mu_E)\, \nu_\pi(y, a)
=
\lim_{T \to \infty} \frac{1}{T} \sum_{t=0}^{T-1}
\E^{\pi, \mu_E}\!\left[\1_{\{x_{t+1} = x\}}\right].
\]
On the other hand, by definition,
\[
\nu_\pi^\X(x)
=
\lim_{T \to \infty} \frac{1}{T} \sum_{t=0}^{T-1}
\E^{\pi, \mu_E}\!\left[\1_{\{x_t = x\}}\right].
\]
The two Cesàro limits coincide because their finite-horizon difference is $O(1/T)$:
{\small
\begin{align*}
\left|
\frac{1}{T} \sum_{t=0}^{T-1} \E^{\pi, \mu_E}\!\left[\1_{\{x_{t+1} = x\}}\right]
-
\frac{1}{T} \sum_{t=0}^{T-1} \E^{\pi, \mu_E}\!\left[\1_{\{x_t = x\}}\right]
\right|
&=
\frac{1}{T}
\left|
\E^{\pi, \mu_E}\!\left[\1_{\{x_T = x\}}\right]
-
\E^{\pi, \mu_E}\!\left[\1_{\{x_0 = x\}}\right]
\right| \\
&\le \frac{1}{T}.
\end{align*}
}
Hence
\[
\nu_\pi^\X(x) = \sum_{(y, a) \in \X \times \A} p(x \mid y, a, \mu_E)\, \nu_\pi(y, a). 
\]
\end{qedblock}
\subsection*{Proof of Lemma~\ref{lem:feature-exp} (Static feature representation)}\label{app:proof-feature-exp}
\begin{qedblock}
By Definition~\ref{def:feature_exp_vec},
\begin{align*}
\langle \varphi \rangle_{\pi, \mu_E}
&= \lim_{T \to \infty} \frac{1}{T} \sum_{t=0}^{T-1}
\E^{\pi, \mu_E}\!\left[\varphi(x_t, a_t, \mu_E)\right] \\
&= \lim_{T \to \infty} \frac{1}{T} \sum_{t=0}^{T-1}
\E^{\pi, \mu_E}\!\Bigg[
\sum_{(x, a) \in \X \times \A}
\varphi(x, a, \mu_E)\, \1_{\{(x_t, a_t) = (x, a)\}}
\Bigg] \\
&= \sum_{(x, a) \in \X \times \A}
\varphi(x, a, \mu_E)\, \nu_\pi(x, a),
\end{align*}
where the interchange of the finite sum and the time-average limit is justified by the finiteness of $\X \times \A$.
\end{qedblock}

\section{Proofs From the Maximum Causal Entropy Section}\label{app:opt-proofs}

This appendix collects the proofs of two consequences of feasibility for the maximum causal entropy problem $\opt$ stated in Section~\ref{sect:MaxEnt}. Both proofs rely on the occupation-measure identities of Appendix~\ref{app:occupation-proofs}.

\subsection*{Proof of Lemma~\ref{lem:muE_nu_marginal}}\label{app:proof-muE-nu-marginal}
\begin{qedblock}
Since $\pi$ is feasible for $\opt$, it satisfies the first constraint:
\[
\mu_E(x) = \sum_{(y, a) \in \X \times \A} p(x \mid y, a, \mu_E)\, \pi(a \mid y)\, \mu_E(y),
\qquad \forall x \in \X.
\]
Hence $\mu_E$ is an invariant distribution for the controlled transition kernel under $(\pi, \mu_E)$. Since the initial distribution is also $\mu_E$, it follows that
\[
\mathbb{P}^{\pi, \mu_E}(x_t = x) = \mu_E(x),
\qquad \forall t \geq 0,\ \forall x \in \X.
\]
Therefore, by Definition~\ref{def:occupation_measures},
\[
\nu_\pi^\X(x)
=
\lim_{T \to \infty} \frac{1}{T} \sum_{t=0}^{T-1}
\E^{\pi, \mu_E}\!\left[\1_{\{x_t = x\}}\right]
=
\lim_{T \to \infty} \frac{1}{T} \sum_{t=0}^{T-1} \mu_E(x)
=
\mu_E(x).
\]
\end{qedblock}
\subsection*{Proof of Lemma~\ref{lem:re-writing-entropy}}\label{app:proof-rewriting-entropy}
\begin{qedblock}
By Definition~\ref{def:entropy} and Definition~\ref{def:occupation_measures},
\[
H(\pi)
=
\sum_{(x, a) \in \X \times \A}
-\log \pi(a \mid x)\, \nu_\pi(x, a).
\]
By the factorisation identity of Lemma~\ref{lem:occ-factorization},
\[
\nu_\pi(x, a) = \pi(a \mid x)\, \nu_\pi^\X(x),
\qquad (x, a) \in \X \times \A,
\]
so
\[
\pi(a \mid x) = \frac{\nu_\pi(x, a)}{\nu_\pi^\X(x)}
\quad \text{whenever } \nu_\pi^\X(x) > 0.
\]
If $\nu_\pi^\X(x) = 0$, then $\nu_\pi(x, a) = 0$ for every $a \in \A$, and the corresponding terms vanish under the conventions $0 \log 0 \coloneqq 0$, $0 \log(0/0) \coloneqq 0$. Hence
\[
H(\pi)
=
\sum_{(x, a) \in \X \times \A}
-\log\!\left(\frac{\nu_\pi(x, a)}{\nu_\pi^\X(x)}\right) \nu_\pi(x, a).
\]
Since $\pi$ is feasible for $\opt$, Lemma~\ref{lem:muE_nu_marginal} yields $\nu_\pi^\X(x) = \mu_E(x)$ for every $x \in \X$. Substituting,
\[
H(\pi)
=
\sum_{(x, a) \in \X \times \A}
-\log\!\left(\frac{\nu_\pi(x, a)}{\mu_E(x)}\right) \nu_\pi(x, a). 
\]
\end{qedblock}

\section{Proofs From the Finite-Dimensional Linear Reward Section}\label{app:linear-proofs}

This appendix collects the deferred proofs of two structural results from Section~\ref{sect:linear_reward}: the equivalence between the policy formulation $\opt$ and its occupation-measure reformulation $\optt$, and the min--max representation of $\optt$ used as the starting point of the dual analysis. Both proofs are routine, but their length warrants placement here rather than in the main exposition.

\subsection*{Proof of Theorem~\ref{thm:linear_correspondence} (equivalence of $\opt$ and $\optt$)}\label{app:proof-linear-correspondence}
\begin{qedblock}
We prove the two directions separately, and then verify that the map and its inverse are inverse to each other and that they preserve objective values.

\paragraph*{$(\Rightarrow)$\ \ From policies to occupation measures.}
Let $\pi \in \Pi$ be feasible for $\opt$, and consider the induced occupation measure $\nu_\pi$. By Lemma~\ref{lem:muE_nu_marginal}, $\nu_\pi^\X(x) = \mu_E(x)$ for all $x \in \X$. Combined with the flow identity in Lemma~\ref{lem:flow}, this yields
\[
\mu_E(x) = \sum_{(y, a) \in \X \times \A} p(x \mid y, a, \mu_E)\, \nu_\pi(y, a),
\qquad \forall x \in \X.
\]
Furthermore, by Lemma~\ref{lem:feature-exp},
\[
\sum_{(x, a) \in \X \times \A} \varphi(x, a, \mu_E)\, \nu_\pi(x, a) = \langle \varphi \rangle_{\pi_E, \mu_E}.
\]
Since $\nu_\pi$ is a probability measure on $\X \times \A$, we also have $\nu_\pi(x, a) \geq 0$. Hence $\nu_\pi$ satisfies all the constraints of $\optt$ and is therefore feasible. By Lemma~\ref{lem:re-writing-entropy}, the objective values at $\pi$ and at $\nu_\pi$ coincide.

\paragraph*{$(\Leftarrow)$\ \ From occupation measures to policies.}
Conversely, suppose $\nu \in \mathcal{P}(\X \times \A)$ is feasible for $\optt$. Define a policy $\pi_\nu \in \Pi$ by
\[
\pi_\nu(a \mid x) \coloneqq
\begin{cases}
\dfrac{\nu(x, a)}{\nu^\X(x)}, & \nu^\X(x) > 0, \\[4pt]
\pi_0(a \mid x), & \nu^\X(x) = 0,
\end{cases}
\]
where $\pi_0 \in \Pi$ is any fixed reference policy and $\nu^\X(x) \coloneqq \sum_a \nu(x, a)$.

Since $\nu$ is feasible for $\optt$, the marginal constraint gives $\nu^\X = \mu_E$, and the flow constraint yields
\[
\mu_E(x) = \sum_{(y, a) \in \X \times \A} p(x \mid y, a, \mu_E)\, \nu(y, a),
\qquad \forall x \in \X.
\]
Using the definition of $\pi_\nu$, we may write $\nu(y, a) = \nu^\X(y)\, \pi_\nu(a \mid y) = \mu_E(y)\, \pi_\nu(a \mid y)$, so the flow identity becomes
\[
\mu_E(x) = \sum_{y \in \X} \sum_{a \in \A} p(x \mid y, a, \mu_E)\, \pi_\nu(a \mid y)\, \mu_E(y),
\]
which says exactly that $\mu_E \in \Lambda(\pi_\nu)$.

Let $\nu_{\pi_\nu}$ denote the occupation measure induced by $\pi_\nu$ under the invariant distribution $\mu_E$. Since $\mu_E \in \Lambda(\pi_\nu)$, the state process under $(\pi_\nu, \mu_E)$ has invariant distribution $\mu_E$, so $\nu_{\pi_\nu}^\X(x) = \mu_E(x)$ for all $x \in \X$. Combined with Lemma~\ref{lem:occ-factorization},
\[
\nu_{\pi_\nu}(x, a) = \pi_\nu(a \mid x)\, \nu_{\pi_\nu}^\X(x) = \mu_E(x)\, \pi_\nu(a \mid x).
\]
By construction of $\pi_\nu$, this equals
\[
\mu_E(x)\, \frac{\nu(x, a)}{\nu^\X(x)} = \nu(x, a) \quad \text{whenever} \quad \nu^\X(x) > 0,
\]
and both sides vanish when $\nu^\X(x) = 0$ (since feasibility and non-negativity then force $\nu(x, a) = 0$). Hence $\nu_{\pi_\nu} = \nu$, and in particular $\nu_{\pi_\nu}^\X(x) = \mu_E(x)$ for all $x \in \X$.

It remains to verify feasibility of $\pi_\nu$ for $\opt$. We have already shown $\mu_E \in \Lambda(\pi_\nu)$. The feature-matching constraint holds by Lemma~\ref{lem:feature-exp}:
\[
\sum_{(x, a) \in \X \times \A} \varphi(x, a, \mu_E)\, \nu_{\pi_\nu}(x, a)
= \sum_{(x, a) \in \X \times \A} \varphi(x, a, \mu_E)\, \nu(x, a)
= \langle \varphi \rangle_{\pi_E, \mu_E}.
\]
Finally, by Lemma~\ref{lem:re-writing-entropy},
\[
H(\pi_\nu)
= \sum_{(x, a) \in \X \times \A} -\log\!\left(\frac{\nu_{\pi_\nu}(x, a)}{\mu_E(x)}\right) \nu_{\pi_\nu}(x, a)
= \sum_{(x, a) \in \X \times \A} -\log\!\left(\frac{\nu(x, a)}{\mu_E(x)}\right) \nu(x, a),
\]
which equals the value of $\optt$ at $\nu$. Hence $\pi_\nu$ is feasible for $\opt$ and achieves the same objective value as $\nu$ in $\optt$.

The two constructions $\pi \mapsto \nu_\pi$ and $\nu \mapsto \pi_\nu$ are inverse to each other on the respective feasible sets, and they preserve objective values. This completes the proof. 
\end{qedblock}

\subsection*{Proof of Proposition~\ref{prop:minmax-linear} (min--max representation)}\label{app:proof-minmax-linear}
\begin{qedblock}
Introduce Lagrange multipliers $\bm{\alpha} \in \R^k$ and $\bm{\beta}, \bm{\theta} \in \R^\X$ for the feature-matching, flow-balance, and marginal constraints of $\optt$, respectively. Then $\optt$ admits the representation
\begin{align*}
\optt = \max_{\nu \in \mathcal{P}(\X \times \A)}
\min_{\substack{\bm{\alpha} \in \R^k \\ \bm{\beta}, \bm{\theta} \in \R^\X}}
\Bigg\{
&\sum_{(x, a) \in \X \times \A} -\log\!\left(\frac{\nu(x, a)}{\mu_E(x)}\right) \nu(x, a) \\
&+ \left\langle \bm{\alpha},\, \sum_{(x, a) \in \X \times \A} \varphi(x, a, \mu_E)\, \nu(x, a) - \langle \varphi \rangle_{\pi_E, \mu_E} \right\rangle \\
&+ \sum_{z \in \X} \bm{\beta}_z \left( \sum_{(x, a) \in \X \times \A} p(z \mid x, a, \mu_E)\, \nu(x, a) - \mu_E(z) \right) \\
&+ \sum_{x \in \X} \bm{\theta}_x \big( \nu^\X(x) - \mu_E(x) \big)
\Bigg\}.
\end{align*}
If $\nu$ is feasible for $\optt$, all constraint residuals vanish and the inner minimisation leaves only the objective of $\optt$; if any constraint is violated, the inner minimisation diverges to $-\infty$ since the multipliers are unrestricted.

We now simplify the inner expression. Using the entropy decomposition
\[
\sum_{(x, a) \in \X \times \A} -\log\!\left(\frac{\nu(x, a)}{\mu_E(x)}\right) \nu(x, a)
= H(\nu) + \sum_{(x, a) \in \X \times \A} \log(\mu_E(x))\, \nu(x, a),
\]
we obtain
\begin{align*}
&
\sum_{(x, a)} -\log\!\left(\frac{\nu(x, a)}{\mu_E(x)}\right) \nu(x, a)
+ \left\langle \bm{\alpha}, \sum_{(x, a)} \varphi(x, a, \mu_E)\, \nu(x, a) - \langle \varphi \rangle_{\pi_E, \mu_E} \right\rangle \\
&\quad + \sum_z \bm{\beta}_z \left( \sum_{(x, a)} p(z \mid x, a, \mu_E)\, \nu(x, a) - \mu_E(z) \right)
+ \sum_x \bm{\theta}_x \big( \nu^\X(x) - \mu_E(x) \big) \\
&= H(\nu) + \sum_{(x, a)} \log(\mu_E(x))\, \nu(x, a)
+ \sum_{(x, a)} \langle \bm{\alpha}, \varphi(x, a, \mu_E) \rangle\, \nu(x, a)
- \langle \bm{\alpha}, \langle \varphi \rangle_{\pi_E, \mu_E} \rangle \\
&\quad + \sum_{(x, a)} \nu(x, a) \sum_z \bm{\beta}_z\, p(z \mid x, a, \mu_E)
- \sum_z \bm{\beta}_z\, \mu_E(z)
+ \sum_x \bm{\theta}_x\, \nu^\X(x)
- \sum_x \bm{\theta}_x\, \mu_E(x).
\end{align*}
Using $\nu^\X(x) = \sum_a \nu(x, a)$, the display simplifies to
\begin{align*}
H(\nu)
&+ \sum_{(x, a) \in \X \times \A} \nu(x, a) \bigg[
\log \mu_E(x)
+ \langle \bm{\alpha}, \varphi(x, a, \mu_E) \rangle
+ \bm{\theta}_x
+ \sum_z \bm{\beta}_z\, p(z \mid x, a, \mu_E)
\bigg] \\
&\quad - \langle \bm{\alpha}, \langle \varphi \rangle_{\pi_E, \mu_E} \rangle
- \sum_z \bm{\beta}_z\, \mu_E(z)
- \sum_x \bm{\theta}_x\, \mu_E(x).
\end{align*}
Since $\nu \in \mathcal{P}(\X \times \A)$, we have $\sum_{(x, a)} \nu(x, a) = 1$, so the term $-\sum_z \bm{\beta}_z\, \mu_E(z)$ can be absorbed into the coefficient of $\nu(x, a)$:
{\small
\begin{align*}
H(\nu)
&+ \sum_{(x, a) \in \X \times \A} \nu(x, a) \bigg[
\log \mu_E(x)
+ \langle \bm{\alpha}, \varphi(x, a, \mu_E) \rangle
+ \bm{\theta}_x
+ \sum_z \bm{\beta}_z \big( p(z \mid x, a, \mu_E) - \mu_E(z) \big)
\bigg] \\
&\quad - \langle \bm{\alpha}, \langle \varphi \rangle_{\pi_E, \mu_E} \rangle
- \sum_x \bm{\theta}_x\, \mu_E(x).
\end{align*}}
This is exactly the representation in the statement of the proposition.
\end{qedblock}

\subsection*{Proof of Lemma~\ref{lem:gauge_fixed_invariance} 
(Invariance of the gauge-fixed parameter space)}
\label{app:proof-gauge-fixed-invariance}

\begin{qedblock}
The only components that require verification are the $\boldsymbol{\beta}$ and
$\boldsymbol{\theta}$-components, since the $\boldsymbol{\alpha}$-component is unconstrained and always
belongs to $\R^k$. Recall that
\[
\R^\X_0=\{v\in\R^\X:\langle v,\mathbf 1\rangle=0\},
\]
where $\mathbf 1\in\R^\X$ denotes the vector whose every component is equal to
one. Thus, it is enough to show that the $\boldsymbol{\beta}$ and $\boldsymbol{\theta}$-components of
the gradient have zero inner product with $\mathbf 1$.

First, for arbitrary $(\boldsymbol{\alpha},\boldsymbol{\beta},\boldsymbol{\theta})$, we have
\[
\nabla_{\boldsymbol{\beta}}h(\boldsymbol{\alpha},\boldsymbol{\beta},\boldsymbol{\theta})
=
\sum_{(x,a)\in\X\times\A}
\nu^\star(x,a)
\big(p(\cdot\mid x,a,\mu_E)-\mu_E(\cdot)\big).
\]
Taking the inner product with $\mathbf 1$ gives
\begin{align*}
\left\langle
\nabla_{\boldsymbol{\beta}}h(\boldsymbol{\alpha},\boldsymbol{\beta},\boldsymbol{\theta}),\mathbf 1
\right\rangle
&=
\sum_{(x,a)\in\X\times\A}
\nu^\star(x,a)
\left(
\sum_{z\in\X}p(z\mid x,a,\mu_E)
-
\sum_{z\in\X}\mu_E(z)
\right)  \\
&=
\sum_{(x,a)\in\X\times\A}
\nu^\star(x,a)(1-1)
=0.
\end{align*}
Hence
\[
\nabla_{\boldsymbol{\beta}}h(\boldsymbol{\alpha},\boldsymbol{\beta},\boldsymbol{\theta})\in\R^\X_0.
\]

\noindent Similarly,
\[
\nabla_{\boldsymbol{\theta}}h(\boldsymbol{\alpha},\boldsymbol{\beta},\boldsymbol{\theta})
=
\sum_{(x,a)\in\X\times\A}
\nu^\star(x,a)e_x(\cdot)
-
\mu_E(\cdot).
\]
Taking the inner product with $\mathbf 1$, and using
$\langle e_x(\cdot),\mathbf 1\rangle=1$ and
$\langle\mu_E(\cdot),\mathbf 1\rangle=1$, we obtain
\begin{align*}
\left\langle
\nabla_{\boldsymbol{\theta}}h(\boldsymbol{\alpha},\boldsymbol{\beta},\boldsymbol{\theta}),\mathbf 1
\right\rangle
&=
\sum_{(x,a)\in\X\times\A}
\nu^\star(x,a)
-
1
=0,
\end{align*}
since $\nu^\star$ is a probability distribution on $\X\times\A$. Thus
\[
\nabla_{\boldsymbol{\theta}}h(\boldsymbol{\alpha},\boldsymbol{\beta},\boldsymbol{\theta})\in\R^\X_0.
\]

\noindent We now prove the invariance of the iterates. Suppose that, for some $n\geq0$,
\[
\boldsymbol{\beta}_n,\boldsymbol{\theta}_n\in\R^\X_0.
\]
Since $\R^\X_0$ is a linear subspace, and since
\[
\nabla_{\boldsymbol{\beta}}h(\boldsymbol{\alpha}_n,\boldsymbol{\beta}_n,\boldsymbol{\theta}_n)\in\R^\X_0,
\qquad
\nabla_{\boldsymbol{\theta}}h(\boldsymbol{\alpha}_n,\boldsymbol{\beta}_n,\boldsymbol{\theta}_n)\in\R^\X_0,
\]
we have
\[
\boldsymbol{\beta}_{n+1}
=
\boldsymbol{\beta}_n-\delta\nabla_{\boldsymbol{\beta}}h(\boldsymbol{\alpha}_n,\boldsymbol{\beta}_n,\boldsymbol{\theta}_n)
\in\R^\X_0
\]
and
\[
\boldsymbol{\theta}_{n+1}
=
\boldsymbol{\theta}_n-\delta\nabla_{\boldsymbol{\theta}}h(\boldsymbol{\alpha}_n,\boldsymbol{\beta}_n,\boldsymbol{\theta}_n)
\in\R^\X_0.
\]
The initial condition gives $\boldsymbol{\beta}_0,\boldsymbol{\theta}_0\in\R^\X_0$. Hence, by induction,
$\boldsymbol{\beta}_n,\boldsymbol{\theta}_n\in\R^\X_0$ for every $n\geq0$. Since
$\boldsymbol{\alpha}_n\in\R^k$ for every $n$, we conclude that
\[
(\boldsymbol{\alpha}_n,\boldsymbol{\beta}_n,\boldsymbol{\theta}_n)
\in
\R^k\times\R^\X_0\times\R^\X_0
\]
for every $n\geq0$.
\end{qedblock}

\section{Proofs From the RKHS Reward Section}\label{app:rkhs-proofs}

This appendix collects the deferred proofs from Section~\ref{sec:rkhs}: the equivalence between $\opt$ and $\opttH$, the Fréchet differentiability of the Bellman fixed point, the score-stationarity characterisation, and the Lipschitz smoothness of the score function.

\subsection*{Proof of Proposition~\ref{prop:RKHS_equivilance} (equivalence of $\opt$ and $\opttH$)}\label{app:proof-rkhs-equivalence}
\begin{qedblock}
Since $\opt$ and $\opttH$ share the same objective and feature-matching constraint, it suffices to show that their first constraints are equivalent.

\paragraph*{$(\Rightarrow)$}\
Suppose $\pi \in \mathcal{P}(\A \mid \X)$ satisfies the first constraint of $\opt$:
\[
\mu_E(x) = \sum_{(y, a) \in \X \times \A} p(x \mid y, a, \mu_E)\, \pi(a \mid y)\, \mu_E(y), \qquad \forall x \in \X.
\]
Then $\mu_E$ is an invariant distribution for the Markov chain induced by $\pi$ under the fixed mean-field term $\mu_E$. By Assumption~\ref{ass:ergodicity}, this chain is positive Harris recurrent and aperiodic, so its time averages converge to the invariant distribution. Hence
\[
\lim_{T \to \infty} \frac{1}{T} \sum_{t=0}^{T-1} \E^{\pi, \mu_E} \!\left[ \1_{\{x_t = x\}} \right] = \mu_E(x), \qquad \forall x \in \X,
\]
which is the first constraint of $\opttH$.

\paragraph*{$(\Leftarrow)$}\
Conversely, suppose $\pi \in \mathcal{P}(\A \mid \X)$ satisfies
\begin{equation}\label{eq:rkhs-optH-state-average}
\lim_{T \to \infty} \frac{1}{T} \sum_{t=0}^{T-1} \E^{\pi, \mu_E} \!\left[ \1_{\{x_t = x\}} \right] = \mu_E(x), \qquad \forall x \in \X.
\end{equation}
For each $t \geq 0$, set $m_t(x) \coloneqq \E^{\pi, \mu_E}[\1_{\{x_t = x\}}]$. By the tower property,
\[
m_{t+1}(x) = \sum_{(y, a) \in \X \times \A} p(x \mid y, a, \mu_E)\, \pi(a \mid y)\, m_t(y), \qquad \forall x \in \X.
\]
Averaging over $t = 0, \ldots, T - 1$,
\[
\frac{1}{T} \sum_{t=0}^{T-1} m_{t+1}(x)
= \sum_{(y, a) \in \X \times \A} p(x \mid y, a, \mu_E)\, \pi(a \mid y) \left( \frac{1}{T} \sum_{t=0}^{T-1} m_t(y) \right).
\]
On the other hand, by telescoping,
\[
\frac{1}{T} \sum_{t=0}^{T-1} m_{t+1}(x) = \frac{1}{T} \sum_{t=0}^{T-1} m_t(x) + \frac{m_T(x) - m_0(x)}{T},
\]
and since $0 \leq m_t(x) \leq 1$, the last term vanishes as $T \to \infty$. Passing to the limit and using~\eqref{eq:rkhs-optH-state-average},
\[
\mu_E(x) = \sum_{(y, a) \in \X \times \A} p(x \mid y, a, \mu_E)\, \pi(a \mid y)\, \mu_E(y), \qquad \forall x \in \X,
\]
which is the first constraint of $\opt$.
\end{qedblock}

\subsection*{Proof of Proposition~\ref{prop:Q-differentiable} (Fréchet differentiability of $Q_\vartheta$)}\label{app:proof-rkhs-Q-differentiable}
\begin{qedblock}
Since $r_\vartheta$ is linear in $\vartheta$ and $V_Q$ is $C^\infty$ as a map $\ell_\infty(\X \times \A) \to \ell_\infty(\X)$, the map $F(Q, \vartheta) = Q - \mathcal{T}^\vartheta Q$ is continuously Fréchet differentiable in $(Q, \vartheta) \in \R^{\X \times \A} \times (\R^\X \times \mathcal{H}_k)$.

A direct computation gives the Fréchet derivative of $V_Q$ at $Q$ in direction $H \in \ell_\infty(\X \times \A)$:
\[
D_Q V_Q[H](x) = \sum_{b \in \A} \frac{e^{Q(x, b)}}{\sum_{b' \in \A} e^{Q(x, b')}}\, H(x, b),
\]
which satisfies
\begin{equation}\label{RKHS_H_bound}
|D_Q V_Q[H](x)| \leq \sup_{a \in \A} |H(x, a)| \leq \|H\|_\infty.
\end{equation}
Therefore the Fréchet derivative of $\mathcal{T}^\vartheta$ with respect to $Q$ is
\small
\begin{equation}\label{RKHS_T_Fderivative}
\begin{aligned}
    D_Q \mathcal{T}^\vartheta(Q)[H](x, a)
&= \sum_{y \in \X} \Tilde{p}(y \mid x, a, \mu_E)\, D_Q V_Q[H](y)
\\&= \sum_{y \in \X} \Tilde{p}(y \mid x, a, \mu_E)\, \sum_{b \in \A} \frac{e^{Q(y, b)}}{\sum_{b' \in \A} e^{Q(y, b')}}\, H(y, b).
\end{aligned}
\end{equation}
\normalsize
For any pair $(Q, \vartheta)$,
\begin{equation}\label{RKHS_absT_derivative_bound}
\begin{aligned}
|D_Q \mathcal{T}^\vartheta(Q)[H](x, a)|
&\leq \sum_{y \in \X} \Tilde{p}(y \mid x, a, \mu_E)\, |D_Q V_Q[H](y)| \\
&\overset{\eqref{RKHS_H_bound}}{\leq} \sum_{y \in \X} \Tilde{p}(y \mid x, a, \mu_E)\, \|H\|_\infty \\
&\leq \kappa\, \|H\|_\infty.
\end{aligned}
\end{equation}
Taking the supremum over $(x, a)$ gives the bound~\eqref{eq:DT-bound}:
\begin{equation}\label{RKHS_DQT_bound}
\|D_Q \mathcal{T}^\vartheta(Q)[H]\|_\infty \leq \kappa\, \|H\|_\infty.
\end{equation}
Hence $D_Q \mathcal{T}^\vartheta(Q)$ is a bounded linear operator on $\ell_\infty(\X \times \A)$. By induction, for every $n \geq 0$,
\begin{equation}\label{RKHS_DQT_bound_n}
\|(D_Q \mathcal{T}^\vartheta(Q))^n[H]\|_\infty \leq \kappa^n\, \|H\|_\infty.
\end{equation}
Since $\kappa \in (0, 1)$, the Neumann series $\sum_{n = 0}^\infty (D_Q \mathcal{T}^\vartheta(Q))^n$ converges and equals $(I - D_Q \mathcal{T}^\vartheta(Q))^{-1}$:
\[
(I - D_Q \mathcal{T}^\vartheta(Q))^{-1} = \sum_{n = 0}^\infty (D_Q \mathcal{T}^\vartheta(Q))^n.
\]
For every $H$,
\begin{equation}
    \begin{aligned}
        \|(I - D_Q \mathcal{T}^\vartheta(Q))^{-1}[H]\|_\infty
&\leq \sum_{n = 0}^\infty \|(D_Q \mathcal{T}^\vartheta(Q))^n[H]\|_\infty
\\&\leq \sum_{n = 0}^\infty \kappa^n\, \|H\|_\infty
= \frac{1}{1 - \kappa}\, \|H\|_\infty,
    \end{aligned}
\end{equation}
giving the resolvent bound~\eqref{eq:resolvent-bound}. The Jacobian $\nabla_Q F(Q, \vartheta) = I - D_Q \mathcal{T}^\vartheta(Q)$ is therefore invertible. By the implicit function theorem, the unique fixed point $Q_\vartheta$ is Fréchet differentiable in $\vartheta \in \R^\X \times \mathcal{H}_k$. 
\end{qedblock}

\subsection*{Proof of Theorem~\ref{thm:rkhs-score-stationarity} (score stationarity)}\label{app:proof-rkhs-score-stationarity}
\begin{qedblock}
Define the vector-valued function
\[
f(x, a) \coloneqq
\begin{bmatrix}
e_x \\
\varphi(x, a, \mu_E)
\end{bmatrix}
\in \R^\X \times \mathcal{H}_k,
\]
so that $\nabla_\vartheta r_\vartheta(x, a, \mu_E) = f(x, a)$. Differentiating the soft Bellman equations gives
\begin{equation}\label{eq:gradQ}
\nabla_\vartheta Q_\vartheta(x, a) = f(x, a) + \sum_{y \in \X} \Tilde{p}(y \mid x, a, \mu_E)\, \nabla_\vartheta V_{Q_\vartheta}(y),
\end{equation}
\begin{equation}\label{eq:gradV}
\nabla_\vartheta V_{Q_\vartheta}(x)
= \sum_{a \in \A} \frac{e^{Q_\vartheta(x, a)}}{\sum_{b \in \A} e^{Q_\vartheta(x, b)}}\, \nabla_\vartheta Q_\vartheta(x, a)
= \sum_{a \in \A} \pi_\vartheta(a \mid x)\, \nabla_\vartheta Q_\vartheta(x, a).
\end{equation}

\paragraph*{Step 1: Gradient identity.}
Multiplying~\eqref{eq:gradQ} by $\nu_{\pi_\vartheta}(x, a)$ and summing over $(x, a)$ gives
\begin{equation}\label{RKHS_thm6_eq1}
\begin{aligned}
    \sum_{(x, a)} \nu_{\pi_\vartheta}(x, a)\, \nabla_\vartheta Q_\vartheta(x, a)
&= \sum_{(x, a)} \nu_{\pi_\vartheta}(x, a)\, f(x, a)
\\& \qquad+ \sum_{(x, a)} \nu_{\pi_\vartheta}(x, a) \sum_{y \in \X} \Tilde{p}(y \mid x, a, \mu_E)\, \nabla_\vartheta V_{Q_\vartheta}(y).
\end{aligned}
\end{equation}
By Lemma~\ref{lem:occ-factorization} and~\eqref{eq:gradV},
\[
\sum_{(x, a)} \nu_{\pi_\vartheta}(x, a)\, \nabla_\vartheta Q_\vartheta(x, a)
= \sum_{x} \nu^\X_{\pi_\vartheta}(x) \sum_{a} \pi_\vartheta(a \mid x)\, \nabla_\vartheta Q_\vartheta(x, a)
= \sum_{x} \nu^\X_{\pi_\vartheta}(x)\, \nabla_\vartheta V_{Q_\vartheta}(x).
\]
For the $\Tilde{p}$-term on the right-hand side of~\eqref{RKHS_thm6_eq1}, interchange the order of summation and apply Lemma~\ref{lem:substochastic-flow} to obtain
\[
\sum_{(x, a)} \nu_{\pi_\vartheta}(x, a) \sum_{y \in \X} \Tilde{p}(y \mid x, a, \mu_E)\, \nabla_\vartheta V_{Q_\vartheta}(y)
= \sum_{y \in \X} \big( \nu^\X_{\pi_\vartheta}(y) - \xi(y) \big)\, \nabla_\vartheta V_{Q_\vartheta}(y).
\]
Substituting these two identities back into~\eqref{RKHS_thm6_eq1} and rearranging,
\begin{equation}\label{eq:xi_weighted_gradV_identity}
\sum_{(x, a)} \nu_{\pi_\vartheta}(x, a)\, f(x, a) = \sum_{y \in \X} \xi(y)\, \nabla_\vartheta V_{Q_\vartheta}(y),
\end{equation}
which by Definition~\ref{def:occupation_measures} reads
\[
\lim_{T \to \infty} \frac{1}{T} \sum_{t=0}^{T-1} \E^{\pi_\vartheta, \mu_E} [f(x_t, a_t)] = \sum_{y \in \X} \xi(y)\, \nabla_\vartheta V_{Q_\vartheta}(y).
\]

\paragraph*{Step 2: Computation of $\nabla_\vartheta \mathcal{V}(\vartheta)$.}
Using $\log \pi_\vartheta(a \mid x) = Q_\vartheta(x, a) - V_{Q_\vartheta}(x)$ from~\eqref{eq:policy_theta}, substituting into~\eqref{eq:score_fn_2}, and differentiating,
\begin{align*}
\nabla_\vartheta \mathcal{V}(\vartheta)
&= \lim_{T \to \infty} \frac{1}{T} \sum_{t=0}^{T-1} \E^{\pi_E, \mu_E} \!\big[ \nabla_\vartheta Q_\vartheta(x_t, a_t) - \nabla_\vartheta V_{Q_\vartheta}(x_t) \big] \\
&\overset{\eqref{eq:gradQ}}{=} \langle f \rangle_{\pi_E, \mu_E}
+ \lim_{T \to \infty} \frac{1}{T} \sum_{t=0}^{T-1} \E^{\pi_E, \mu_E}\!\Bigg[ \sum_{y \in \X} \Tilde{p}(y \mid x_t, a_t, \mu_E)\, \nabla_\vartheta V_{Q_\vartheta}(y) - \nabla_\vartheta V_{Q_\vartheta}(x_t) \Bigg].
\end{align*}
Writing $\Tilde{p} = p - \xi$, the tower property gives
\[
\E^{\pi_E, \mu_E}\!\Big[ \sum_{y \in \X} p(y \mid x_t, a_t, \mu_E)\, \nabla_\vartheta V_{Q_\vartheta}(y) \Big]
= \E^{\pi_E, \mu_E}\!\big[ \nabla_\vartheta V_{Q_\vartheta}(x_{t+1}) \big],
\]
so the time-averaged telescoping
\[
\lim_{T \to \infty} \frac{1}{T} \sum_{t=0}^{T-1} \E^{\pi_E, \mu_E}\!\big[ \nabla_\vartheta V_{Q_\vartheta}(x_{t+1}) - \nabla_\vartheta V_{Q_\vartheta}(x_t) \big] = 0
\]
holds. Combining with~\eqref{eq:xi_weighted_gradV_identity},
\[
\nabla_\vartheta \mathcal{V}(\vartheta)
= \langle f \rangle_{\pi_E, \mu_E} - \sum_{y \in \X} \xi(y)\, \nabla_\vartheta V_{Q_\vartheta}(y)
= \langle f \rangle_{\pi_E, \mu_E} - \lim_{T \to \infty} \frac{1}{T} \sum_{t=0}^{T-1} \E^{\pi_\vartheta, \mu_E}[f(x_t, a_t)].
\]
Hence
\[
\nabla_\vartheta \mathcal{V}(\vartheta^\star) = 0
\iff
\langle f \rangle_{\pi_E, \mu_E} = \lim_{T \to \infty} \frac{1}{T} \sum_{t=0}^{T-1} \E^{\pi_{\vartheta^\star}, \mu_E}[f(x_t, a_t)].
\]
Reading off the two components of $f$,
\[
\langle f \rangle_{\pi_E, \mu_E}
= \begin{bmatrix} \mu_E \\ \langle \varphi \rangle_{\pi_E, \mu_E} \end{bmatrix},
\quad
\lim_{T \to \infty} \frac{1}{T} \sum_{t=0}^{T-1} \E^{\pi_\vartheta, \mu_E}[f(x_t, a_t)]
= \begin{bmatrix} \lim_T \frac{1}{T} \sum_t \E^{\pi_\vartheta, \mu_E}[\1_{\{x_t = \cdot\}}] \\ \lim_T \frac{1}{T} \sum_t \E^{\pi_\vartheta, \mu_E}[\varphi(x_t, a_t, \mu_E)] \end{bmatrix}.
\]
Therefore $\nabla_\vartheta \mathcal{V}(\vartheta^\star) = 0$ if and only if $\pi_{\vartheta^\star}$ satisfies both constraints of $\opttH$.

\paragraph*{Step 3: Strong duality.}
Let $\pi^\star$ be an optimal solution of $\opttH$. Feasibility makes the Lagrangian-multiplier terms in $\mathcal{L}(\pi^\star, \vartheta)$ vanish, so for every $\vartheta$,
\[
\opttH = H(\pi^\star) = \mathcal{L}(\pi^\star, \vartheta) \leq \max_\pi \mathcal{L}(\pi, \vartheta) = \mathcal{D}(\vartheta).
\]
Taking the minimum over $\vartheta$ preserves the inequality:
\begin{equation}\label{RKHS_tmh6_ineq_0}
\opttH \leq \min_{\vartheta \in \R^\X \times \mathcal{H}_k} \mathcal{D}(\vartheta).
\end{equation}
By definition, $\min_\vartheta \mathcal{D}(\vartheta) \leq \mathcal{D}(\vartheta^\star)$, and $\mathcal{D}(\vartheta^\star) = \mathcal{L}(\pi_{\vartheta^\star}, \vartheta^\star)$ for the inner maximiser $\pi_{\vartheta^\star}$. Hence
\begin{equation}\label{RKHS_tmh6_ineq_1}
\opttH \overset{\eqref{RKHS_tmh6_ineq_0}}{\leq} \min_\vartheta \mathcal{D}(\vartheta) \leq \mathcal{D}(\vartheta^\star) = \mathcal{L}(\pi_{\vartheta^\star}, \vartheta^\star).
\end{equation}
By Step 2, $\pi_{\vartheta^\star}$ is feasible for $\opttH$, so the multiplier terms vanish in $\mathcal{L}(\pi_{\vartheta^\star}, \vartheta^\star)$, leaving only $H(\pi_{\vartheta^\star})$. Since $\opttH$ maximises $H$ over all feasible policies,
\begin{equation}\label{RKHS_tmh6_ineq_2}
\mathcal{L}(\pi_{\vartheta^\star}, \vartheta^\star) = H(\pi_{\vartheta^\star}) \leq \opttH.
\end{equation}
Combining~\eqref{RKHS_tmh6_ineq_1} and~\eqref{RKHS_tmh6_ineq_2},
\[
\opttH \leq \min_\vartheta \mathcal{D}(\vartheta) \leq \mathcal{D}(\vartheta^\star) = \mathcal{L}(\pi_{\vartheta^\star}, \vartheta^\star) \leq \opttH,
\]
so every inequality is in fact an equality. Therefore
\[
\vartheta^\star \in \arg\min_{\vartheta \in \R^\X \times \mathcal{H}_k} \mathcal{D}(\vartheta)
\quad \text{and} \quad
\pi_{\vartheta^\star} \in \arg\max_{\pi \in \mathcal{P}(\A \mid \X)} \opttH. 
\]
\end{qedblock}

\subsection*{Proof of Theorem~\ref{RKHS_thm_smoothness} (Lipschitz smoothness of $\mathcal{V}$)}\label{app:proof-rkhs-smoothness}

\begin{qedblock}
We follow the four-step outline in the main text. Throughout, we use the operators
\[
[\Pi(\vartheta)]_{y, (x, a)} \coloneqq \1_{\{x = y\}}\, \pi_\vartheta(a \mid x),
\qquad
[\Tilde{P}]_{(x, a), y} \coloneqq \Tilde{p}(y \mid x, a, \mu_E),
\qquad
F(x, a) \coloneqq f(x, a),
\]
which encode, respectively, policy averaging over actions, propagation through the sub-stochastic kernel, and the feature mapping.

\paragraph*{Step (i): Lipschitz bound for $Q_\vartheta$.}
For the soft Bellman operator, $\mathcal{T}^{\vartheta_1} Q - \mathcal{T}^{\vartheta_2} Q = r_{\vartheta_1} - r_{\vartheta_2} = \langle \vartheta_1 - \vartheta_2, f(\cdot, \cdot) \rangle_{\R^\X \times \mathcal{H}_k}$ pointwise. Using the fixed-point identities and Theorem~\ref{thm:Q_contraction},
\begin{align*}
\|Q_{\vartheta_1} - Q_{\vartheta_2}\|_\infty
&= \|\mathcal{T}^{\vartheta_1} Q_{\vartheta_1} - \mathcal{T}^{\vartheta_2} Q_{\vartheta_2}\|_\infty \\
&\leq \|\mathcal{T}^{\vartheta_1} Q_{\vartheta_1} - \mathcal{T}^{\vartheta_1} Q_{\vartheta_2}\|_\infty + \|\mathcal{T}^{\vartheta_1} Q_{\vartheta_2} - \mathcal{T}^{\vartheta_2} Q_{\vartheta_2}\|_\infty \\
&\leq \kappa\, \|Q_{\vartheta_1} - Q_{\vartheta_2}\|_\infty + K\, \|\vartheta_1 - \vartheta_2\|_{\R^\X \times \mathcal{H}_k},
\end{align*}
which gives
\begin{equation}\label{ineq:thm7_Q_theta}
\|Q_{\vartheta_1} - Q_{\vartheta_2}\|_\infty \leq \frac{K}{1 - \kappa}\, \|\vartheta_1 - \vartheta_2\|_{\R^\X \times \mathcal{H}_k}.
\end{equation}

\paragraph*{Step (ii): Lipschitz bound for $\Pi(\vartheta)$.}
By \citet[Proposition~4]{GaoPav18}, for each $x \in \X$,
\[
\|\pi_{\vartheta_1}(\cdot \mid x) - \pi_{\vartheta_2}(\cdot \mid x)\|_2
\leq \|Q_{\vartheta_1}(x, \cdot) - Q_{\vartheta_2}(x, \cdot)\|_2.
\]
Cauchy--Schwarz converts the $\ell_2$-bound on $\pi$ to an $\ell_1$-bound at a cost of $\sqrt{|\A|}$, and the standard $\ell_2$-to-$\ell_\infty$ comparison on $Q$ costs another $\sqrt{|\A|}$:
\[
\|\pi_{\vartheta_1}(\cdot \mid x) - \pi_{\vartheta_2}(\cdot \mid x)\|_1
\leq \sqrt{|\A|}\, \|Q_{\vartheta_1}(x, \cdot) - Q_{\vartheta_2}(x, \cdot)\|_2
\leq |\A|\, \|Q_{\vartheta_1}(x, \cdot) - Q_{\vartheta_2}(x, \cdot)\|_\infty.
\]
Taking the maximum over $x$ and combining with~\eqref{ineq:thm7_Q_theta},
\begin{equation}\label{ineq:thm7_Pi_theta}
\|\Pi(\vartheta_1) - \Pi(\vartheta_2)\|_\infty \leq |\A|\, \|Q_{\vartheta_1} - Q_{\vartheta_2}\|_\infty \leq \frac{|\A|\, K}{1 - \kappa}\, \|\vartheta_1 - \vartheta_2\|_{\R^\X \times \mathcal{H}_k}.
\end{equation}

\paragraph*{Step (iii): Resolvent representation.}
Applying $[\Pi(\vartheta) G](x) = \sum_a \pi_\vartheta(a \mid x)\, G(x, a)$ to $G = \nabla_\vartheta V_{Q_\vartheta}$ converts the pointwise identity~\eqref{eq:gradV} into the operator identity
\begin{equation}\label{eq:gradV_OPT}
\nabla_\vartheta V_{Q_\vartheta} = \Pi(\vartheta)\, \nabla_\vartheta Q_\vartheta.
\end{equation}
Substituting~\eqref{eq:gradV_OPT} into~\eqref{eq:gradQ} yields
\[
\nabla_\vartheta Q_\vartheta(x, a) = F(x, a) + \big(\Tilde{P}\, \Pi(\vartheta)\, \nabla_\vartheta Q_\vartheta\big)(x, a),
\]
or equivalently the operator equation $(I - \Tilde{P}\, \Pi(\vartheta))\, \nabla_\vartheta Q_\vartheta = F$. By $\|\Tilde{P}\|_\infty = \kappa$ and $\|\Pi(\vartheta)\|_\infty = 1$, we have $\|\Tilde{P}\, \Pi(\vartheta)\|_\infty \leq \kappa < 1$, so the Neumann series gives
\begin{equation}\label{eq:thm7_Q_OPT}
\nabla_\vartheta Q_\vartheta = (I - \Tilde{P}\, \Pi(\vartheta))^{-1}\, F,
\qquad
\|(I - \Tilde{P}\, \Pi(\vartheta))^{-1}\|_\infty \leq \frac{1}{1 - \kappa}.
\end{equation}

\paragraph*{Step (iv): Comparison of resolvents and final constant.}
Using $B^{-1} - A^{-1} = A^{-1}(A - B) B^{-1}$ on the two resolvents,
\[
\big\| (I - \Tilde{P}\, \Pi(\vartheta_1))^{-1} - (I - \Tilde{P}\, \Pi(\vartheta_2))^{-1} \big\|_\infty
\leq \frac{\kappa}{(1 - \kappa)^2}\, \|\Pi(\vartheta_1) - \Pi(\vartheta_2)\|_\infty,
\]
hence, combining with~\eqref{eq:thm7_Q_OPT} and $\|F\|_\infty \leq K$,
\begin{equation}\label{ineq:thm7_gradQ_theta}
\|\nabla_\vartheta Q_{\vartheta_1} - \nabla_\vartheta Q_{\vartheta_2}\|_\infty
\leq \frac{\kappa\, K}{(1 - \kappa)^2}\, \|\Pi(\vartheta_1) - \Pi(\vartheta_2)\|_\infty
\overset{\eqref{ineq:thm7_Pi_theta}}{\leq} \frac{\kappa\, |\A|\, K^2}{(1 - \kappa)^3}\, \|\vartheta_1 - \vartheta_2\|_{\R^\X \times \mathcal{H}_k}.
\end{equation}

For $\nabla_\vartheta V_{Q_\vartheta}$, write
\[
\nabla_\vartheta V_{Q_{\vartheta_1}} - \nabla_\vartheta V_{Q_{\vartheta_2}}
= \Pi(\vartheta_1)\big(\nabla_\vartheta Q_{\vartheta_1} - \nabla_\vartheta Q_{\vartheta_2}\big)
+ \big(\Pi(\vartheta_1) - \Pi(\vartheta_2)\big)\, \nabla_\vartheta Q_{\vartheta_2}.
\]
Using $\|\Pi(\vartheta_1)\|_\infty = 1$, $\|\nabla_\vartheta Q_{\vartheta_2}\|_\infty \leq K/(1 - \kappa)$ from~\eqref{eq:thm7_Q_OPT}, and~\eqref{ineq:thm7_Pi_theta} and~\eqref{ineq:thm7_gradQ_theta},
\begin{align}
\|\nabla_\vartheta V_{Q_{\vartheta_1}} - \nabla_\vartheta V_{Q_{\vartheta_2}}\|_\infty
&\leq \frac{\kappa\, |\A|\, K^2}{(1 - \kappa)^3}\, \|\vartheta_1 - \vartheta_2\|_{\R^\X \times \mathcal{H}_k}
+ \frac{|\A|\, K^2}{(1 - \kappa)^2}\, \|\vartheta_1 - \vartheta_2\|_{\R^\X \times \mathcal{H}_k} \nonumber \\
&= \frac{|\A|\, K^2}{(1 - \kappa)^3}\, \|\vartheta_1 - \vartheta_2\|_{\R^\X \times \mathcal{H}_k}.\label{ineq:thm7_V_theta}
\end{align}

Since $\nabla_\vartheta \log \pi_\vartheta(a \mid x) = \nabla_\vartheta Q_\vartheta(x, a) - \nabla_\vartheta V_{Q_\vartheta}(x)$, the triangle inequality combined with~\eqref{ineq:thm7_gradQ_theta} and~\eqref{ineq:thm7_V_theta} yields
\begin{equation}\label{ineq:thm7_log_theta}
\|\nabla_\vartheta \log \pi_{\vartheta_1} - \nabla_\vartheta \log \pi_{\vartheta_2}\|_\infty
\leq \frac{|\A|\, K^2\, (\kappa + 1)}{(1 - \kappa)^3}\, \|\vartheta_1 - \vartheta_2\|_{\R^\X \times \mathcal{H}_k}.
\end{equation}

\paragraph*{Final step: Lipschitz bound for the score gradient.}
By~\eqref{eq:score_fn}, 
$$\nabla_\vartheta \mathcal{V}(\vartheta) = \sum_{(x, a)} \nu_{\pi_E}(x, a)\, \nabla_\vartheta \log \pi_\vartheta(a \mid x).$$ 
The triangle inequality, the bound~\eqref{ineq:thm7_log_theta}, and $\sum_{(x, a)} \nu_{\pi_E}(x, a) = 1$ give
\[
\|\nabla_\vartheta \mathcal{V}(\vartheta_1) - \nabla_\vartheta \mathcal{V}(\vartheta_2)\|_{\R^\X \times \mathcal{H}_k}
\leq \frac{|\A|\, K^2\, (\kappa + 1)}{(1 - \kappa)^3}\, \|\vartheta_1 - \vartheta_2\|_{\R^\X \times \mathcal{H}_k},
\]
which is the stated constant. 
\end{qedblock}

\bibliography{reference}

\end{document}